\documentclass[11pt]{article}
\usepackage{graphicx} % Required for inserting images
\usepackage{fullpage}
\usepackage{natbib}
\usepackage{xcolor}
\usepackage{url,hyperref}
\usepackage{bbm}
\usepackage{mathrsfs}
\usepackage{amsmath,amssymb,amsfonts,amsthm}
\usepackage{stmaryrd}
\usepackage{mathtools}
\usepackage{algorithm}
\usepackage{algpseudocode}
\algnewcommand\algorithmicinput{\textbf{Input: }}
\algnewcommand\INPUT{\State\algorithmicinput}
\algnewcommand\algorithmicinitialize{\textbf{Initialize: }}
\algnewcommand\INIT{\State\algorithmicinitialize}
\algnewcommand\algorithmicrun{\textbf{Run: }}
\algnewcommand\RUN{\State\algorithmicrun}
\algnewcommand\algorithmicupdate{\textbf{Update: }}
\algnewcommand\UPDATE{\State\algorithmicupdate}
\algnewcommand\algorithmicset{\textbf{Set: }}
\algnewcommand\SET{\State\algorithmicset}
\algnewcommand\algorithmicquery{\textbf{Query: }}
\algnewcommand\QUERY{\State\algorithmicquery}
\algnewcommand\algorithmicoutput{\textbf{Output: }}
\algnewcommand\OUTPUT{\State\algorithmicoutput}
\usepackage{cleveref}

\newcommand{\cmean}[0]{\ol{\mu}}
\newcommand{\SR}[0]{D}
\newcommand{\samp}[0]{N}
\newcommand{\Nstep}[0]{N_{\textup{step}}}
\newcommand{\vn}[1]{\left|#1\right|}
\newcommand{\Di}{D}
\newcommand{\Rn}{M}

\newcommand{\xl}[0]{x^{\leftarrow}}
\newcommand{\yl}[0]{y^{\leftarrow}}

\newcommand{\defeq}{:=}
\newcommand{\convolve}{*}
\newcommand{\holden}[1]{{\color{purple}Holden: #1}}

\newcommand{\E}[0]{\mathbb{E}}

\newcommand{\sL}[0]{\mathscr{L}}

\newcommand{\N}[0]{\mathbb{N}}

\newcommand{\sP}[0]{\mathscr{P}}
\newcommand{\Pj}[0]{\mathbb{P}}

\newcommand{\R}[0]{\mathbb{R}}

\newcommand{\one}[0]{\mathbbm{1}}
\newcommand{\mk}[0]{\mathbf{k}}
\newcommand{\al}[0]{\alpha}

\newcommand{\ga}[0]{\gamma}

\newcommand{\de}[0]{\delta}
\newcommand{\De}[0]{\Delta}
\newcommand{\ep}[0]{\varepsilon}

\newcommand{\ka}[0]{\kappa}
\newcommand{\la}[0]{\lambda}

\newcommand{\Te}[0]{\Theta}

\newcommand{\Om}[0]{\Omega}
\newcommand{\si}[0]{\sigma}

\newcommand{\ze}[0]{\zeta}

\newcommand{\nin}[0]{\not\in}

\newcommand{\sub}[0]{\subset}

\newcommand{\subeq}[0]{\subseteq}

\newcommand{\bs}[0]{\backslash}
\newcommand{\iy}[0]{\infty}
\newcommand{\rc}[1]{\frac{1}{#1}}
\newcommand{\prc}[1]{\pa{\rc{#1}}}

\newcommand{\fc}[2]{\frac{#1}{#2}}

\newcommand{\pf}[2]{\pa{\frac{#1}{#2}}}%Shortcut for fraction with parentheses

\newcommand{\dd}[2]{\frac{d #1}{d #2}}

\newcommand{\ddd}[1]{\frac{d}{d #1}}

\newcommand{\pl}[0]{\partial}
\newcommand{\nb}[0]{\nabla}

\newcommand{\dx}{\,dx}

\newcommand{\ab}[1]{\left| {#1} \right|}
\newcommand{\an}[1]{\left\langle {#1}\right\rangle}
\newcommand{\ba}[1]{\left[ {#1} \right]}
\newcommand{\bc}[1]{\left\{ {#1} \right\}}

\newcommand{\ce}[1]{\left\lceil {#1}\right\rceil}

\newcommand{\pa}[1]{\left( {#1} \right)}

\newcommand{\ve}[1]{\left\Vert {#1}\right\Vert}

\newcommand{\set}[2]{\left\{{#1}:{#2}\right\}}

\newcommand{\ol}[1]{\overline{#1}}

\newcommand{\wt}[1]{\widetilde{#1}}
\newcommand{\wh}[1]{\widehat{#1}}

\DeclareMathOperator*{\amax}{arg\,max}
\DeclareMathOperator*{\amin}{arg\,min}

\newcommand{\err}{\operatorname{err}}

\newcommand{\id}{\mathrm{id}}

\newcommand{\KL}[0]{\operatorname{KL}}

\newcommand{\poly}{\operatorname{poly}}

\newcommand{\Tr}[0]{\operatorname{Tr}}
\newcommand{\tr}[0]{\operatorname{tr}}
\newcommand{\TV}[0]{\operatorname{TV}}

\providecommand{\cal}[1]{\mathcal{#1}}
\renewcommand{\cal}[1]{\mathcal{#1}}

\newcommand{\pull}[9]{
#1\ar@/_/[ddr]_{#2} \ar@{.>}[rd]^{#3} \ar@/^/[rrd]^{#4} & &\\
& #5\ar[r]^{#6}\ar[d]^{#8} &#7\ar[d]^{#9} \\}

\newcommand{\cmp}[9]{
\xymatrix{
#1 \ar[r]^{#4}{#5} \ar@/_2pc/[rr]^{#8}_{#9} & #2 \ar[r]^{#6}_{#7} & #3
}
}

\newcommand{\ha}[1]{\ar@{^(->}[#1]}
\newcommand{\ls}[1]{\ar@{-}[#1]}
\newcommand{\sj}[1]{\ar@{->>}[#1]}
\newcommand{\aq}[1]{\ar@{=}[#1]}
\newcommand{\acir}[1]{\ar@{}[#1]|-{\textstyle{\circlearrowright}}}
\newcommand{\acil}[1]{\ar@{}[#1]|-{\textstyle{\circlearrowleft}}}
\newcommand{\ard}[1]{\ar@{.>}[#1]}
\newcommand{\mt}[1]{\ar@{|->}[#1]}
\newcommand{\inm}[1]{\ar@{}[#1]|-{\in}}
\newcommand{\inr}{\ar@{}[d]|-{\rotatebox[origin=c]{-90}{$\in$}}}
\newcommand{\inl}{\ar@{}[u]|-{\rotatebox[origin=c]{90}{$\in$}}}

\newcommand{\sumr}[2]{\sum_{\scriptsize \begin{array}{c}{#1}\\{#2}\end{array}}}%sum with 2 rows
\newcommand{\trow}[2]{\scriptsize \begin{array}{c}{#1}\\{#2}\end{array}}

\newcommand{\sumo}[2]{\sum_{#1=1}^{#2}}
\newcommand{\sumz}[2]{\sum_{#1=0}^{#2}}
\newcommand{\prodo}[2]{\prod_{#1=1}^{#2}}

\newcommand{\coltwo}[2]{
\begin{pmatrix}
{#1}\\
{#2}
\end{pmatrix}}

\newcommand{\beq}[1]{\begin{equation}\llabel{#1}}
\newcommand{\eeq}[0]{\end{equation}}
\newcommand{\bal}[0]{\begin{align*}}
\newcommand{\eal}[0]{\end{align*}}%this doesn't work; i don't know why
\newcommand{\ban}[0]{\begin{align}}
\newcommand{\ean}[0]{\end{align}}

\newcommand{\fixme}[1]{{\color{red}#1}}
\newcommand{\llabel}[1]{\label{#1}\text{\fixme{\tiny#1}}}

\newcommand{\arxiv}[1]{\url{http://www.arxiv.org/abs/#1}}

\newcommand{\vocab}[1]{\textbf{#1}} %also index automatically.

\allowdisplaybreaks[2]

\DeclareFontFamily{U}{wncy}{}
    \DeclareFontShape{U}{wncy}{m}{n}{<->wncyr10}{}
    \DeclareSymbolFont{mcy}{U}{wncy}{m}{n}
    \DeclareMathSymbol{\Sh}{\mathord}{mcy}{"58}

\newcommand\numberthis{\addtocounter{equation}{1}\tag{\theequation}}
\newtheoremstyle{norm}
{12pt}
{12pt}
{}
{}
{\bf}
{:}
{.5em}
{}

\newtheorem{thm}{Theorem}[section]
\newtheorem*{thm*}{Theorem}

\newtheorem*{clm*}{Claim}

\newtheorem*{conj*}{Conjecture}
\newtheorem{cor}[thm]{Corollary}
\newtheorem{lem}[thm]{Lemma}
\newtheorem*{lem*}{Lemma}

\newtheorem{lemma}[thm]{Lemma}
\newtheorem{definition}[thm]{Definition}

\theoremstyle{norm}
\newtheorem{prb}[thm]{Problem}%[section]
\newtheorem*{prb*}{Problem}

\newtheorem{assm}[thm]{Assumption}

\newtheorem*{ax*}{Axiom}
\newtheorem{df}[thm]{Definition}
\newtheorem*{df*}{Definition}

\newtheorem*{ex*}{Example}

\newtheorem{expl}[thm]{Exploration}%prb

\newtheorem*{pos*}{Postulate}
\newtheorem{pr}[thm]{Proposition}
\newtheorem*{pr*}{Proposition}

\newtheorem*{qu*}{Question}

\newtheorem*{rem*}{Remark}

\title{Learning Mixtures of Gaussians Using Diffusion Models}
\author{
  \begin{tabular}{c}
    Khashayar Gatmiry \\
    \texttt{gatmiry@mit.edu} \\
    MIT
  \end{tabular} \quad 
  \begin{tabular}{c}
    Jonathan Kelner \\
    \texttt{kelner@mit.edu} \\
    MIT
  \end{tabular} \quad
  \begin{tabular}{c}
    Holden Lee \\
    \texttt{hlee283@jhu.edu} \\
    JHU
  \end{tabular} 
  }
\date{\today}

\begin{document}

\maketitle

\begin{abstract}
    We give a new algorithm for learning mixtures of $k$ Gaussians (with identity covariance in $\mathbb{R}^n$) to TV error $\varepsilon$, with quasi-polynomial ($O(n^{\text{poly\,log}\left(\frac{n+k}{\varepsilon}\right)})$) time and sample complexity, under a minimum weight assumption. Our results extend to continuous mixtures of Gaussians where the mixing distribution is supported on a union of $k$ balls of constant radius. In particular, this applies to the case of Gaussian convolutions of distributions on low-dimensional manifolds, or more generally sets with small covering number, for which no sub-exponential algorithm was previously known. Unlike previous approaches, most of which are algebraic in nature, our approach is analytic and relies on the framework of diffusion models.     
    Diffusion models are a modern paradigm for generative modeling, which typically rely on learning the score function (gradient log-pdf) along a process transforming a pure noise distribution, in our case a Gaussian, to the data distribution. 
    Despite their dazzling performance in tasks such as image generation, there are few end-to-end theoretical guarantees that they can efficiently learn nontrivial families of distributions; we give some of the first such guarantees.
    We proceed by deriving higher-order Gaussian noise sensitivity bounds for the score functions for a Gaussian mixture to show that that they can be inductively learned using piecewise polynomial regression (up to poly-logarithmic degree), and combine this with known convergence results for diffusion models. 
\end{abstract}

\newpage
\tableofcontents
\newpage
%\newpage
%\tableofcontents
%\newpage
\section{Introduction and main results}
%\subsection{Mixture of $k$ Gaussians}

We address the problem of learning \emph{generalized} mixture of Gaussians (with identity covariance) from samples, using the framework of diffusion models. Formally, we wish to learn the following distribution on $\R^n$ from iid samples:
\[
P_0 = Q_0 \convolve \cal N(0,\si_0^2 I_n),
\]
which we think of as a (possibly continuous) mixture of Gaussians, where $Q_0$ is the distribution of the means. We assume that the support of $Q_0$ is contained within $k$ Euclidean balls, each holding a non-trivial amount of mass. A precise definition will be provided shortly in~\Cref{a:mog}. In the special case that $Q_0 = \sumo jk \al_j \de_{\mu_j}$, this is exactly a mixture of $k$ Gaussians; however, our results hold for any generalized mixture of Gaussians under~\Cref{a:mog}. Note that if the covariance is known but non-identity, we can first transform the data to be in this setting. Our goal is distribution learning, to output samples from a distribution $\ep$-close in TV distance to the actual one. %, and the approach we will use is diffusion models.

Our motivation for this class is twofold: First, as mixtures of Gaussians are one of the simplest but nevertheless challenging mixture models, this is a classic learning problems in statistics and computer science.
%\khasha{talk about non-parameteric somewhere.} 
As a special case, our result also obtains a completely new method distinct from the common algebraic approaches for learning a discrete mixtures of $k$ Gaussians, e.g., \cite{diakonikolas2020small} which obtains quasi-polynomial complexity for that setting (see \Cref{ss:rel} for a discussion of this and other related work). Second, diffusion models are an empirically successful paradigm for generative modeling, which work well for learning multimodal distributions in practice but for which theoretical guarantees are lacking. By applying diffusion models to the problem of learning generalized Gaussian mixtures, our work is the first to give theoretical grounding to the success of diffusion models by fully learning a highly non-trivial class of distributions without assuming oracle access to the score function estimates. This necessitates solving the problem of learning the score function in sub-exponential time. %for this multi-modal class. 
Interestingly, our approach to learning the score across multiple noise levels---essential for executing the diffusion process---actually leverages the diffusion process itself via maintaining a set of ``warm-starts."
%\kadd{Interestingly, for the subproblem of learning the score across multiple noise levels, which is required for running the diffusion process, our score learning algorithmic approach in fact takes advantage of, and depends on,  the diffusion process itself, by maintaining a set of ``warm-starts.'' }

\subsection{Main results}
For $x_0\in \R^n$, let $B_R(x_0) = \set{x\in \R^n}{\vn{x-x_0}\le R}$ denote the closed ball of radius $R$ around $x$. 
We make the following assumptions on $Q_0$. %\holden{should either have $R_0\si_0$ here and $R_0$ in bounds, or $R_0$ here and $R_0/\si_0$ in bounds} 
\begin{assm}[$k$-locality]\label{a:mog}
Fix $R_0\ge 1$, $\SR$, and $k$. The following hold:
\begin{enumerate}
    \item For every point $\mu$ in the support of $Q_0$, we have $Q_0(B_{R_0}(\mu))\ge \al_{\min}$.
    \item There exist $\cmean_1,\ldots, \cmean_k$ such that the support of $Q_0$ is contained in $\bigcup_{i=1}^k B_{R_0}(\cmean_i)$.
    \item $Q_0(B_\SR(0))=1$.
\end{enumerate}
\end{assm}
Note that generalized mixture models are a strict generalization of a mixture of $k$ Gaussians: we permit mixtures of $k$ arbitrary distributions supported on balls of radius $R_0$, convolved with Gaussians. Our main theorem is that these mixtures can be learned with quasi-polynomial time and samples with an algorithm based on diffusion models.
%\khasha{add comparsion to DK20}
\begin{thm}\label{t:main}
Given $\ep > 0$ with %$\ep^2 = O\pa{\pa{\Di \sigma^2 \wedge 1}\alpha_{\min}, \sigma_0^2 \wedge \frac{M_2 + d}{\Di^2} \wedge \frac{\Di \sigma_0}{R_0}}$, 
$\ep \le \min\bc{ \rc 2, \fc{\si_0}{R_0}, \rc{D} , \rc n , \al_{\min}}$,
and given Assumption~\ref{a:mog}, \Cref{a:main} learns a distribution that is $\ep$-close in TV distance to $P_0$ with time and sample complexity
%$\pa{n\ln \pf{k}{\alpha_{\min}\de}}^{c\ln \pf{k\Di}{\ep}^3\ln\pa{\frac{1}{\ep}}^2 + c\pa{\frac{R_0}{\sigma_0}}^6\ln\pa{\frac{1}{\ep}}^2}$ 
$$\pa{n \ln \prc{\de}}^{O\pa{\pa{\ln \pf{1}{\ep}^3 + \pa{\frac{R_0}{\sigma_0}}^6}\ln\pa{\frac{1}{\ep}}^4}}$$
with probability $\ge 1-\de$.
%$n^{\poly\log\pf{n+k}{\ep}}$.
\end{thm}

\begin{algorithm}[h!]\label{alg:mainalgo}
\caption{Learning Gaussian mixture with diffusion model}
\begin{algorithmic}[1]
    \INPUT Error $\ep$, failure probability $\de$, sample access to mixture $P_0$ satisfying \Cref{a:mog}. \newline 
    \algorithmiccomment{Learning}
        \State Let $t_1 = %\fc{\ep}{4\sqrt n}
        \frac{\ep^2 \sigma_0^2}{2\sqrt n}$ and choose a step size schedule $t_1<\cdots < t_{\Nstep}$ as in \Cref{t:dm-kl}.
    \State Let 
    $d = \Te\pa{\pa{\ln \pf{1}{\ep}^3 + \pa{\frac{R_0}{\sigma_0}}^6}\ln\pa{\frac{1}{\ep}}^4}$ for an appropriate constant. \algorithmiccomment{Degree of polynomial approximation}
    %$d=O\pa{\ln \pf{k\Di}{\ep}^3\ln\pa{\frac{1}{\ep}}^2 + \pa{\frac{R_0}{\sigma_0}}^6\ln\pa{\frac{1}{\ep}}^2}$.
    \State Set $\cal C_{\Nstep}=\{0\}$. \algorithmiccomment{Set of warm starts}
    \For{$\ell$ from $\Nstep$ to 1}
        \State Let $V_1,\ldots, V_{k_\ell}$ be the Voronoi partition induced by $\cal C_\ell = \{\wh \mu_1,\ldots, \wh \mu_{k_\ell}\}$.
        \State Draw $\samp=\pa{n%\ln \frac{k}{\alpha_{\min}\de}
        \ln \prc{\de}}^{\Te\pa{d}}$ samples $x_1,\ldots, x_\samp\sim P_0$ and let $y_i = x_i + \sqrt{t_\ell}\cdot \xi_i$, $\xi_i\sim \cal N(0,I_n)$.  
        \algorithmiccomment{Construct the dataset for the denoising objective, for learning the score function.}
        \For{$j$ from 1 to $k_\ell$}
            \State %Let $S_i = \set{y_j}{y_j\in V_i}$. 
            Let $\si_\ell^2 = t_\ell + \si_0^2$ and %\si_0^2=1
            \[
            (\wh b_\mk^{(j)})_{|\mk|\le d} 
            = \amin_{(b_\mk)_{|\mk|\le d}\in B_\Di(0)}
            \sum_{i: y_i\in V_j}\vn{\sum_{|\mk|\le d}b_\mk h_\mk(y_i - \wh \mu_j) - \pa{\pa{1-\fc{\si_\ell^2}{t_\ell}}y_i + \fc{\si_\ell^2}{t_\ell} x_i}}^2
            \]
            where $h_{\mk}$ is the multivariate Hermite polynomial with multi-index $\mk$ and variance $\si^2$ (see \Cref{sec:notation}).   \algorithmiccomment{Perform piecewise polynomial regression according to Voronoi partition, for the denoising objective.}
        \EndFor{}
        \State Define $\wh g_\ell(y)\defeq \sumo i{k_\ell} \one_{V_i}(y) \sum_{|\mk|\le d} \wh b_\mk^{(i)} h_\mk(y-\wh \mu_i)$ and $s_\ell(y) = \fc{\wh g_\ell(y)-y}{\si_\ell^2}$.\newline
        \algorithmiccomment{Define the piecewise polynomial and the score estimate.}
        \If{time $t_\ell$ has halved since last computation of warm starts}
        \State Run \Cref{a:warm-starts} with fresh samples, and failure probability $\fc{\de}{2\Nstep}$ to obtain $\cal C_{\ell-1}$. \newline
        \algorithmiccomment{Recompute warm starts by clustering data points denoised using the score estimate.}
        \Else{\, let $\cal C_{\ell-1}=\cal C_\ell$}
        \EndIf 
    \EndFor{}
    \OUTPUT{Score functions $s_1,\ldots, s_{\Nstep}$.}
    \INPUT{Score functions $s_1,\ldots, s_{\Nstep}$ for times $t_1<\cdots <t_{\Nstep}$}  \algorithmiccomment{Generation}
    \State Draw $\wh y_{t_{\Nstep}}\sim \cal N(0,t_{\Nstep} I_n)$.
    \For{$\ell$ from $N$ to 2}
        \State Let $ \wh y_{t_{\ell-1}} = \wh y_{t_{\ell}} + %h
        2\ba{(t_{\ell}-1) - \sqrt{(t_{\ell-1}-1)(t_{\ell}-1)}}
        s_{t_{\ell}}(\wh y_{t_{\ell}}) + \sqrt{t_{\ell} - t_{\ell-1}}\cdot  \xi_{t_\ell}$ where $\xi_{t_\ell}\sim \cal N(0,I_n)$.
        \algorithmiccomment{Discretization of reverse SDE}
    \EndFor{}
    %\khasha{I think we need to different time sequence for score computation and DDPM, in order to avoid $1/\ep^2$ in sample complexity, the one for score (and warm starts) double each time, but the one for DDPM is like $e^{XXX/\ep^2 (d\vee L)}$} 
    %Can we chat in slack please? ok
    \OUTPUT{$\wh y_{t_1}$}
\end{algorithmic}
\label{a:main}
\end{algorithm}
Note that because of the restriction on $\ep$, $\ln \prc{\ep}$ implicitly has logarithmic dependence on $D$, $n$, and $\rc{\al_{\min}}\ge k$. In the case where $\fc{R_0}{\si_0}$ is a constant, we can remove the dependence on $\fc{R_0}{\si_0}$. We further remark that in the case of a (discrete) mixture of $k$ Gaussians, a straightforward SVD pre-processing step can replace the dependence on $n$ to $\min\{n,k\}$, at an extra additive cost polynomial in $n$  (see e.g., \cite{vempala2004spectral}).
%\khasha{Required condition on $\ep$: $\ep^2 \leq \pa{\Di\sigma_0^2 \vee 1} \alpha_{\min}, \sigma_0, \frac{M_2 + d}{\Di^2}, \frac{L\sigma_0}{R_0}$}
%\holden{figure out the precise power. Looks like we get worse than \cite{diakonikolas2020small} in the mixture of $k$ case.}
%\holden{Proof moved to Proof Overview section.}
%We note that this gives a rare instance of a non-parametric class of distributions which 

We note that as a non-parametric family of distributions, it is already highly non-trivial that a $k$-local generalized mixture of Gaussians has sample complexity that is sub-exponential in dimension. This is an interesting example of a class which does not suffer from the curse of dimensionality in either the sample or time complexity.
%\kadd{Additionally,~\Cref{t:main} asserts that Algorithm~\ref{alg:mainalgo} has quasi-polynomial time and sample complexity for learning generalized Gaussian mixtures under~\Cref{a:mog}---a non-parameteric family of distributions. Beyond the challenge of designing a quasi-polynomial run-time algorithm, many non-parametric classes do not enjoy sub-exponential sample complexity due to the curse of dimensionality.}
%%

\Cref{a:main} consists of two parts: The ``learning" part involves learning the score functions of distributions that bridge the data distribution with a pure noise (Gaussian) distribution. Once these scores are obtained, the ``generation" part uses these learned score functions and can generate as many samples as desired. The high-probability bound is over the learning part: with high probability, the learned score functions are such that the generation process satisfies the TV distance bound.

For the learning part, a step size schedule is chosen that would allow generation with specified error; the algorithm sequentially learns the score functions from large to small $t_\ell$ (time or noise level). 
As the noise level $t_\ell$ decreases, a set of warm starts (or cluster centers) $\cal C_{\ell}$ at the current resolution is maintained and updated. Initially (at the highest noise level $t_{\Nstep}$), all points belong to the same cluster, so $\cal C_{\Nstep}=\{0\}$. At each time step, following the recipe for learning diffusion models, we cast the score estimation problem for $P_{t_\ell} = P_0 * \cal N(0,t_\ell I_n)$ as a supervised learning problem involving denoising data points. The score function is learned within the family of piecewise low-degree polynomials, whose regions are given by the Voronoi diagram of the cluster centers. This is a polynomial regression problem that can be efficiently solved. Whenever the noise level $t_\ell$ is reduced by a constant factor, the set of warm starts $\cal C_{\ell}$ is refined using \Cref{a:warm-starts}. \Cref{a:warm-starts} uses the current score estimate to denoise the data points and obtain estimates of the means, accurate at the current resolution. These means are clustered to obtain $\cal C_{\ell-1}$. Finally, the generation part follows the generation procedure for a diffusion model: start with a Gaussian sample, and follow the reverse SDE with the score estimate to obtain a sample from the learned data distribution.

A special case of \Cref{t:main} is the problem of learning a distribution that is equal to a distribution on a low-dimensional manifold convolved with a Gaussian. The ``manifold" assumption that we need is much weaker: simply that it has can be covered by $C^l$ %\khasha{d is not defined before except in the alg box} 
balls of radius $R_0$, for some constant $C$, where $l > 0$ is the parameter that governs the sample complexity. 
It is straightforward to obtain the following. 
%Substituting $ k=C^d$ above then gives the following.
\begin{cor}
\label{c:manifold}
    Fix $\si_0=1$ and constants $R_0,C>1$. Let $0<\ep<\rc 2$. Suppose that $Q_0$ is supported on a set $M$ such that $M$ has radius $D$, $M$ can be covered with $C^l$ balls of radius $R_0$, and such that for every point $\mu$ in the support of $Q_0$, $Q_0(B_{R_0}(\mu))\ge \fc{\ep}{C^l}$. 
    Then \Cref{a:main} learns a distribution that is $\ep$-close in TV distance to $P_0$ with time and sample complexity $\pa{n\ln \prc\de}^{O\pa{l + \ln\pf{nD}\ep}^7}$ with probability $\ge 1-\de$.
\end{cor}

We note that this is a setting where diffusion models can provably learn under a manifold assumption, but in contrast to most prior work, the distribution \emph{cannot} be learned using straightforward methods such as through binning or kernel density estimation. For example, if $M$ can be covered with $\pf{C}{\ep}^l$ balls of radius $\ep$, then learning a distribution exactly supported on $M$ can be done to Wasserstein distance $\ep$ with complexity $\wt O\pa{\pf{C}{\ep}^l}$ simply with a binning procedure. However, we consider learning a distribution on $M$ convolved with a Gaussian, which is a more challenging problem. 

Additionally, the complexity of our algorithm (stated in Theorem~\ref{t:main}) for the special case of discrete mixtures of $k$ Gaussians is known  using a completely different algorithm based on algebraic methods \citep{diakonikolas2020small} (without the dependence on $\al_{\min}$ or $D$ and with better exponents). 
More precisely, their result learns a distribution $\ep$-close in TV distance to the mixture with time and sample complexity  $\poly(nk/\ep) + (k/\ep)^{O(\ln^2 k)}$, and outputs a density function. 
Though their algorithm does obtain better dependencies, their result relies essentially on a discrete mixture structure, while our result holds for the much larger, non-parametric family of generalized mixtures; to our knowledge the extension to a generalized mixture is novel. (Because their algorithm relies on finding an $\ep$-cover of possible parameters, and an $\ep$-cover of a constant-radius ball is exponential in dimension, it seems unlikely that their methods extend to this setting.)
It is interesting to note that the generic framework of diffusion models allows us to match (up to polylogarithmic factors in the exponent) guarantees obtainable using more specialized algebraic procedures. 

Note that we do \emph{not} proceed by learning the density function; instead, ``learning" the distribution means that we have a procedure  to generate an additional sample, each step of which involves evaluation of a learned score function. It may be possible to upgrade this guarantee to a guarantee of learning the density~\citep{qin2023fit}.

We leave open the questions of removing the requirement on $\al_{\min}$ and improving the polynomial dependence on $\ln\pf{1}{\ep}$ (in Theorem~\ref{t:main}) and $d$ (in~\Cref{c:manifold}). The question remains of whether the complexity of learning mixtures of Gaussians (of equal known covariance) is truly quasi-polynomial, or is actually polynomial. Since our methods work just as well in the more general setting of continuous mixtures, we believe that doing better than quasi-polynomial complexity would require an algorithm specific to a mixture of $k$ Gaussians. Further structure, e.g., hierarchical structure, could also make the problem easier, though our current analysis does not benefit from such assumptions.

We note that mixtures of Gaussians are particularly suited to learning with diffusion models because diffusions preserve the Gaussian. 
We expect that similar results are possible in other settings with a ``match" between the family and the diffusion, e.g., mixtures of product distributions on the hypercube where diffusion is a random walk on the hypercube (i.e. bit-flip noise). It would be interesting to find other families of distributions which be learned using diffusion models, including families of conditional distributions such as mixtures of linear regressions.

Finally, we note that in contrast to our algorithm based on piecewise polynomial regression, in practice the score function is typically learned with a neural network. The score function for a Gaussian mixture is exactly represented by a softmax neural network, which raises the question of whether it can be learned by a neural network with gradient descent. Understanding neural network training dynamics for diffusion models is an important open direction.

\subsection{Related work}
\label{ss:rel}
\paragraph{Learning mixtures of Gaussians.} The problem of learning a mixture of Gaussians from samples has an illustrious history \citep{titterington1985statistical}.
We first distinguish between several types of results. 
First, one can ask for parameter learning or distribution learning---either outputting parameters that are close to the ground-truth parameters, or simply a distribution that is close (e.g., in TV distance) to the ground-truth distribution. For distribution learning, the learning can be improper, that is, the output need not be a mixture of Gaussians. 
Second, one can either study the problem from an information theoretic perspective---the minimum number of samples required to get within a specified error regardless of computational cost---or computational complexity perspective---where the emphasis is on the running time of the algorithm. Below, $n$ denotes the ambient dimension, $k$ the number of components, and $\ep$ the target accuracy.

Most earlier works go through parameter learning and require the means to be sufficiently separated. For identity-covariance Gaussians, \cite{dasgupta2000two,arora2005learning,vempala2004spectral} show that spectral methods work with a separation of at least $\wt\Om(\min\{n,k\}^{1/4})$. 
Allowing arbitrary covariances, \cite{moitra2010settling} show that whenever the mixture is ``$\ep$-statistically learnable," it can be learned with running time and sample complexity $\exp(k)\poly(n,\rc{\ep})$.
Several works use the sum-of-squares method %or tensor methods 
\citep{kothari2018robust,hopkins2018mixture} to learn a mixture with separation $\Om(k^\ga)$ in time $n^{\poly(1/\ga)}$. These methods extend to a broader class of mixture distributions, where components have moments which can be certifiably bounded. 
\cite{liu2022clustering} obtain a polynomial-time algorithm whenever the separation is $\Om(\ln^{\rc 2+c}k)$ for constant $c>0$. 

Separation conditions are unavoidable for sample-efficient parameter learning. \cite{regev2017learning} shows that the threshold for efficient parameter learning is $\Te(\sqrt{\ln k})$: with separation $\Om(\sqrt{\ln k})$, polynomially many samples suffice (information-theoretically), while with separation $o(\sqrt{\ln k})$, super-polynomially many samples are required. \cite{doss2020optimal} conducts a more fine-grained study of this problem.

Hence, any sample-efficient algorithm for learning mixtures of Gaussians without separation cannot go through parameter learning. The optimal information theoretic complexity is known up to logarithmic factors: \cite{ashtiani2018nearly} show a sample complexity bound in TV distance of $\wt \Te(kn^2/\ep^2)$ for a mixture of $k$ Gaussians in $\R^n$ (with any variance) and $\wt \Te(kn/\ep^2)$ for axis-aligned Gaussians. However, their algorithms are based on brute-force search and have
exponential running time. %$2^{kn^2\poly\log(n,k,\rc\ep)}$ and $2^{kn\poly\log(n,k,\rc\ep)}$, respectively. 
\cite{acharya2017sample} give an improper nearly linear-time algorithm based on polynomial interpolation which learns a mixture of Gaussians with arbitrary variances in 1 dimension, with $\wt O(k/\ep^2)$ samples.
%\cite{li2015nearly}
%It is probably possible to extend this to any fixed number of dimensions.
Most relevant for us, for a mixture of Gaussians with identity covariance, 
a breakthrough work by \cite{diakonikolas2020small} uses algebraic geometry to obtains a time and sample complexity of $(k/\ep)^{O(\ln^2k)}$ plus polynomial factors. % \fixme{n dependence?}

%\khasha{even the statistical sample complexity of this class is not clear and our result is the first}
In the statistics literature, the model is referred to as the Gaussian location mixture. \cite{saha2020nonparametric,kim2022minimax} consider arbitrary mixing measures $Q_0$ and give finite-sample bounds using non-parametric MLE (maximum likelihood estimation) for squared Hellinger risk, which scale as $\fc{\log N}{N}$ in terms of the number of samples $N$. The risk depends on the volume of an approximate support of $Q_0$, but also include a constant with unspecified dependence on the dimension $n$.

We note that the picture is more complicated when variances are unknown: \cite{diakonikolas2017statistical} obtain statistical query (SQ) lower bounds ($2^{n^{\Om(1)}}$ queries of fixed polynomial precision) based on a connection with non-Gaussian component analysis, with a ``parallel pancake" construction in a unknown direction. \cite{gupte2022continuous} show that $\log n$ components are enough to obtain a super-polynomial lower bound assuming exponential hardness of the classical LWE problem.
%exact?

Finally, a recent line of work considers the problem of robust learning of Gaussian mixtures \citep{liu2021settling,bakshi2022robustly}, i.e., under adversarial corruption of some fraction of samples; these methods have complexity $n^{O(k)}$.

\paragraph{Concurrent work.} 
During the preparation of this manuscript, we were made aware of independent and concurrent work by Chen, Kontonis, and Shah \citep{chen2024learning} which also gave guarantees for learning Gaussian mixtures using diffusion models and piecewise polynomial regression for score estimation. They consider the more general case where covariances are well-conditioned but arbitrary, but their runtime scales exponentially in $\poly(k/\epsilon)$ rather than $\poly\log(k/\epsilon)$, which is unavoidable by SQ lower bounds \citep{diakonikolas2017statistical}. In contrast to our work, their work uses a different approach to polynomial approximation and does one-shot learning of parameters via spectral methods.

\paragraph{Diffusion models.} Diffusion models \citep{sohl2015deep,song2019generative,song2020score} are a modern generative modeling paradigm which involves defining a forward noising process which turns a data distribution into a pure noise (e.g., Gaussian) distribution, and then learning to simulate the reverse process. For diffusion models based on SDEs (stochastic differential equations), the reverse process involves the score function (gradient of log-pdf) of the intermediate distributions; hence they are also called score-based generative models (SGM). See \cite{tang2024score} for a technical tutorial. 
We note that diffusion models are essentially a reparameterization of stochastic localization \citep{eldan2013thin} as pointed out by \cite{montanari2023sampling}. Stochastic localization has been independently studied in the probability literature and been used to obtain new results in sampling \citep{chen2022localization,el2022sampling}.

Theoretical work has focused on two problems: (1) When is it possible to efficiently learn the score? (2) Given a learned ($L^2$-accurate) score function, what guarantees can we obtain for sampling from the data distribution? Answers to these two questions together would give an end-to-end result for learning via diffusion models.

Addressing (2), it is a remarkable fact that having a $L^2$-accurate score function for the sequence of distributions is sufficient for sampling under minimal distributional assumptions, allowing even multimodal distributions which cause slow mixing for local MCMC algorithms %lee2022convergence,
\citep{lee2023convergence,chen2023score,chen2023improved}. %(2) is now well-understood, and a line of work  culminating in 
\cite{benton2023linear} show that it suffices to have a number of steps linear in the dimension. 

Question (1) has proved to be thornier; it has been a challenge to obtain end-to-end results for non-trivial settings. 
Several works consider the problem of representability by neural networks, such as \cite{cole2024score} for distributions whose log-density relative to a Gaussian can be represented by a low-complexity neural network, or \cite{mei2023deep} for graphical models. 
Following the work on neural network function approximation for smooth functions, \cite{oko2023diffusion} give nearly minimax optimal estimation rates for densities in Besov spaces. \cite{wibisono2024optimal} relates score learning to kernel density estimation. 
The manifold assumption is another popular setting for analysis: \cite{de2022convergence} considers generalization error, and \cite{chen2023score} give learning guarantees when the distribution is supported on a subspace (however, this family of distributions can be trivially learned by first recovering the subspace). 
Finally, we note that score matching can be used to learn exponential families for which sampling is difficult \citep{koehler2022statistical,pabbaraju2024provable}; however, these methods only use the score for the data distribution, rather than a sequence of distributions as in a diffusion model.

\cite{shah2023learning} consider using diffusion models to learn Gaussian mixtures, and show that diffusion models can do as well as the EM algorithm. However, they either require $k=2$, or the components to be well-separated and $O(1)$-warm starts to be given for all the means. 
Gaussian mixtures are a popular toy model for understanding various aspects or behavior of diffusion models, including learning behavior \citep{cui2023analysis}, sample complexity \citep{biroli2023generative}, guided diffusion \citep{wu2024theoretical}, and critical windows \citep{li2024critical}.

\paragraph{Analytic conditions for learning functions.} We rely on the noise sensitivity/stability framework \citep{klivans2008learning}, which shows that Gaussian noise stability (or small Gaussian surface area) implies approximability by a low-degree polynomial, giving an efficient ``low-degree algorithm" for learning. We note a similar-in-spirit result that under the Gaussian distribution, intersections of $k$ halfspaces can be learned in time $n^{O(\ln k)}$. The noise sensitivity framework was previously developed for boolean functions on the hypercube \citep{benjamini1999noise} and applied to learning function classes such as functions of halfspaces \citep{klivans2004learning,kalai2008agnostically}.

Learnability by neural networks can also be related to complex analytic properties of the function: \cite{schwab2023deep} relate a radius of analyticity to the Hermite expansion which gives results on representability by neural networks. %; however their bounds are in general exponential in dimension.
Learnability by neural networks for multi-index models \citep{bietti2023learning}---functions depending on the projection of the input to a few dimensions (such as the score function of a Gaussian mixture with $k\ll n$)---is also related to the Hermite expansion of the function.

\subsection{Notation}\label{sec:notation}
We let $\ga_{\mu, \si^2}$ denote the density of $\cal N(\mu,\si^2I_n)$, and abbreviate $\ga_{\si^2} = \ga_{0,\si^2}$. We abbreviate this as $\ga$ when $\si$ is understood. In general, we denote probability measures by uppercase letters and their corresponding densities by lowercase letters, though sometimes we will conflate the two.

We use $\vn{v} = \vn{v}_2$ to denote the norm of a vector $v\in \R^n$, to avoid confusion with function norms. For a measure $\nu$ on $\Om$, let $\ve{f}_{L^p(\nu)}= \pa{\int_\Om f^p\,d\nu}^{1/p}$.
When the measure is clear, we may simply write $\ve{f}_p$.
Let $\ve{f}_\nu\defeq \ve{f}_{L^2(\nu)}$. 
For a $\R^n$-valued function, we write $\ve{f}_{L^p(\nu)}$ to mean $\ve{\vn{f}}_{L^p(\nu)}$. For $x_0\in \R^n$, let $B_R(x_0) = \set{x\in \R^n}{\vn{x-x_0}\le R}$, and let $B_R=B_R(0)$.
Let $\binom S{k}$ denote the set of subsets of $S$ of size $k$, and $\binom S{\le k}$ denote the set of subsets of $S$ of size at most $k$.
Let $h_{k,1}$ be the Hermite polynomial of degree $k$, which satisfies the following recursive equation for $z\in \mathbb R$,
\begin{align*}
    h_{k,1}(z) = zh_{k-1,1}(z) - h'_{k-1,1}(z),\quad  h_{0,1}(z) = 1
\end{align*}
and define the rescaled version (with variance $\si^2$) to be $h_{k,\si^2} (z) = h_{k,1}(z/\si)$. 
We use the notation $\mk$ to denote a multi-index in $\mathbb N_0^n$.
The multivariate Hermite polynomials, indexed by $\mk\in \mathbb N_0^n$ is defined as $h_{\mk, \si^2}(z) = \prodo in h_{\mk_i, \si^2} (z_i)$, where $z\in \R^n$. 
%\begin{align*}
%\end{align*}
It is well-known that $(h_{\mk, \si^2})_{\mk\in \mathbb N_0^n}$ forms an orthogonal basis for $\mathbb R^n$ with respect to $\ga_{\si^2}$.

We give some further background and notation on Markov semigroups and generators in \Cref{s:ns}.
%\khasha{Define hermite polynomials}
\section{Proof overview}

Diffusion models give a way to reduce the problem of learning a probability distribution $P_0$ from samples, to the problem of estimating \emph{score functions} of $P_0*\cal N(0,tI_n)$ for a sequence of noise levels $t$. A sequence of works \citep{lee2023convergence,chen2023score,chen2023improved,benton2023linear} shows that this reduction works for general distributions, with explicit error bounds.

Our proof then proceeds by sequentially learning the score functions for a decreasing sequence of $t$'s. The proof hinges on the fact that the score function for a mixture of Gaussians is approximable by a piecewise low-degree polynomial function. We first give some intuition as to why we expect this to be true. 
% We also give some intuition on why we expect the score function to be locally approximable by a piecewise low-degree polynomial function. 
Consider a mixture of two Gaussians $P=\rc 2 \cal N(-\mu, 1) + \rc 2\cal N(\mu, 1)$. The score function for this distribution is 
\[
\nb \ln \pa{e^{-\fc{(x-\mu)^2}{2}} + e^{-\fc{(x+\mu)^2}2}} = -x + \mu \tanh(\mu x).
\]
Consider two cases:
\begin{enumerate}
    \item When $\mu$ is bounded, the two Gaussians have non-negligible overlap. Using smoothness properties of $\tanh$, we can approximate the score function uniformly with a low-degree polynomial on $[-\mu-C, \mu+C]$. 
    \item When $\mu\to \iy$, then we can no longer uniformly approximate the score function with a low-degree polynomial on $[-\mu-C, \mu+C]$, because $\tanh(\mu x)$ is very steep around 0 and close to flat for an interval whose length approaches $\iy$. However, importantly, we don't need to: we only care about error with respect to $P$, which is mostly supported on $[-\mu-C,-\mu+C]\cup [\mu-C,\mu+C]$. On each of these intervals, the score function is close to flat---because the other mixture component has negligible effect---and so well-approximated by a polynomial.
\end{enumerate}
In general, we can hope that we can cluster the Gaussians, show that the score function for each cluster can be approximated by a polynomial, and each cluster has negligible effect on the other clusters, giving the piecewise polynomial structure.

Before giving a detailed proof sketch, we we highlight some key techniques used in our proof which may be more generally useful.
\begin{itemize}
    \item \textbf{Higher-order noise sensitivity}: We first consider the case of a single cluster, that is, all the means in the mixture are within a small ball. Let $f$ be the function we wish to estimate. We use the technique of \emph{Gaussian noise sensitivity}: by bounding the norm of $\sL f$ with respect to a Gaussian $\ga$, where $\sL$ is a differential operator, namely the generator of the Ornstein-Uhlenbeck (OU) process, we can show that $f$ is approximable by a low-degree polynomial with respect to $\ga$. However, this only allows approximation by a polynomial of degree $O(1/\ep)$, leading to a sample and time complexity of $n^{O(1/\ep)}$. 
    To overcome this, we instead prove \emph{higher-order} noise sensitivity bounds, bounding $\ve{\sL^m f}_\ga$ for $m$ logarithmically large. To our knowledge, this is the first time that higher-order noise sensitivity has been considered. 
    The high level of smoothness required to bound $\ve{\sL^m f}_\ga$ derives from the fact that the score function has a nice representation in terms of a posterior expectation; we bound its derivatives with careful bookkeeping.
    %After writing $f$ in terms of a posterior expectation, we do some careful bookkeeping in order to bound $\sL^m f$. 
    \item \textbf{Higher-order smoothing}: A particularly delicate part of our proof is a \emph{change-of-measure} argument. Gaussian noise sensitivity gives error bounds under the Gaussian measure $\ga$; however, we care about the error under the true data distribution $\ga'$, which is a mixture. If $f$ is the true function (related to the score function) and $g$ is the estimate, we have by the Cauchy-Schwarz inequality that $\ve{f-g}_{\ga'}^2 \le \ve{f-g}_{\ga}^2  + \ve{f-g}_{L^4(\ga)}^2 \chi^2(\ga'\|\ga)^{1/2}$. However, this means that we need $\ve{f-g}_{L^4(\ga)}$ to be small, while we only have control over $\ve{f-g}_{L^2(\ga)}$. A standard way that we can bound a higher $L^p$ norm by a lower one is to use hypercontractivity: smooth $f$ by the OU semigroup to obtain $\sP_tf$. Again, this ``first-order" smoothing turns out to be insufficient, and we introduce a higher-order smoothing obtained from a higher-order finite-differencing.
    \item \textbf{Using the diffusion model for maintaining clusters}: In the case of a single cluster, the above argument shows the existence of a logarithmic-degree polynomial approximation of the score function, which can be efficiently learned using polynomial regression. In the general case, by localizing the effect on the score function from different clusters, there exists a \emph{piecewise} polynomial approximation. The pieces can be taken to be the Voronoi cells of a suitable set of cluster centers, and the score function can be efficiently learned \emph{if} these cluster centers were known. 
    The key observation is that having an estimate of the score function at the previous (higher) noise level allows this approximate clustering: by Tweedie's formula, the score function exactly points in the direction of denoising a data point, i.e., the posterior of the Gaussian mean that the point came from! Thus, the denoising principle of the diffusion model is an integral part of inductively allowing us to maintain the clusters as the noise level $t$ decreases. As $t$ decreases, we obtain more accurate estimates of the means, which allows us to refine the clusters to the current resolution.
    %In order to obtain an approximate clustering at the current noise level $t$, 
\end{itemize}

The first two techniques are already present in the single-cluster setting, while the last one is a key part of extending the argument to the multiple-cluster setting.

\subsection{Learning with diffusion models}
The idea behind diffusion models is to first define a forward process based on a SDE that takes the data distribution to a pure noise distribution, in our case a Gaussian. 
Then using a result on reversing a SDE
\citep{anderson1982reverse}, we can write down the reverse process which involves the score function. We apply this in the case when the forward process is simply Brownian motion,
\[
dx_t = dW_t, \quad x_0\sim P_0
\]
to obtain that this process on $[0,T]$ is equivalently described by
\[
dx_t = \nb \ln p_t(x_t) + d\wt W_t, \quad 0\le t\le T, \quad x_T\sim P_T
\]
where $x_t$ has distribution $p_t$ and $\wt W_t$ is reverse Brownian motion. 
We choose $T$ large enough so that $P_T$ is close to Gaussian.
Hence, if we learn the score functions $\nb \ln p_t(x_t)$, then we can approximately simulate the reverse process. 
In Section~\ref{s:dm}, we make this precise by adapting known convergence results on diffusion models, which show that we can approximately sample from the data distribution given a $L^2$-accurate score function.

By Tweedie's formula, the score function admits an interpretation in terms of the posterior mean given an observation (see \eqref{e:score-Q2}), so the score matching objective $\E_{p_t} \vn{\nb \ln p_t - s}^2$ can be rewritten as the \emph{supervised} denoising auto-encoder objective (see \Cref{ss:score-gm}).
The problem is now reduced to that of 
%It now remains to give a way to 
learning the score function $\nb \ln p_t$ for each time $t$. For this, we will show that it resides in a low-dimensional function class we can optimize over. 

\subsection{Learning the score for a single cluster}
In Section~\ref{s:ns}, we show that in the special case where all centers are close together, i.e. $Q_0$ is supported on a small Euclidean ball, there is a polynomial of low ($\poly\log k$) degree that approximates the score function with respect to the mixture distribution. The key enabling result is \Cref{l:ns-ld}, which shows that Gaussian noise stability implies low-degree polynomial approximability, by considering the expansion in Hermite polynomials (eigenfunctions of $\sL$). In contrast to existing literature, we employ a \emph{higher-order} version of noise stability which involves bounding the $L^2$-norm of $\sL^mf$---iterates of the generator $\sL$ of the Ornstein-Uhlenbeck process. %Langevin diffusion. 
Specifically, any function $f$ can be approximated by a polynomial $g$ of degree $<d$ such that
\[
\ve{f-g}_{L^2(\cal N(0,I_n))} \le \fc{\ve{\sL^{m}f}_{L^2(\cal N(0,I_n))}}{d^{m}}.
\]
As long as we can bound $\ve{\sL^{m}f}_{L^2(\cal N(0,\si^2 I_n))}$ by $L^{m}$ for reasonable $L$, then the dependence of the necessary degree $d$ on the desired accuracy is $O(\poly\log(1/\ep))$ rather than $O(1/\ep)$. 

The interpretation of the score function as a posterior mean gives us a handle on computing (\Cref{l:Ld} in \Cref{ss:Lmf}) and bounding (\Cref{lem:csbound} in \Cref{ss:bd-Lm}) its derivatives. For convenience, we consider a function $f$ which is linearly related to the score function, and bound $\sL^{2m}f$. Specifically, $f(y) = \an{x}_y$, where $\an{\cdot}_y:=\E_{P_{x|y}}$ is an expectation of the posterior distribution of the mean of the Gaussian that $y$ originated from. Differentiating this creates ``replicas," giving expectations of higher moments under the posterior. Arguments using Jensen's and H\"older's inequalities then bounds this using Gaussian moments. 

However, one key problem remains: \Cref{lem:csbound} gives us low-degree approximability with respect to the Gaussian $\ga$, not the actual mixture $\ga' = P_t$. A standard change-of-measure argument can bound the $L^2$ error under $\ga'$ by a higher $L^p$ error under $\ga$: by the Cauchy-Schwarz inequality,
\[
\ve{f-g}_{\ga'}^2 \le \ve{f-g}_{\ga}^2  + \ve{f-g}_{L^4(\ga)}^2 \chi^2(\ga'\|\ga)^{1/2}.
\]
Here, $f$ is the actual function and $g$ is the polynomial approximation. 
However, we are not able to directly obtain a higher $L^p$ polynomial approximation of $f$. The problem is that the degree of the polynomial required depends poly-logarithmically on the desired $L^2$ accuracy $\ep$, while passing to a higher $L^p$ norm results in a 
%the change of measure gives a 
multiplicative factor exponential in the degree, resulting in an error $\ep e^{\ln^C(1/\ep)}\gg 1$.

By hypercontractivity, we know that we can bound higher $L^p$ norms of $\sP_t f$---the smoothing of $f$ by the Ornstein-Uhlenbeck operator for time $t$---by $\ve{f}_{L^2(\ga)}$. We can approximate $f\approx \sP_tf$ and then approximate $\sP_t f$ by the polynomial $g$. In order for $\ve{f-\sP_t f}_{\ga'}\le \ep$, we need $t=O(\ep)$. However, this turns out to be not enough smoothing. We would like to write $f$ approximately as a reasonably-sized linear combination of $\sP_tf$ (for different $t$'s), for $t$'s that are as large as possible. The insight is that $f-\sP_tf \approx t \cdot \ddd t \sP_t f|_{t=0}$; higher-order finite differencing gives \[(\id - \sP_t)^m f\approx t^m \dd{{}^m}{t^m} \sP_t f|_{t=0}.\] Hence, by taking $m\sim \ln \prc\ep$, we can take $t$ much larger---independent of $\ep$---and still obtain an $\ep$-approximation. This is carried out in \Cref{l:poly-approx-mix} in \Cref{ss:approx-poly}. In summary, we approximate $f$ by a higher-order smoothing by the Ornstein-Uhlenbeck process \emph{before} the polynomial approximation.

%A simple change-of-measure argument fails to work because the degree of the polynomial required depends poly-logarithmically on the desired accuracy $\ep$, while the change of measure gives a multiplicative factor exponential in the degree, resulting in a factor $\ep e^{\ln^C(1/\ep)}\gg 1$. Instead, we have to smooth the distribution by the Ornstein-Uhlenbeck process \emph{before} the polynomial approximation; moreover, we need to consider a higher-order smoothing (\Cref{l:poly-approx-mix} in \Cref{ss:approx-poly}).

\subsection{Learning the score for multiple clusters}
In \Cref{s:multiple} we consider the general setting of multiple clusters. Recall that as we would like to run the backwards diffusion process from pure Gaussian noise back to our mixture, 
%To demonstrate our approach, suppose that 
we want to learn the score function $\nabla \ln p_t(x_t)$ at various times $t=t_\ell$. 
%The major idea that we use in the single cluster case, where the mixture measure $Q_0$'s support has a small radius, is that we show the Hermite polynomials centered around a landmark point in that cluster are expressive enough to well-approximate the score function. To implement this strategy in the multiple cluster case, 
It is no longer true that a single polynomial of low degree approximates the score function with respect to the mixture distribution. 
We must deal with two challenges:
\begin{enumerate}
    \item The score function now has a more complex structure, as it is the posterior mean with respect to multiple clusters of $Q_0$.
    \item We care about the error with respect to a mixture measure which is concentrated around $k$ clusters.
\end{enumerate}
%we face two major challenges with this approach: 
%(1) the mixture measure is not concentrated around a single cluster, but $k$ clusters that can be far from each other, and (2) the score function now has a more complex structure, as it is the posterior mean with respect to multiple clusters of $Q_0$. 
How can we overcome these issues?

First  assume we are given a set of ``warm starts'' $\{\wh \mu_i\}_{i=1}^{k_\ell}$ to the mixture measure $P_t = Q_0 * \cal N(0, (\si_0^2+t)I_n)$, in the sense that balls of radius roughly $\tilde O(\sqrt t)$ around these warm starts cover the support of $Q_0$ (\Cref{d:ws}). Then, we hope to approximate the score function by a suitable piecewise polynomial function with respect to the Voronoi diagram of the warm starts, such that, in every Voronoi cell $V_i$, the corresponding polynomial approximates well the score function restricted to that cell:
\[\wh g_\ell(y)\defeq \sumo i{k_\ell} \one_{V_i}(y) \sum_{|\mk|\le d} \wh b_\mk^{(i)} h_\mk(y-\wh \mu_i).\]
Here, $h_{\mk}$ is the Hermite polynomial with multi-index $\mk$. 
To reduce the %multiple cluster 
problem of approximating the score in a Voronoi cell in the multiple cluster case to the single cluster case, we prove that there is a neighborhood $S_i$ of the Voronoi cell $V_i$, such that, viewing the score function as the posterior mean, we can approximate it in %that cell 
$V_i$ by ``pretending'' that the prior distribution has been restricted to the mixing measure $Q_0$ \emph{restricted} to %that neighborhood 
$S_i$ (\Cref{l:nearby-score} in \Cref{ss:nearby-score}). Given that the radius of this neighborhood is small, we can use this argument to overcome challenge 1. For challenge 2, we bound the R\'enyi divergence of the mixture measure with respect to the Gaussian centered at the warm start point in the Voronoi cells (\Cref{l:chi-p-V} in Section~\ref{ss:voronoi}), and use that for a change-of-measure argument given by \Cref{l:poly-approx-mix}, to bound the approximation error under the mixture measure in terms of the approximation error under the Gaussian.
%change to approximation measure between the former and the latter.
%%
Having shown the existence of the piecewise polynomial approximation by combining the Hermite approximations within all the Voronoi cells (\Cref{lem:lowdegree} in Section~\ref{ss:voronoi}), we can then efficiently approximate the score by applying a polynomial regression  with suitable choice of degree restricted to each Voronoi cell. We analyze the sample complexity of piecewise polynomial regression in~\Cref{ss:ppoly} by bounding the generalization gap of the squared loss of low-degree polynomial functions in predicting the score (\Cref{lem:gengap}).  

The remaining problem is to obtain a set of warm starts. 
%pivotal question is, {how do we obtain such a warm-start set of points?} 
It turns out that we can exploit the score function learned in the previous iteration of the algorithm (for a larger $t$) to generate these warm-start points (\Cref{lem:scorewarmstart}, \Cref{a:warm-starts} in Section~\ref{s:warm}); the idea is that the score function exactly points toward the ``denoising'' direction, i.e. approximately towards points sampled form the mixing measure $Q_0$. Therefore, adding the score function estimate $s$ to sufficiently many noisy samples $y_i$ is sufficient for obtaining a suitable set of warm starts:
\[
\wh \mu_i = y_i + \si^2 s(y_i).
\]
The $\wh \mu_i$ can be greedily clustered and subselected to keep the set of warm starts small. 
Based on this observation, we propose an inductive recipe to iteratively estimate the score functions for different noise levels, while maintaining a warm-start set.

\iffalse
Imagine running the backwards diffusion process from pure Gaussian noise back to our mixture: the idea is to inductively maintain ``warm starts" to all the Gaussian centers at the current resolution, and perform piecewise polynomial regression in Voronoi cells centered at these warm starts. 
If we have the warm starts, then we can approximate the score function with ``local" score functions on Voronoi cells (\Cref{l:nearby-score} in \Cref{ss:nearby-score}) and piggybacking off the results of \Cref{s:ns}, obtain a low-degree piecewise polynomial approximation (\Cref{lem:lowdegree} in \Cref{ss:voronoi}). 
We analyze the sample complexity of piecewise polynomial regression in \Cref{ss:ppoly} (\Cref{lem:gengap}) and address the problem of maintaining warm starts in Section~\ref{s:warm} (\Cref{lem:scorewarmstart}); this we get inductively from the score estimates as they approximately point in the ``denoising" direction, i.e., towards these centers.
With all these pieces, we can then prove the main theorem.
\fi

\section{Diffusion models}
\label{s:dm}

Using the framework of diffusion models, the algorithm and analysis consist of two main parts: showing that we can learn this score function efficiently, and showing that an accurate score function allows generating more samples from the distribution. In this section, we focus on the second part, which has been well-studied.

\subsection{Convergence guarantees for diffusion models}

Two popular parameterizations of diffusion models are the variance-exploding (VE) and variance-preserving (VP) processes. The diffusion models are also known as score matching with Langevin dynamics (SMLD) and denoising diffusion probabilistic models (DDPM), respectively. The forward SDE for the VE process is simply Brownian motion, while the forward SDE for the VP process is the Ornstein-Uhlenbeck process. Below, we lay out the forward SDE for each of these process, and the backward SDE obtained by \cite{anderson1982reverse}:
\begin{align*}
\text{VE:} && \text{VP:}\\
    dy_t &= dW_t &dx_t &= -x_t \,dt + \sqrt 2 \,dW_t\\
    dy_t &= -\nb \ln p_t(y_t)\,dt + \,d\wt W_t&
    dx_t &= (-x_t - 2\nb \ln p_t(x_t)) \,dt + \sqrt 2 \,d\wt W_t \\
    d\yl_t &= \nb \ln p_{T-t}(\yl_t)\,dt + \,dW_t' & 
    d\xl_t &= (\xl_t + 2\nb \ln p_{T-t}(\xl_t)) \,dt + \sqrt 2 \,d W_t'
\end{align*}
where $p_t$ is the density of the either $x_t$ or $y_t$, $\wt W_t$ is reverse Brownian motion, and $W_t$, $W_t'$ are usual Browian motions, and the reverse processes are initialized at $\yl_0 \sim p_T$, $\xl_0\sim p_T$ and we can match up trajectories $\yl_t = y_{T-t}$, $\xl_t = x_{T-t}$ for $t\in [0,T)$. For convenience, we work with the VE process. As most convergence guarantees in the literature are for the VP process, we need to adapt those results. Note that as continuous processes we have $x_t = e^{-t} y_{e^{2t}-1}$, but some care is required when matching up the discretizations.

The following is an adaptation of \cite[Theorem 1]{benton2023linear}, after reparameterizing the VP into the VE process. 

\begin{thm}[Reverse KL guarantee for variance-exploding diffusion models]\label{t:dm-kl}
Let $0<t_1<t_2<\ldots < t_{\Nstep}=T$ and $\zeta_k = t_{k+1}-t_k$. Suppose $T\ge 1$.  
%Consider the following assumptions.
Assume the following.
\begin{enumerate}
    \item We have a score function estimate $s_{t_k}(y)$ for every $t_k$ such that 
    \[
    \E_{P_{t_k}}\vn{\nb \ln  p_{t_k}(y) - s_{t_k}(y)}^2 \le \ep_k^2.
    \]
    \item The data distribution has bounded second moment $M_2 = \E_{P_0}\vn{y}^2$.
    \item For some $\ka<1$, the 
    step size schedule satisfies
    \[
t_k+1 \ge 
(t_{k+1}+1)
\max\bc{e^{-2\ka}, (t_{k+1}+1)^{-\ka}}.
    \]
\end{enumerate}
Let $\wh p_{t}$ denote the distribution of the following discretization of the reverse process when initialized at $\wh p_T = \cal N(0,(T+1)\cdot I_n)$:\footnote{It would be more natural to have $h_k$ as the coefficient in front of the score term. This coefficient comes from doing a change-of-variable from \Cref{t:dm-kl-orig} for the DDPM process, rather than reproving the theorem for this process.}
\[
 \wh y_{t_k} = \wh y_{t_{k+1}} + %h
        2\ba{(t_{k+1}-1) - \sqrt{(t_k-1)(t_{k+1}-1)}}
        s_{t_{k+1}}(\wh y_{t_{k+1}}) + \sqrt{\zeta_k}\cdot  \xi_{t_k}.
\]
Then
\begin{align}
\label{e:dm-kl}
\KL(p_{t_1}\|\wh p_{t_1}) 
\lesssim \fc{n+M_2}{T+1} + \sumo k{\Nstep-1} \ln \pf{t_{k+1}+1}{t_k+1} \cdot \rc{t_k+1}
    \cdot \ep_{k+1}^2 + 
\ka n \ln (T+1) + \ka^2 n\Nstep + \ka M_2.
\end{align}
Moreover, we can choose a schedule of length $\Nstep=O\pa{\rc{\ka} \ln\pf{T+1}{t_1}}$ to make assumption 3 hold. If we choose $T=\fc{M_2+n}{\ep^2}$ and $\ka = \fc{\ep^2}{M_2+n\ln (T+1)}$, and we have $\ep_k^2\le \fc{\ep^2(t_k+1)}{\ln (T+1)}$ for each $k$, then the error is $O(\ep^2)$.
\end{thm}

The proof is deferred to \Cref{s:dm-appendix}.

\subsection{Score function computation for Gaussian mixture}
\label{ss:score-gm}

To learn a distribution using a standard diffusion model, we need to learn its score function (gradient of the log of the pdf) under Gaussian convolution, with varying levels of noise. For Gaussian mixtures, the score function has a particularly nice form, which we now derive.

First, we let $P_t=P_0 \convolve \cal N(0,tI_n)$ and $Q_t = Q_0 \convolve \cal N(0,tI_n)$. Then $P_t = Q_0 \convolve \cal N(0,\si^2 I_n) = Q_{\si^2}$ where $\si^2 = \si_0^2+t$. 
We consider the following probabilistic model where $\mu, \xi_1, \xi_2$ are drawn independently:
\begin{align}\label{e:gen}
\mu \sim Q_0, \quad 
\xi_1, \xi_2 \sim \cal N(0,I_n), \quad
X = \mu + \si_0 \xi_1, \quad 
Y = X + \sqrt t \xi_2 = \mu + \si_0 \xi_1 + \sqrt t \xi_2 .
\end{align}
Then $X\sim P_0$ and $Y\sim P_t$. 
Letting $p_t, q_t$ be the corresponding densities and $V_t=\ln p_t$, we have by Tweedie's formula \cite{robbins1992empirical} (see \Cref{s:tweedie} for a derivation) that 
%\holden{Can give a self-contained derivation in the appendix}
\begin{align}
\label{e:score-Q}
\nb V_t(y) &=
\rc{\si^2} \E [\mu-y | Y=y] = 
\rc{\si^2}\pa{-y +  \fc{\int_{\R^n} \mu \exp\pa{\fc{\an{y,\mu}}{\si^2} - \fc{\vn{\mu}^2}{2\si^2}}\,dQ_0(\mu)}{\int_{\R^n} \exp\pa{\fc{\an{y,\mu}}{\si^2} - \fc{\vn{\mu}^2}{2\si^2}}\,dQ_0(\mu)}}\\
\label{e:score-Q2}
f_{\si^2}(y) \defeq y+ \si^2 \nb V_t(y) &= \E [\mu | Y=y] 
=\fc{\int_{\R^n} \mu \exp\pa{\fc{\an{y,\mu}}{\si^2} - \fc{\vn{\mu}^2}{2\si^2}}\,dQ_0(\mu)}{\int_{\R^n} \exp\pa{\fc{\an{y,\mu}}{\si^2} - \fc{\vn{\mu}^2}{2\si^2}}\,dQ_0(\mu)}.
\end{align}
Note that in the same way we have 
\begin{align}
\label{e:score-P}
    \nb V_t(y) &=
\rc{t} \E [X-y | Y=y] = 
\rc{t}\pa{-y +  \fc{\int_{\R^n} x \exp\pa{\fc{\an{y,x}}{t} - \fc{\vn{x}^2}{2t}}\,dP_0(x)}{\int_{\R^n} \exp\pa{\fc{\an{y,x}}{t} - \fc{\vn{x}^2}{2t}}\,dP_0(x)}}\\
\label{e:score-P2}
y+ t \nb V_t(y) &= \E [X | Y=y] 
=\fc{\int_{\R^n} x \exp\pa{\fc{\an{y,x}}{t} - \fc{\vn{x}^2}{2t}}\,dP_0(x)}{\int_{\R^n} \exp\pa{\fc{\an{y,x}}{t} - \fc{\vn{x}^2}{2t}}\,dP_0(x)}.
\end{align}
Because of this identity, the score function can be learned as the minimal mean squared estimator in a (supervised) denoising problem of estimating $X$ given $Y$, known as the denoising auto-encoder (DAE) objective \cite{vincent2011connection}. \eqref{e:score-Q}--\eqref{e:score-Q2} will be useful for analysis while \eqref{e:score-P}--\eqref{e:score-P2} is used for the actual learning algorithm. Because of the nice form of \eqref{e:score-Q2}, we will actually aim to learn $f_{\si^2}$. We remark that in the case of a finite mixture, \eqref{e:score-Q2} shows that $f_{\si^2}$ is represented by a softmax neural network with 1 hidden layer of $k$ units. 

\section{Learning the score for a single cluster}
%{Noise sensitivity of score function for a cluster}
\label{s:ns}
We will follow the approach in \cite{klivans2008learning}, though with functions that are $\R^n$-valued rather than $\{0,1\}$-valued. 

First, we give some background on the Ornstein-Uhlenbeck process. 
Let $\ga\defeq\ga_{\si^2}$ denote the density of $\cal N(0,\si^2 I_n)$.
The generator $\sL\defeq\sL_{\si^2}$ of the scaled Ornstein-Uhlenbeck process with variance $\si^2$, also known as Langevin dynamics for $\ga$, is
\begin{equation}
\sL f(x) = -\rc{\si^2}\an{x, \nb f(x)} + \De f(x).
\label{e:L}
\end{equation}
For a $\R^d$-valued function $f$, we interpret this componentwise, i.e.,
\begin{align*}
    \sL f (x) &= - \rc{\si^2} Df(x) x + \sumo id \pl_{ii}f.
\end{align*}
The eigenfunctions are the (suitably scaled) Hermite polynomials $(h_{\mathbf k})_{\mathbf k\in \N_0^n}$, with $h_{\mathbf k}:= h_{\mk, \si^2}$, as defined in~\Cref{sec:notation}, having eigenvalue $-\fc{|\mathbf k|}{\si^2}$. Here we use $|\mk|$ to denote $|\mk|_1$. The Hermite polynomials form a complete orthogonal basis %eigenbasis 
for $L^2(\ga)$.  
The scaled Ornstein-Uhlenbeck process can be described by the SDE
\begin{align}\label{e:ou-soln}
dx_t &= -\rc{\si^2} x_t + \sqrt{2} \,dW_t.
\end{align}
Let $(\sP_t)_{t\ge 0}$ be the Markov semigroup of the scaled Ornstein-Uhlenbeck process, with generator given by \eqref{e:L}. That is, we have 
\begin{align*}
\numberthis \label{e:sPt}
    \sP_t f(x) &= \E_{x_0=x} f(x_t) \text{ when $x_t$ solves \eqref{e:ou-soln}},\\
    \sL f & = \lim_{t\to 0^+} \fc{\sP_t f - f}{t}.
\end{align*}

The following encapsulates the technique of noise sensitivity/stability: a bound on the $L^2(\ga)$ norm of $\sL^m f$ implies approximability of $f$ by a low-degree polynomial.  
To our knowledge, only the case $m=1$ has been considered in the literature;
however, similarly to how bounds on higher moments imply better tail bounds, bounding $\sL^m f$ for larger $m$ can give better approximation.
\begin{lem}[Noise stability implies low-degree approximability]\label{l:ns-ld}
Let $\ga$ denote the density of $\cal N(0,\si^2 I_n)$. 
    Suppose that $f:\R^d\to \R^n$ %has finite $L^2$ norm under $N(0,1)$
    satisfies $\ve{f}_{L^2(\ga)}<\iy$, and that $m\in \N, L$ are such that $ \ve{\sL^{m} f}_{L^2(\ga)}\le L^{m}$. For $d\in \N$, there exists a polynomial $g$ of degree $<d$ such that 
    \[
\ve{f-g}_{L^2(\ga)}
\le \pf{L\si^2}{d}^{m}.
    \]
\end{lem}
\begin{proof}
Expand $f$ in the eigenfunction basis of $\sL_{\si^2}$ as
$f = \sum_{\mathbf k\in \N_0^d} a_{\mk}  h_{\mk}$ where $a_{\mk}\in \R^d$. 
    Then
    \begin{align*}
        \sL^{m} f &= \sum_{\mathbf k\in \N_0^n} \fc{|\mk|^{m}}{\si^{2m}} a_{\mk} h_{\mk}.
    \end{align*}
    %\khasha{where we first introduce $\sL$ it shoudl be defined what it means to iteratively apply the operator}
    Hence
    \begin{align*}
    \fc{d^{2m}}{\si^{4m}}\sum_{|\mk|\ge d} \vn{a_{\mk}}^2\le 
    \sum_{\mk\in \N_0^n} \fc{|\mk|^{2m}}{\si^{4m}} \vn{a_{\mk}}^2 \le L^{2m} 
    \implies
    \sum_{|\mk|\ge d} \vn{a_{\mk}}^2 \le \pf{L\si^2}{d}^{2m}.
    \end{align*}
    %\holden{In general: $\pf{L\si^2}{d}^{2m}$}
    Taking $g = \sum_{|\mk|< d} a_{\mk} h_{\mk}$, we have that $\ve{f-g}_{L^2(\ga)}^2 = \sum_{|\mk|\ge d} \vn{a_{\mk}}^2$, which gives the desired bound.
\end{proof}

\subsection{Calculation of $\sL^m f$}\label{ss:Lmf}

Let $\nu$ be a measure on $\R^n$ with all moments finite. 
Consider generating $y=x+\si \xi$ where $x\sim \nu$ and $\xi\sim \cal N(0,I_n)$.
Let $\an{\cdot} = \an{\cdot}_y \defeq \E_{P(x|y)}$ where $P(x|y)$ is the posterior distribution given by
\begin{align*}
    \fc{dP(\cdot|y)}{d\nu}(x) &\propto \exp\pa{\fc{\an{y,x}}{\si^2} - \fc{\ve{x}^2}{2\si^2}}.
\end{align*}
(This is not to be confused with the inner product.) 
Then letting $f=f_{\si^2}$ and $\nu=Q_0$ as in~\eqref{e:score-Q2}, we have $f(y):=f_{\si^2}(y) = \an{x}_y$.

We first derive some formulas for differentiating posterior expectations $\an{\cdot}$. 
Let $\wt x = x-\an{x}_y$ denote the centered random variable. 
For a function $g:\R^n\to \R$, we have the following general formula for differentiation with respect to $y_i$:
\begin{align*}
\si^2\pl_i \an{g(x)}_y &= 
\fc{\int_{\R^n} g(x) x_i \exp\pa{\fc{\an{y, x}}{\si^2} - \fc{\vn{x}^2}{2\si^2}} \nu(dx)}{\int_{\R^n} \exp\pa{\fc{\an{y, x}}{\si^2} - \fc{\vn{x}^2}{2\si^2}} \nu(dx)}\\
&\quad 
- \fc{\int_{\R^n} g(x)  \exp\pa{\fc{\an{y, x}}{\si^2} - \fc{\vn{x}^2}{2\si^2}} \nu(dx)\int_{\R^n} x_i  \exp\pa{\fc{\an{y, x}}{\si^2} - \fc{\vn{x}^2}{2\si^2}} \nu(dx)}{\pa{\int_{\R^n} \exp\pa{\fc{\an{y, x}}{\si^2} - \fc{\vn{x}^2}{2\si^2}} \nu(dx)}^2}\\
&=
\an{g(x) \wt x_i}_y\\
\nb \an{g(x)}_y &=\rc{\si^2} \an{g(x)\wt x}_y.
\end{align*}
More generally, let $\an{g(x^{(1)},\ldots, x^{(r)})}_y$ denote $\E g(x^{(1)},\ldots, x^{(r)})$ where $x^{(1)},\ldots, x^{(r)}$ are independent draws from the posterior $p(\cdot|y)$. 
We use the trick of replacing the means with independent copies of the random variable (see e.g., \cite{talagrand2010mean}) to obtain
\begin{align*}
    \si^2\nb \an{g(x^{(1)},\ldots, x^{(r)})} 
    &= 
    \an{g(x^{(1)},\ldots, x^{(r)})(\wt x^{(1)} + \cdots + \wt x^{(r)})}\\
    &= 
    \an{g(x^{(1)},\ldots, x^{(r)})( x^{(1)} + \cdots +  x^{(r)} - r x^{(r+1)})}.
\end{align*}
and for a function $h:\R^r\to \R$,
\begin{align*}
    &\nb \an{g(x^{(1)},\ldots, x^{(r)}) h\pa{\an{x^{(1)}, y},\ldots, \an{x^{(r)}, y}}}\\
    &= 
    \rc{\si^2}\an{g(x^{(1)},\ldots, x^{(r)}) h\pa{\an{x^{(1)}, y},\ldots, \an{x^{(r)}, y}} 
    \pa{x^{(1)} + \cdots + x^{(r)} 
     - rx^{(r+1)}}}\\
     &\quad + \an{g(x^{(1)},\ldots, x^{(r)}) \sum_j \pl_j h\pa{\an{x^{(1)}, y},\ldots, \an{x^{(r)}, y}} x^{(j)}
    }.
\end{align*}
Now suppose $h$ is a homogeneous polynomial of degree $s$. Then 
\begin{multline}
\label{e:D1}
    %&
    D \an{g(x^{(1)},\ldots, x^{(r)}) h\pa{\an{x^{(1)}, y},\ldots, \an{x^{(r)}, y}}} y\\
    %&
    = \rc{\si^2}
    \an{g(x^{(1)},\ldots, x^{(r)}) h\pa{\an{x^{(1)}, y},\ldots, \an{x^{(r)}, y}} 
    \pa{\an{x^{(1)},y} + \cdots + \an{x^{(r)},y} 
     - r\an{x^{(r+1)},y}}}\\
     %&
     \quad 
     + s \an{g(x^{(1)},\ldots, x^{(r)})h\pa{\an{x^{(1)}, y},\ldots, \an{x^{(r)}, y}}}
\end{multline}
by Euler's formula $\sumo jr\pl_j h(x_1,\ldots, x_j) x_j = s\cdot h$.

Now consider $h(z_1,\ldots, z_r)=\prod_{\ell=1}^t z_{j_\ell}$. Let $u(x^{(1)},\ldots, x^{(r)}) = g(x^{(1)},\ldots, x^{(r)}) h\pa{\an{x^{(1)}, y},\ldots, \an{x^{(r)}, y}}$.
We have
\begin{multline}
\label{e:D2}    
    \De \an{u(x^{(1)},\ldots, x^{(r)})}
    %&
    = \rc{\si^4}
    \an{u \cdot \an{x^{(1)}+\cdots + x^{(r)} - rx^{(r+1)} , 
    x^{(1)}+\cdots + x^{(r+1)} - (r+1)x^{(r+2)}}} \\
    %&
    \quad + 
    \fc{2}{\si^2}\an{g \cdot %\sum_{(i_\ell) \in h}
    \sum_{\ell'=1}^s \prod_{\ell \ne \ell'} 
    \an{x^{(i_\ell)}, y} \cdot 
    \an{x^{(i_{\ell'})}, x^{(1)}+\cdots + x^{(r)} - rx^{(r+1)}}
    }\\
    %&
    \quad + 
    \an{g \cdot %\sum_{(i_\ell)\in h}
    \sumr{1\le \ell', \ell''\le s}{\ell'\ne \ell''} \prod_{\ell \ne \ell', \ell''} 
    \an{x^{(i_\ell)}, y} \cdot 
    \an{x^{(i_{\ell'})}, x^{(i_{\ell''})}}
    }
\end{multline}
Applying the above to monomials $g(x^{(1)},\ldots, x^{(r)}) = \prod_{\ell=1}^s \an{x^{(i_\ell)}, x^{(i_\ell')}}$ 
and using induction, 
we have the following.
\begin{lem}[Probabilistic interpretation of $\sL^mf$]\label{l:Ld}
Let $f=f_{\si^2}$ be as in \eqref{e:score-Q2}.
We have
    \begin{align}
    \label{e:Ld}
\sL^m f (y)= 
\an{x^{(1)}
\sum_{s+t\le m}
\sum_{i,i'\in [2m+1]^s, j\in [2m+1]^t}
a_{i,i',j}\si^{-2(s+t+m)}
\prod_{\ell=1}^s \an{x^{(i_\ell)}, x^{(i_\ell')}} \prod_{\ell=1}^t \an{x^{(j_\ell)}, y}
}
    \end{align}
    where $\sum_{i,i',j} |a_{i,i',j}|\le 30^m m!^2$.
\end{lem}

\begin{proof}
    Recall that $\sL f (x) = - \rc{\si^2}Df(x) x + \De f(x)$. We induct on $m$. Suppose that the lemma holds for $m-1$; each term has at most a number of replicas $r\le 2m-1$ and $s,t\le m-1$. 
    We consider the effect of the map $f\mapsto \rc{\si^2}Df(y) \cdot y$~\eqref{e:D1} and $\De$~\eqref{e:D2} on a single term in~\eqref{e:Ld}. 
    First, note that in each case, we obtain a factor $\rc{\si^2}$, and a factor of $\rc{\si^2}$ for each additional $\an{x^{(i_\ell)}, x^{(i_\ell')}}$ term as well as $\an{x^{(i_\ell)}, y}$ term (with a $\si^2$ factor if we remove a term); this justifies the $\si^{-2(s+t+m)}$ factor. 
    \begin{enumerate}
        \item In~\eqref{e:D1}, the number of replicas $r$ in the resulting terms increases by at most 1, $t$ increases by at most 1, and the sum of absolute value of coefficients is at most \[2r + s \le 
        2(2m-1) + (m-1) \le 5m-3.\]
        \item 
        In~\eqref{e:D2}, the number of replicas $r$ in the resulting terms increases by at most 2, $s$ increases by at most 1, and the sum of absolute value of coefficients is at most
        \begin{align*}
&2r\cdot 2(r+1) + 2\cdot s\cdot 2r + s(s-1)\\
&\le 
4\cdot (2m-1)(2m+1)
+ 2\cdot (m-1) \cdot 2(2m-1)
+ (m-1)(m-2)\\
&\le
16m^2 + 8m^2 + m^2
\le 25m^2
        \end{align*}
    \end{enumerate}
    The sum of absolute value of coefficients multiplies by at most $30m^2$. This finishes the induction step.
\end{proof}

%Hence

\subsection{Bounding $\sL^mf$}
\label{ss:bd-Lm}

From~\Cref{l:ns-ld}, we know that that a bound  on $\ve{\sL^{2m} f}_{L^p(\ga_{\si^2})}$ implies a low-degree approximation for $f$. Here we would like to bound $\ve{\sL^m f}_{L^p(\ga_{\si^2})}$ for all $m$. The kind of growth we get in $m$ is captured by the following definition.
\begin{df}
    We say $f: \mathbb R^n \rightarrow \mathbb R^{n'}$  is a \vocab{$(r,\sigma)$-Gaussian-noise-sensitive function} if for all $m \in \mathbb N$ and $p \geq 1$, 
    \begin{align*}
        %\ve{f}_{W^{m,p}_{\sL_{\si^2}}} \defeq 
        \ve{\sL^m f}_{L^p(\gamma_{\sigma^2})}
        \leq r\sigma\pa{\frac{rm^2}{\sigma^{2}}}^m \max\bc{r, \sqrt{mp}}^m.
    \end{align*}
\end{df}

In the following lemma, we bound the moments of higher iterates of the Ornstein-Uhlenbeck operator $\sL$ applied to the score function $f$ under the Gaussian measure. This shows the Gaussian-noise-sensitivity property for $f$ in the case where $Q_0$ is supported on a ball of radius $R$, forming a single cluster.
We begin by switching from the Gaussian measure to the mixture measure using Lemma~\ref{l:com-poly}, and then leverage the probabilistic interpretation of $\sL^m f$ derived in~\Cref{l:Ld}, which is based on moments of replicas of the posterior distribution.

\begin{lem}[Control on iterates of the OU operator]\label{lem:csbound}
Let $f=f_{\si^2}$ be as in \eqref{e:score-Q2}.
    Suppose $Q_0$ is supported on $B_R(0)$, $R\ge \si$. Let $P$ be the density of $Q_0 \convolve \cal N(0,\si^2I_n)$.
    %\khasha{need to change the $p$ notation here} where $\ga_{\si^2}$ is the density of $\cal N(0,\si^2I_n)$. 
    Then for any $m\in \N$, we have the following:
\begin{align*}
\ve{\sL^m f}_{L^p(P)}
&= 
R \cdot 
O\pa{\pf{1}{\si}^2 \pf{mR}{\si}^2 \pa{1+\fc{(mp)^{1/2}\si}R}}^m\\
\ve{\sL^m f}_{L^p(\ga_{\si^2})} &\le
R\cdot O\pa{
\pf{1}{\si}^2 \pf{m^2R}{\si} \max\bc{\fc{R}{\si}, (mp)^{1/2}}
}^m
%\]
%}
\end{align*}
Therefore, $f$ is $\pa{O\pf R\si, \si}$-Gaussian-noise-sensitive.
\end{lem}
%\holden{Note dependence on $\si$ (not canceled out by $R$) is $\rc{\si^{4d}}=\be^{2d}$. Potentially loose in $m$. (Can we make it $m^2$ in the base?)}
\begin{proof}
In~\eqref{e:Ld} in Lemma~\ref{l:Ld}, 
\[
\ab{\prod_{\ell=1}^s \an{x^{(i_\ell)}, x^{(i_\ell')}}} \le 
R^{2s}\le R^{2m}.
%\prod_{m=1}^t \an{x^{(j_m)}, y}
\]
Note that %by Nishimori's identity, 
the joint distribution of $(x^{(j)}, y)$ is the same for any $j$, namely, it is the distribution when $x^{(j)}\sim p_0$ and $y=x^{(j)} + \si \xi^{(j)}$ when $\xi\sim \cal N(0,I_n)$ is independent of $x^{(j)}$. 
Then for $p\in \N$, 
\begin{align*}
&\E
    \an{\ab{\prodo {\ell}t \an{x^{(j_\ell)}, y}}}^p\\
    &\le 
    \E_{\trow{x^{(1)} \sim p_0, y = x^{(1)}+\si \xi}{x^{(r)} \sim p(\cdot |y),r>1}}
    \ab{\prod_{\ell=1}^t \an{x^{(j_\ell)}, y}}^p& \text{by Jensen's inequality}\\
    &\le \prodo {\ell}t 
    \ba{\E\ab{\an{x^{(j_\ell)},y}}^{tp}}^{1/t}&\text{by H\"older's inequality}\\
    &\le \E\ab{\an{x^{(1)},x^{(1)} + \si \xi^{(1)}}}^{tp}& (x^{(j)},y) \stackrel d= (x^{(1)}, x^{(1)} + \si \xi^{(1)})\\%\text{Nishimori's identity}\\
    &\le \E
    \sum_{k=0}^{tp} 
    \binom{tp}k \vn{x^{(1)}}^{2(tp-k)} \si^k \ab{\an{x^{(1)}, \xi^{(1)}}}^{k}&\text{Binomial theorem}\\
    &\le \sum_{k=0}^{tp} 
    \binom{tp}k R^{2(tp-k)} \si^k R^k\E_{X\sim \cal N(0,1)}|X|^{k}\\
    &\le \sum_{k=0}^{tp} 
    \binom{tp}k R^{2(tp-k)} \si^k R^k (k-1)!!&\text{Gaussian moment bound}\\
    &\le \sum_{k=0}^{tp} 
    \binom{tp}k R^{2(tp-k)} \si^k R^k
    (tp)^{k/2}
    \le R^{2tp}\pa{1+\fc{(tp)^{1/2}\si}{R}}^{tp}.
    %&\le \sum_{k=0}^{tp} R^{2tp} \pf{\si tp}{R}^k
    %\le (tp+1) R^{2tp} \pa{1 + \pf{tp\si}{R}^{tp}}.
\end{align*}
Consider one term $a_{i,i',j}\an{x^{(1)} \si^{-2(s+t+m)}
\prod_{\ell=1}^s \an{x^{(i_\ell)}, x^{(i_\ell')}} \prod_{\ell=1}^t \an{x^{(j_\ell)}, y}}$ in \eqref{e:Ld}. We have by Jensen's inequality that 
\begin{align*}
    &\ve{\an{x^{(1)} \si^{-2(s+t+m)}
\prod_{\ell=1}^s \an{x^{(i_\ell)}, x^{(i_\ell')}} \prod_{\ell=1}^t \an{x^{(j_\ell)}, y}}}_{L^p(P)}\\
&\le \ba{\E \an{\ab{x^{(1)} \si^{-2(s+t+m)}
\prod_{\ell=1}^s \an{x^{(i_\ell)}, x^{(i_\ell')}} \prod_{\ell=1}^t \an{x^{(j_\ell)}, y}}^p}}^{\rc p}\\
&\le \si^{-2(s+t+m)} R^{2s+1} %(tp+1)^{1/p} 
R^{2t}\pa{1 + \fc{(tp)^{1/2}\si}{R}}^{t}\\
&= O\pa{\fc{R}{\si^{2m}} \pf{R}{\si}^{2(s+t)} \pa{1+\fc{(tp)^{1/2}\si}R}^t}= O\pa{\fc{R}{\si^{2m}} \pf{R}{\si}^{2m} \pa{1+\fc{(mp)^{1/2}\si}R}^m}
\end{align*}
using $s+t\le m$. 
Then 
\begin{align}
\nonumber
    \ve{\sL^m f}_{L^p(P)}
    &\le 30^{m} m!^2 
    \cdot O\pa{\fc{R}{\si^{2m}} \pf{R}{\si}^{2m} \pa{1+\fc{(mp)^{1/2}\si}R}^m}\\
    &= %O(d)^{2d} \cdot 
    R\cdot O\pa{\pf{1
    }{\si}^{2m} \pf{mR}{\si}^{2m} \pa{1+\fc{(mp)^{1/2}\si}R}^m}
    %R^{2d} (d+1)^2\pa{R^{2d}+(\si d)^{2d}}.
    \label{e:Ldfp}
\end{align}
By H\"older's inequality and Lemma~\ref{l:chi-a} (appropriately scaled),
\begin{align*}
\ve{\sL^m f}_{L^p(\ga_{\si^2})} 
% = 
%     \pa{\int_{\R^n} 
%     |\sL^{m} f|^p\,d\ga}^{\rc p}
    &=
    \pa{\int_{\R^n} 
    |\sL^{m} f|^p\,\fc{d\ga_{\si^2}}{dP}dP}^{\rc p}
    \le 
    \ve{\sL^m f}_{L^{p(1+q)}(\ga_{\si^2})}
    \ve{\dd{\ga_{\si^2}}{P}}_{L^{1+\rc q}(P)}^{\rc p}\\
    &\le O\pa{\prc{\si}^2 \pf{mR}{\si}^2 \pa{1+\fc{(mp(q+1))^{1/2}\si}R}}^{m} e^{\fc{(R/\si)^2}{2pq}}.
\end{align*}
To optimize this bound, set $q = \frac{(R/\sigma)^2}{pm}$. Then by H\"older's inequality,
    \begin{align*}
        \ve{\sL^m f}_{L^p(\ga_{\si^2})}& \le
    Re^{\fc{(R/\si)^2}{2pq}}\cdot 
    O\pa{\pf{1}{\si}^2 \pf{mR}{\si}^2 \pa{1+\fc{mp(q+1)\si}{R}}}^m
    \\
    &\le
    Re^{m/2}\cdot  
    O\pa{\pf{1}{\si}^2 \pf{mR}{\si}^2}^m
    \cdot 
    \ba{1 + O\pf{(mpq)^{1/2}\si}{R}^m + O \pf{(mp)^{1/2}\si}{R}^m}
    \end{align*}
The first two terms give $R\cdot O\pa{\prc{\si}^2 \pf{mR}{\si}^2}^m$, while the last term gives 
$R\cdot O\pa{\prc{\si}^2 \pf{m^2R}{\si} (mp)^{1/2}}$. This completes the proof.
\end{proof}

\subsection{Approximation with polynomial}\label{ss:approx-poly}

%Without loss of generality, suppose $\si_0^2=1$. 
Suppose we want to approximate $f=f_{\si^2}$ (in~\eqref{e:score-Q2}) with a low-degree polynomial on the mixture distribution $\ga'$. %Let $\ga$ be the Gaussian centered at the warm start; without loss of generality, $\cal N(0,I_n)$. 
We seek a low-degree polynomial $g$ such that $\ve{f-g}_{\ga'}^2$ is small. 
To do this, we first smooth $f$ by the Ornstein-Uhlenbeck semigroup to obtain $\sP_t f$ (see \eqref{e:sPt}), and then find a $g$ that approximates $\sP_tf$. 
We can bound using Cauchy-Schwarz that
\begin{align*}
\ve{f-g}_{\ga'}^2
&\le 2\pa{\ve{f-\sP_t f}_{\ga'}^2 + \ve{\sP_t f - g}_{\ga'}^2}\\
&\le 2\pa{\ve{f-\sP_t f}_{\ga'}^2 + 
\ve{\sP_t f - g}_{L^2(\ga)}^2
+
\ve{\sP_t f - g}_{L^4(\ga)}^2
\chi^2(\ga'\|\ga)^{1/2}
}.
\end{align*}
Doing the smoothing ensures that we have better control over the term with higher $p$-norm, $\ve{\sP_t f - g}_{L^4(\ga)}$.
Choosing $t=\ep$, $\sP_tf$ has exponential decay in coefficients with rate $\ep$, so we can approximate $\sP_tf$ with a degree-$\Te\prc{\ep}$ polynomial.
To bound 
$\ve{f-\sP_tf}_{\ga'}^2$, 
%$\ve{\sP_tf - g}_{L^2(\ga)}^2$, 
we bound its derivative:
\begin{align*}
\ddd t \ve{f-\sP_tf}_{\ga'}^2
&\le \int_{\R^n} 2(f-\sP_tf) (-\sL \sP_tf) \,d\ga'\\
&\le 2\ve{f-\sP_tf}_{\ga'} \ve{\sL \sP_tf}_{\ga'} \\
\implies 
\ddd t \ve{f-\sP_tf}_{\ga'}
&\le \ve{\sL \sP_tf}_{\ga'}.
% \\
% \implies 
% \ve{f-\sP_tf}_{\ga'} &\le \ve{\sL f}_{\ga'}t.
\end{align*}
%Thus we can choose $t=\ep$ to get a reasonable bound. 
Focusing on the dependence of total error on $\ep$, we can bound this with a change-of-measure inequality and~\Cref{lem:csbound}, which, after integrating in $t$, gives error $O(\ep)$. 
In order to obtain poly-logarithmic degree, we need a higher-order version of this argument. In preparation for \Cref{s:multiple}, when we need to consider the norm with respect to a different measure, we state the following more generally.
%take on the order of $\rc{\ep}$ coefficients. 

\begin{lem}\label{l:poly-approx-mix}
Suppose $f$ is $(O(r),\sigma)$-Gaussian-noise-stable and $\ga'$ is a measure such that for all $a\ge 0$, $\ve{\dd{\ga'}{\ga}}_{L^{1+a}(\ga)} \le e^{\fc{ar^2}2}$ (for $\gamma \defeq \gamma_{\sigma^2}$) (e.g., from Lemma~\ref{l:chi-a}, $Q\convolve \cal N(0,I_n)$ where $Q$ is supported on $B_{r}(0)$), where $r\ge 1$.  
%$Q_0$ is supported on $B_R(0)$, where $R\ge 1$. 
    There is a polynomial $g$ of degree at most 
    $O\pa{r^4\ln\prc\ep^4 \max\bc{r^2, \ln \prc\ep}}$
    such that 
\[
\ve{f-g}_{\ga'}^2\le \ep^2r^2\sigma^2 \quad \text{and} \quad \ve{g}_{\gamma}^2 \leq \ve{f}_{\gamma}^2.
\]
\end{lem}

\begin{proof}
To get a better bound, we use a more clever smoothing strategy. For some $\wt f$ to be defined, we will bound
\begin{align}
\label{e:smooth-triangle}
\ve{f-g}_{\ga'}^2 
\le 2\pa{\ve{f-\wt f}_{\ga'}^2 + \ve{\wt f-g}_{\ga'}^2},
\end{align}
where $g$ is a polynomial approximation of $\wt f$ obtained by truncating the Hermite expansion.

To define $\wt f$, we approximate $f$ with a numerical differentiation formula for a higher derivative. Define the finite difference by $\De_{x,h} g(x) = \De_{h}g(x) \defeq g(x+h)-g(x)$. We can write this as $\De_h g = T_h g-g$, where $T_hg(x)\defeq g(x+h)$. 
Suppose that $g\in C^m$.
By the Binomial Theorem on $T_h-\id$ and Taylor's Theorem,
\begin{align*}
\De_h^m g(0) &=
\sumz jm \binom mj (-1)^{m-j} \ba{\sumz i{m-1} g^{(i)}(0) \fc{(hj)^i}{i!} + \fc{g^{(m)}(\xi_j)}{m!} (hj)^m} \text{ for some }\xi_j\in [0,hj]\\
&= 
\sumz i{m-1} \fc{g^{(i)}(0)}{i!} \De_{x,h}^m (x^i)|_{x=0} + h^m \sumz jm \binom mj (-1)^{m-j} \fc{g^{(j)}(\xi_j)}{m!}j^m\\
&=
h^m \sumz jm \binom mj (-1)^{m-j} \fc{g^{(j)}(\xi_j)}{m!}j^m
\end{align*}
where in the last step we use the fact that finite differencing reduces the degree of a polynomial by 1, the $m$th finite difference of polynomials of degree $<m$ is 0. Hence
\[
|\De_h^m g(0)|\le
h^m \sumz jm \binom mj \pf{ej}{m}^m \max_{\xi\in [0,mh]} |f^{(m)}(\xi)|
\le 
h^m (2e)^m \max_{\xi\in [0,mh]}|f^{(m)}(\xi)|.
\]
Let $\wt f = \sum_{j=1}^m (-1)^{j+1} \binom m{j} \sP_{jh}f$. Note $f - \wt f = (-1)^m \De_{t,h}^m (\sP_tf)|_{t=0}$. % where the finite difference is with respect to $t$.
We then have (noting that $\sP_s$ and $\sL$ commute)
\begin{align*}
    \ve{f - \wt f}_{\ga'}^2 
    &\le 
    (2he)^{2m} \int_{\R^n} \max_{t\in [0,mh]} \pa{
    \dd{{}^m}{t^m} \sP_tf(x) }^2\,d\ga'(x)\\
    &\le(2he)^{2m} \int_{\R^n} 
    \pa{\dd{{}^m}{t^m} \sP_tf(x)|_{t=0} + \int_0^{mh} \ab{\dd{{}^{m+1}}{s^{m+1}} \sP_sf(x)}\,ds}^2 \,d\ga'(x) \\
    &\le 2(2he)^{2m}
    \pa{
    \int_{\R^n} 
    |\sL^{m} f(x)|^2\,d\ga'(x) + 
    mh \int_0^{mh} \int_{\R^n} 
    | \sP_s\sL^{m+1}f(x)|^2\,d\ga'(x)\,ds}.
    \numberthis \label{e:f-smooth}
\end{align*}
%\holden{Problem: get $\int ...\max_t ...\sL^k \sP_tf$, need $\int$ to be inside the $\max$.}
First we bound the first term and the second term when $s=0$. We use the Gaussian noise sensitivity assumption, the assumption on $\ga'$, and H\"older's inequality to obtain
\begin{align*}
    \int_{\R^n} |\sL^l f|^2\,d\ga' &\le \ve{\sL^l f}_{L^{(p+1)}(\ga)}^2 \ve{\dd{\ga'}{\ga}}_{L^{\fc{p+1}p}(\ga)}\\
    &\le r^2 \si^2 O\pf{rl^2}{\si^2}^{2l} \max\{r^2, l(1+p)\}^l e^{\fc{r^2}{2p}}.
    \numberthis\label{e:Lf-bd}
\end{align*}
To get a bound for the second term when $s>0$, 
we bound the following derivative, again using the Gaussian noise sensitivity assumption and H\"older's inequality:
\begin{align*}
    \ddd s \int_{\R^n} |\sP_s\sL^l f|^2\,d\ga'
% &\le 
% \ddd s\ba{
% \int_{\R^n} (\sL^l \sP_sf)^2\,d\ga' - 
% \int_{\R^n} (\sL^l \sP_sf)^2\,d\ga}\\
% &=
% \int_{\R^n} 2(\sL^l \sP_s f)(\sL^{l+1}\sP_sf)\,d\ga' - \int_{\R^n} 2(\sL^l \sP_s f)(\sL^{l+1}\sP_sf)\,d\ga\\
&\le\int_{\R^n} 2\an{ \sP_s \sL^lf,\sP_s \sL^{l+1} f}\dd{\ga'}{\ga}\,d\ga\\
% &\le 
% 2 \ve{\sL^l\sP_sf}_{L^{2(p+1)}(\ga)}
% \ve{\sL^{l+1}\sP_sf}_{L^{2(p+1)}(\ga)}
% \ve{\dd{\ga'}{\ga}}_{L^{\fc{p+1}p}(\ga)}\\
&\le
2 \ve{\sP_s\sL^lf}_{L^{2(p+1)}(\ga)}
\ve{\sP_s\sL^{l+1}f}_{L^{2(p+1)}(\ga)}
\ve{\dd{\ga'}{\ga}}_{L^{\fc{p+1}p}(\ga)}\\
&\le 2 \ve{\sL^lf}_{L^{2(p+1)}(\ga)}
\ve{\sL^{l+1}f}_{L^{2(p+1)}(\ga)}
\ve{\dd{\ga'}{\ga}}_{L^{\fc{p+1}p}(\ga)}
\numberthis \label{e:mono-P}
\\
% &\le 
% r^2 \sigma^2\pa{\frac{2e r l^2}{\sigma^{2}}}^{2 l} \max\bc{r^2, \ell p}^{2 l} e^{\fc{r^2}{2p}},\\
&\le r^2 \si^2 O\pf{rl^2}{\si^2}^{2l} \max\{r^2, l(1+p)\}^l e^{\fc{r^2}{2p}},
\numberthis \label{e:dPL-bd}
\end{align*}
where in \eqref{e:mono-P} we use 
the fact that for $q\ge 1$, $\ve{\sP_sg}_{L^q(\ga)}$ is monotonically decreasing in $s$.
% ,
% and 
% in the last step we use~\Cref{lem:csbound} and the assumption on $\ga'$. 
In both~\eqref{e:Lf-bd} and \eqref{e:dPL-bd},
%For $l=m+1$, 
we optimize the bound by taking $p=\fc{r^2} l$ 
to obtain
%(considering the two cases $\pf{R}{\si}^2\ge m$ and $\pf{R}{\si}^2<m$)
\begin{align*}
\int_{\R^n} |\sL^l f|^2\,d\ga'\
&\le  r^2 \si^2 O\pf{rl^2}{\si^2}^{2l} \max\{r^2,l\}^l, & l=m, \,m+1
\\
    \ddd s \int_{\R^n} |\sL^l \sP_sf|^2\,d\ga'
    &\le 
    %r^2 \sigma^2\pa{\frac{2e^2 r^3 l^2}{\sigma^{2}}}^{2l}
    %\color{red}
     r^2 \si^2 O\pf{rl^2}{\si^2}^{2l} \max\{r^2,l\}^l, & l=m+1.
\end{align*}
Integrating the last inequality twice %, using \Cref{lem:csbound} to bound $\int_{\R^n}|\sL^l f(x)|^2\,d\ga'(x)$ for $l=m,m+1$, 
and substituting into \eqref{e:f-smooth} gives
% \begin{align*}
%     \ve{f-\wt f}_{\ga'}^2 
%     &\lesssim 
%     h r^2 \sigma^2 O\pa{\frac{h r^3 m^2}{\sigma^2}}^{2m}
%     \label{e:f-wtf}\numberthis
% \end{align*}
\begin{align*}
    \ve{f-\wt f}_{\ga'}^2 
    &\lesssim r^2\si^2 O(h)^{2m} \pa{
    %\pa{\pf{mr}{\si}^2 \pa{1+\fc{m^{1/2}}r} }^{2m}
    %+ 
    (1+(mh)^3) \pf{rm^2}{\si^2}^{2m} \max\{r^2,m\}^m
    }\\
    &\lesssim r^2\si^2 O\pa{\fc{hrm^2\max\{r,\sqrt m\}}{\si^2}}^{2m} 
    \label{e:f-wtf}\numberthis
\end{align*}
when $mh=O(1)$. 
%. This is $O(R^2\ep^2)$ if 
 We choose $m\sim \ln \prc{\ep}$ and %$h \le c\pa{m\cdot \fc R\si \cdot \max\bc{m^2, \pf{R}{\si}^4}}^{-1}$ 
%$h \le c\pa{\frac{r^3 m^2}{\sigma^2}}^{-1}$ 
$h\le \fc{c\si^2}{rm^2 \max\{r,\sqrt m\}}$
for an appropriate constant $c$ to get
\begin{align*}
    \ve{f-\wt f}_{\ga'}^2 
    \lesssim \ep^2 r^2 \sigma^2.
\end{align*}

Write $f=\sum_{\mk\in \N_0^d} a_\mk h_\mk$ where $a_\mk\in \R^d$, and let $p_l = \sum_{|\mk|=l} a_\mk h_\mk$, so that $f=\sumz l{\iy} p_l$. 
Then
%\[
\begin{align}
\wt f = \sum_{\mk\in \N_0^n} b_\mk h_\mk \quad
\text{ for } \quad 
b_\mk = \sumo jm (-1)^{j+1} \binom mj e^{-jh|\mk|/\sigma^2} a_\mk
=  [1-(1-e^{-h|\mk|/\sigma^2})^m] a_{\mk}.\label{eq:akbk}
\end{align}
%\]
Let $q_l = \sum_{|\mk|=l} b_\mk h_\mk$. 
Note
\begin{align*}
    \ve{q_l}_\ga^2 = \sum_{|\mk|=l} \vn{b_\mk}^2 = 
    \sum_{|\mk|=l} [1-(1-e^{-h|\mk|/\sigma^2})^m]^2 \vn{a_\mk}^2
    \le m^2 e^{-2hl/\sigma^2} \ve{p_l}_{\ga}^2. 
\end{align*}
We approximate $\wt f$ with $g=\sum_{|\mk|<L} b_\mk h_\mk$ where $L\ge Ch^{-1}\ln \prc\ep$ for an appropriate constant $C$. 
We have by Cauchy-Schwarz and~\Cref{l:com-poly} that
\begin{align*}
    \ve{\wt f - g}_{\ga'}^2 
    %- \ve{\wt f - g}_{\ga}^2
    &\le 
    \int_{\R^n} \ab{\sum_{l\ge L} q_l}^2 
     \,d\ga' \\
    &\le \sum_{l\ge L} e^{-hl}\cdot \sum_{l\ge L} e^{hl}\int_{\R^n}  |q_l|^2 \,d\ga'\\
    &\le 1\cdot \sum_{l\ge L} e^{hl}\ve{q_l}_\ga^2 e^{2\sqrt l r} \\
    &\le \sum_{l\ge L} e^{hl} m^2 e^{-2hl/\sigma^2}
    %\ve{p_l}_\ga^2
    \ve{p_l}_\ga^2 e^{2\sqrt l r}\\ %/\si
    &\le \max_{l\ge L} m^2e^{2\sqrt l r - hl/\sigma^2} \ve{f}_\ga^2\le \ep^2 \ve{f}_\ga^2
    %/\si
    \numberthis \label{e:change-ga}
\end{align*}
when we take 
$L\ge C\fc{\sigma^4 r^2}{h^2} \asymp r^4\ln\prc\ep^4 \max\bc{r^2, \ln \prc\ep}$
%\ln\prc{\ep}^{4} R^{8}$ 
for an appropriate constant $C$. 
Note that the Gaussian-noise-sensitive property also implies $\ve{f}_\ga^2\le r^2\sigma^2$.
Plugging into~\eqref{e:smooth-triangle} the inequalities~\eqref{e:f-wtf} and~\eqref{e:change-ga},
\begin{align*}
    \ve{f-g}_{\ga'}^2 \lesssim 
\ve{f-\wt f}_{\ga'}^2 + 
\ve{\wt f - g}_{\ga'}^2 
%+ 
%\pa{\ve{\wt f - g}_{\ga'}^2 
%    - \ve{\wt f - g}_{\ga}^2
%}
\lesssim r^2\sigma^2 \ep^2.
\end{align*}
Choosing constants appropriately then gives the desired bound. Finally note that Equation~\eqref{eq:akbk} implies $|a_\mk| \leq |b_\mk|$ for all $\mk \in \N_0^n$. Hence
\begin{align*}
    \ve{g}_{\gamma}^2 \leq \ve{\wt f}_{\gamma}^2 = \sum_{\mk \in \N_0^n} \vn{b_\mk}^2 \leq \sum_{\mk \in \N_0^n} \vn{a_\mk}^2 = \ve{f}_{\gamma}^2.
\end{align*}
\end{proof}

Combining Lemmas~\ref{lem:csbound} and~\ref{l:poly-approx-mix} implies a low degree approximation of order $\text{poly}\pa{\frac{R}{\sigma}, \ln\pa{\frac{1}{\epsilon}}}$ for the single cluster case.

\section{From one to multiple clusters}
\label{s:multiple}
To handle the general case of multiple clusters, we approximate the score function using a piecewise low-degree polynomial over a set of warm-starts. Before applying Lemma~\ref{l:poly-approx-mix}, we locally approximate the score within the Voronoi cell corresponding to each warm-start $\hat \mu$ by $f_{\textup{loc}}$. We can then approximate $f_{\textup{loc}}$ using a low-degree polynomial by controlling the terms $\ve{\sL^m f_{\textup{loc}}(\cdot +\hat \mu)}_{L^p(\ga_{\si^2})}$, where  $f_{\textup{loc}}$ is shifted by the warm start $\hat \mu$. We start with the definition of the warm starts.
%\khasha{find a good notation to dfine the shifted Guassian distribution for above}

\begin{df}\label{d:ws}
Let $\cal C\sub \R^n$. 
We say $\cal C$ is a \vocab{complete set of $R$-warm starts} for $Q_0$ if 
\[Q_0\pa{\bigcup_{\wh \mu\in \cal C}B_R(\wh \mu)}=1.\]
\end{df}
In other words, for all $\mu$ in the support of $Q_0$, there exists $\wh \mu\in \cal C$ such that $\ve{\wh\mu-\mu}\le R$. 

\begin{df}
Let $\cal C = \{\wh \mu_1, \ldots, \wh\mu_{k'}\}$. The \vocab{Voronoi partition} $V_1,\ldots, V_{k'}$ corresponding to $\cal C$ is defined by $V_j = \set{x\in \R^n}{\vn{x-\wh \mu_j} = \min_{1\le j'\le k'} \vn{x-\wh \mu_{j'}}}$. 
\end{df}
Note that up to the boundaries (which are a measure 0 set), this induces a partition of $\R^n$.

To tackle the multiple cluster setting, we first suppose that we have a complete set of $R$-warm starts $\cal C=\{\wh \mu_1,\ldots, \wh \mu_{k'}\}$,
and let $V_1,\ldots, V_{k'}$ be the corresponding Voronoi partition. 
For a probability measure $P$ and set $S$, define the unnormalized restriction by $P_S(A) = P(A\cap S)$ and the normalized restriction by $P|_S(A) = \fc{P(A\cap S)}{P(S)}$. Define $P^S = (Q_0)_S \convolve \cal N(0,\si^2 I_n)$---that is, we do the restriction before the convolution. 
Define 
\[
f_{S, \si^2}(y) = 
\E_{\trow{\mu \sim Q_0|_S}{Y=\mu + \si\xi,\,\xi\sim \cal N(0,I_n)}}[\mu \mid Y=y]
= y + \si^2 \nb \ln p^S(y).
\]
When $\si$ is understood, we omit it from the subscript.
%For simplicity, we will take $\si=1$ and omit it from the subscript.

We need to show that we still have a good polynomial approximation for $f$ under the measure $P|_{V_i} = (Q_0\convolve\cal N(0,\si^2I_n))|_{V_i}$. 
%The main complication is that in general, $Q_0$ is no longer exclusively supported on $B_R(\wh \mu_i)$. 
By the analysis for one cluster, 
we have good approximation of $f_{V_i,\si^2}$ 
under %$Q_0|_{V_i}*\cal N(0,\si^2I_n)$
$P^{V_i}$.
To obtain the result for multiple clusters, we need to show that this approximation is preserved even when we consider $f$ instead of $f_{V_i,\si^2}$ and $P|_{V_i}$ instead of $P^{V_i}$, i.e., deal with the leakage into $V_i$ from the other Voronoi cells after $Q_0$ is convolved with $\cal N(0,\si^2I)$, in both the score function and the measure which the norm is with respect to.
We do this in Sections~\ref{ss:nearby-score} and~\ref{ss:voronoi}, respectively.

\subsection{Estimation with the ``nearby" score}

\label{ss:nearby-score}

We will actually bound $\ve{f-f_{S_i}}_{L^2(P_{V_i})}$, for an expanded neighborhood $S_i$ of $\wh\mu_i$, allowing an extra ``buffer region" where mass is allowed to leak in.

\begin{lem}\label{l:nearby-score}
%[Without $M$]
    Suppose that $\cal C=\{\wh \mu_1,\ldots, \wh \mu_{k'}\}$ is a complete set of %$C_1\si \sqrt{\ln \pf{k}{\ep}}$
    $R$-warm starts for $Q_0$, 
    and let $V_1,\ldots, V_{k'}$ be the corresponding Voronoi partition. Suppose $Q_0$ is supported on $B_{M}(0)$.
    Given $R'>R$, let $S_i = B_{R'}(\wh \mu_i)$ and define
    \begin{align*}
%f_{\textup{loc},\si^2}(y) &=
f_{\textup{loc},\si^2}^{\cal C, R'}(y) 
&= 
f_{S_i,\si^2}(y) \text{ when }y\in V_i,\\
\text{where }f_{S_i,\si^2}(y) &= 
\E_{\trow{\mu \sim Q_0|_{S_i}}{Y=\mu + \si\xi,\,\xi\sim \cal N(0,I_n)}}[\mu \mid Y=y] = 
y + \si^2 \nb \ln p^{S_i}(y).
    \end{align*}
We write $f_{\textup{loc},\si^2}$ when $\cal C, R'$ are clear. 
    Then for $R'=3R + 2\sqrt 2 \si \sqrt{\ln \pf{k'}{\ep}}$, 
    \[
\ve{f_{\si^2}-f_{\textup{loc}, \si^2}}_{L^2(P)}^2 \le \pa{32\sigma^2 + 6{R'}^2}\ep
\lesssim \pa{R^2 + \si^2 \ln \pf{k'}{\ep}}\ep.
    \]
\end{lem}

Note that 
\[
\ve{f_{\si^2}-f_{\textup{loc}, \si^2}}_{L^2(P)}^2
= \sum_{i=1}^{k'} \ve{f_{\si^2}-f_{S_i,\si^2}}_{L^2(P_{V_i})}^2,
\]
so this gives a bound for the $L^2$ norm within each Voronoi cell.

To prove this, we 
first show 
that with high probability, $Y=\mu+\si \xi$ does not stray too far from the Voronoi cell of $\mu$.

\begin{lem}\label{l:leave-V}
Let $\cal C=\{\wh \mu_1,\ldots, \wh \mu_{k'}\}$
    and $V_1,\ldots, V_{k'}$ be the corresponding Voronoi partition.
    Suppose that $\vn{\mu - \wh \mu_i}\le R$.
    %C_1\si \sqrt{\ln \pf k\ep}$. 
    Define $Y=\mu + \si \xi$ where $\xi\sim \cal N(0,I_n)$, and let
    $i'$ be such that $Y\in V_{i'}$. Then with probability $\ge 1-\ep$, 
\[
\vn{\mu - \wh \mu_{i'}} \le 
%C_2 \si \sqrt{\ln \pf{\max\{k,k'\}}{\ep}}
3R + 2\sqrt 2 \si \sqrt{\ln \pf{k'}{\ep}}.
\]
%where $C_2 = 3C_1+2\sqrt 2$.
\end{lem}
In other words, with high probability, even if $Y=\mu + \si \xi \in V_{i'}$ for $i'\ne i$ (i.e., adding a Gaussian brings the point to a different Voronoi cell), $\mu$ will not be too far from the center of the new cell $\wh \mu_{i'}$. 
\begin{proof}
    Let $v_{ii'} =\fc{\wh \mu_{i'}-\wh \mu_i}{\vn{\wh \mu_{i'}-\wh \mu_i}}$ be the unit vector pointing from $\wh \mu_i$ toward $\wh\mu_{i'}$. Since $\mu\in V_i$, we have
    \[
\an{\mu - \wh \mu_i , v_{ii'}} \le \fc{\vn{\wh \mu_{i'} - \wh \mu_i}}{2}.
    \]
    By Gaussian tail bounds and a union bound, with probability $\ge 1-\ep$, we have that for all $i'$,
    \[
\an{Y - \wh \mu_i , v_{ii'}}
= 
\an{\mu + \si \xi - \wh \mu_i , v_{ii'}}
\le \an{\mu - \wh \mu_i , v_{ii'}}  + \sqrt2 \si \sqrt{\ln \pf{k'}{\ep}}
\le 
R + \sqrt2 \si \sqrt{\ln \pf{k'}{\ep}}
%(C_1+\sqrt 2)\si \sqrt{\ln \pf{\max\{k,k'\}}{\ep}}
    \]
    Therefore, if $y\in V_{i'}$, then we have 
    \begin{align*}
    \fc{\vn{\wh \mu_{i'} - \wh \mu_i}}{2} &\le 
\an{Y - \wh \mu_i , v_{ii'}} \le %(C_1+\sqrt 2)\si \sqrt{\ln \pf{\max\{k,k'\}}{\ep}}
R + \sqrt2 \si \sqrt{\ln \pf{k'}{\ep}}
    \end{align*}
    and
    \[
    \vn{\mu - \wh \mu_{i'}}
    \le 
    \vn{\mu -\wh \mu_i} + 
\vn{\wh \mu_{i} - \wh \mu_{i'}} 
% \le C_2 \si \sqrt{\ln \pf{\max\{k,k'\}}{\ep}}.
\le 3R + 2\sqrt2 \si \sqrt{\ln \pf{k'}{\ep}}.
    \]
    %for $C_2 = 3C_1+2\sqrt 2$.
\end{proof}

\begin{proof}[Proof of \Cref{l:nearby-score}]
Given $\mu$, let $i$ be such that $y\in V_i$. 
We note 
\begin{align*}
    f_{\si^2}(y) &= \E[ \one_{S_i}(\mu) \mu \mid Y=y ] +\E[  \one_{S_i^c}(\mu)\mu  \mid Y=y]\\
    f_{\textup{loc}, \si^2}(y) &= 
    \Pj(\mu\in S_i \mid Y=y) f_{\textup{loc}, \si^2}(y)
    + \Pj(\mu\notin S_i \mid Y=y) f_{\textup{loc}, \si^2}(y)\\
    &= \Pj(\mu\in S_i \mid Y=y) \fc{\E\ba{\one_{S_i}(\mu) \mu\mid Y=y}}{\E\ba{\one_{S_i}(\mu) \mid Y=y}}
    + \Pj(\mu\notin S_i \mid Y=y) f_{\textup{loc}, \si^2}(y)\\
    &= \E[\one_{S_i}(\mu) \mu  \mid Y=y] + \E[ \one_{S_i^c}(\mu) \E_{{\mu' \sim Q_0|_{S_i}}}[\mu'\mid Y=y]\mid Y=y]
\end{align*}
    Hence, using the fact that $\mathbb P_{(\mu, y)}\pa{\mu \notin S_i} \leq \ep$, we have
    %\holden{How do you get the first equality? $i$ needs to be defined based on $\mu$.}
    \begin{align*}
        \ve{f_{\si^2}-f_{\textup{loc}, \si^2}}_{L^2(P)}^2 &= 
        \E_y \vn{\E_{{\mu \sim Q_0}}\Big[\one_{S_i^c}(\mu)\pa{\mu - \E_{{\mu' \sim Q_0|_{S_i}}}[\mu' \mid Y=y]}\Big| Y=y\Big]}^2\\
        &\leq \E_{{(\mu,y)}}\Big[\one_{S_i^c}(\mu)\vn{\mu - \E_{{\mu' \sim Q_0|_{S_i}}}[\mu' \mid Y=y]}^2\Big]\\
        &\leq 2\E_{{(\mu,y)}}\Big[\one_{S_i^c}(\mu)\vn{\mu - \wh \mu_i}^2 + \one_{S_i^c}(\mu)\vn{\wh \mu_i - \E_{{\mu' \sim Q_0|_{S_i}}}[\mu' \mid Y=y]}^2\Big]\\
        &\leq 2\E_{{(\mu,y)}}\Big[\one_{S_i^c}(\mu)\vn{\mu - \wh \mu_i}^2\Big] + 2\ep R'^2\\
        &\leq 2\int_{0}^\infty\mathbb P_{(\mu, y)}\pa{\vn{\mu - \wh \mu_i}^2 \geq 2R'^2 + a}da + 6\ep{R'}^2\\
         &\leq 2\int_{0}^\infty\mathbb P_{(\mu, y)}\pa{\vn{\mu - \wh \mu_i} \geq R' + \sqrt{\frac{a}{2}}}da + 6\ep{R'}^2.\numberthis\label{eq:mainder}
    \end{align*}
    But using Lemma~\ref{l:leave-V}, we have
    \begin{align*}
        \mathbb P\pa{\vn{\mu - \wh \mu_i} \geq 3R + r} \leq k'e^{-\frac{r^2}{8\sigma^2}},
    \end{align*}
    which implies
    \begin{align}
        \Pj_{(\mu, y)}\pa{\vn{\mu - \wh \mu_i} \geq R' + \sqrt{\frac{a}{2}}} \leq k'e^{-\frac{\pa{2\sqrt 2 \si \sqrt{\ln \pf{k'}{\ep}} + \sqrt{\frac{a}{2}}}^2}{8\sigma^2}}
        \leq k'e^{-\frac{8\si^2 \ln \pf{k'}{\ep} + a/2}{8\sigma^2}} = \ep e^{-\frac{a}{16\sigma^2}}.\label{eq:keyint}
    \end{align}
    Plugging Equation~\eqref{eq:keyint} into Equation~\eqref{eq:mainder},
    \begin{align*}
        \ve{f_{\si^2}-f_{\textup{loc}, \si^2}}_{L^2(P)}^2
        \leq 2\int_{0}^\infty\ep e^{-\frac{a}{16\sigma^2}} \,da + 6\ep{R'}^2 = 32\ep\sigma^2 + 6\ep{R'}^2.
    \end{align*}
\end{proof}

\subsection{Approximation within a Voronoi cell}

\label{ss:voronoi}

In the multiple cluster case, in order to apply~\Cref{lem:csbound} for approximating the score function in the Voronoi cell of a warm-start $\wh \mu_i$ by a low-degree polynomial, we also need to bound the R\'enyi divergence of the mixture measure $P$ restricted to that Voronoi cell of $\wh \mu_i$, with respect to the Gaussian measure around $\wh \mu_i$.

\begin{lem}\label{l:chi-p-V}
Let $\cal C=\{\wh \mu_1,\ldots, \wh \mu_{k'}\}$
and $V_1,\ldots, V_{k'}$ be the corresponding Voronoi partition. Suppose that $\cal C$ is a complete set of $R$-warm starts.
Then
\begin{align}
\label{e:chi-a-V}
\int_{V_i}
\pa{\dd{P_{V_i}}{\ga_{\wh \mu_i, 1}}}^{p+1} \,d\ga_{\wh \mu_i, 1} \le 
\exp\pf{pR^2}2.
%O\pa{\max\{\sqrt p, R\}}^p e^{-\fc{R^2}2}.
\end{align}
\end{lem}
\begin{proof}
We can write 
\[
P_{V_i} = Q_0(V_i) \fc{P^{V_i}}{Q_0(V_i)} + Q_0(V_i^c) \fc{P^{V_i^c}}{Q_0(V_i^c)};
\]
note that $Q_0(V_i)$ and $Q_0(V_i^c)$ are the normalizing constants for the respective probability measures. By convexity, it suffices to show that~\eqref{e:chi-a-V} holds for $P^{V_i}$ and $P^{V_i^c}$ in place of $P_{V_i}$. 

For $P^{V_i}$,~\eqref{e:chi-a-V} follows directly from~\Cref{l:chi-a}.
% \[
% \int_{V_i}
% \pa{\dd{P^{V_i}}{\ga_{\wh \mu_i, 1}}}^{p+1} \,d\ga_{\wh \mu_i, 1} \le 
% \]
For $P^{V_i^c}$, noting that it is a convex mixture of $\cal N(\mu_j, I_n)$ for $\mu_j\nin V_i$, 
again by convexity it suffices to show that~\eqref{e:chi-a-V} holds for one such $\cal N(\mu_j, I_n)$. Suppose that $\mu_j\in V_j$, $j\ne i$. 
Note that by definition of the Voronoi cell, 
$\ab{\mu_j - \wh \mu_i} \ge \ab{\mu_j - \wh\mu_j}$ so
$\dd{\ga_{\wh \mu_j, 1}}{\ga_{\wh \mu_i,1}}\le 1$ on $V_i$. Then 
\begin{align*}
    \int_{V_i}
\pa{\dd{\ga_{\mu_j, 1}}{\ga_{\wh \mu_i, 1}}}^{p+1}
d\ga_{\wh \mu_i, 1}
&=
\int_{V_i}
\pa{\dd{\ga_{\mu_j, 1}}{\ga_{\wh \mu_i, 1}}}^{p} 
d\ga_{\mu_j, 1}\\
&\le 
\int_{V_i}
\pa{\dd{\ga_{\mu_j, 1}}{\ga_{\wh \mu_j, 1}}}^{p} 
d\ga_{\mu_j, 1}\\
&= \int_{V_i}
\pa{\dd{\ga_{\mu_j, 1}}{\ga_{\wh \mu_j, 1}}}^{p+1} 
d\ga_{\wh \mu_j, 1}
\le e^{\fc{aR^2}2}\numberthis\label{eq:keyeq}
\end{align*}
where the last step follows from \Cref{l:chi-a}, as $\ve{\mu_j-\wh\mu_j}\le R$.
\end{proof}

\begin{lem}[Existence of low-degree polynomial]\label{lem:lowdegree}
Suppose that $f=f_{\si^2}$ as in \eqref{e:score-Q2}, $Q_0$ is supported on $B_\Di$, 
$\cal C=\{\wh \mu_1,\ldots, \wh \mu_{k'}\}$ is a complete set of $R$-warm starts for $Q_0$, 
and let $V_1,\ldots, V_{k'}$ be the corresponding Voronoi partition. %Suppose that $Q_0$ is supported on $B_{L}(0)$. 
For each $i$, there is a polynomial $g_i$ of degree at most $O\pa{\pa{\frac{R}{\sigma} + \sqrt{\ln \pf{k'}{\ep}}}^6 \ln \prc{\ep}^4}$
such that 
\[
\ve{f-g_i}_{P_{V_i}}^2\lesssim \ep^2\pa{R^2+\sigma^2 \ln \pf {k'}{\ep}} \quad \text{and} \quad \ve{g_i}_{\gamma_{\wh \mu_i, \sigma^2}} \leq \Di.
%R'.
\]
%\holden{Work out~\Cref{l:poly-approx-mix} but with $P_{V_i}$ instead of $\ga'$, using \Cref{l:chi-p-V} to do the change of measure.  }
\end{lem}
\begin{proof}
    \iffalse
    First, we use the following Holder inequality to change measure from $P_{V_i}$ to $\gamma_{\wh \mu_i, 1}$; using Lemma~\ref{l:chi-p-V}, for any $g$ we have
    \begin{align*}
        \ve{f - g}_{P_{V_i}} \leq \ve{f-g}_{ L^p(\gamma_{\wh \mu_i, 1})}
    \end{align*}
    \fi
    Let $R'=3R + 2\sqrt 2 \si \sqrt{\ln \pf{k'}{\ep}}$.
    First, note that from Lemma~\ref{l:nearby-score},
    \begin{align} \label{e:ld1}
    \sum_{i=1}^{k'} \ve{f_{\si^2}-f_{S_i,\si^2}}_{L^2(P_{V_i})}^2=
\ve{f_{\si^2}-f_{\textup{loc}, \si^2}}_{L^2(P)}^2 \lesssim \pa{R^2 + \si^2\ln \pf{k'}{\ep}}\ep^2,
    \end{align}
    where 
    $S_i$ is the expanded neighborhood in \Cref{l:nearby-score}. By \Cref{l:poly-approx-mix} applied to $P_{V_i}$, $f_{\textup{loc}, \si^2}$ is $(O(R'/\sigma),\si)$-Gaussian-noise-sensitive. By \Cref{l:chi-p-V}, for each $p>0$, $\ve{\dd{P_{V_i}}{\ga_{\wh \mu_i, \si^2}}}_{L^{1+p}(\ga_{\wh \mu_i, \si^2)}} \le e^{\fc{a(R/\si)^2}2}$. Hence $P_{V_i}$ satisfies the conditions of \Cref{l:poly-approx-mix}, and 
    there exists a polynomial $g_i$ of degree at most 
    $$O\pa{\pa{{\frac{R'}{\sigma}}}^4 \ln \prc\ep^4\max\bc{\pa{\frac{R'}{\sigma}}^2, \ln \prc \ep}} = O\pa{\pa{\frac{R'}{\sigma}}^6 \ln \prc\ep^4}$$ 
    such that 
    \begin{align}\label{e:ld2}
    \ve{f_{\textup{loc}, \si^2} - g_i}_{P_{V_i}}^2 = 
\ve{f^{S_i}-g_i}_{P_{V_i}}^2 \le \ep^2{R'}^2.
    \end{align}
    Combining~\eqref{e:ld1} and~\eqref{e:ld2} gives the result.
    Finally, from~\Cref{l:poly-approx-mix},
    \begin{align*}
        \ve{g_i}_{\gamma_{\wh \mu_i, \sigma^2}} \leq \ve{f_{\textup{loc}, \si^2}}_{P_{V_i}} \leq \Di.
        %R'.
    \end{align*}
%Adding and adjusting $\ep$ by a constant as necessary and noting that $R_0 \leq L$ gives the result.
\end{proof}

\subsection{Piecewise polynomial regression}
\label{ss:ppoly}
Thus far, we have demonstrated the existence of low-degree polynomials in the Voronoi cells corresponding to the warm starts, that provide a good approximation of the score function with respect to the mixture measure in each cell. Now we investigate how many samples we need, so that applying polynomial regression separately in each Voronoi cell enables us to recover the coefficients of these low-degree polynomials accurately.
For $(x_i,y_i)$ generated according to~\eqref{e:gen}, by~\eqref{e:score-P2} we wish to find the least-squares solution to 
\[
y_i + t s_t(y_i) \approx x_i.
\]
where by~\eqref{e:score-Q2} we parameterize
\[
f_{\si^2}(y) = y + \si^2 s_t(y) = \sum_{|\mk|\le d} b_\mk h_\mk(y - \wh \mu),
\]
where $\wh \mu$ is a warm start. 
Equivalently, we wish to find the least-squares solution to
\[
\sum_{|\mk|\le d}b_\mk h_\mk(y_i - \wh \mu) \approx y_i + \fc{\si^2}{t} (x_i-y_i) = 
\pa{1-\fc{\si^2}t}y_i + \fc{\si^2}t x_i.
\]
We can restrict to $(b_{\mk})_{|\mk|\le d}\in B_R$ (in $\R^{n\binom{[n]}{\le d}}$), so we solve
\[
\amin_{(b_{\mk})_{|\mk|\le d}\in B_R}
\wh L((b_{\mk})_{|\mk|\le d})\quad \text{ where }\quad 
\wh L((b_{\mk})_{|\mk|\le d})= 
\rc{\samp} \sumo i{\samp} \ab{\sum_{|\mk|\le d}b_\mk h_\mk(y_i-\wh \mu) - \pa{\pa{1-\fc{\si^2}t}y_i + \fc{\si^2}{t} x_i}}^2.
\]
Let $\ol x_i = \E[X|Y=y_i] = y_i + t\nb V_t(y_i)$ and $x_i= \ol x_i+\ze_i$, so that conditioned on $y_i$, $\ze_i$ is mean-zero noise. Let $\eta_i = \fc{\si^2}t\ze_i$. 
Let $z_i = \E[\mu|Y=y_i] = \pa{1-\fc{\si^2}t} y_i + \fc{\si^2}t \ol x_i$. Then, for each Voronoi cell $V_j$, we calculate the empirical loss for samples where $y_i$ falls into $V_j$,
\[
\wh L^{(j)}((b_{\mk})_{|\mk|\le d})= 
\rc {\samp} \sum_{i: y_i \in V_j} 
\ab{\sum_{|\mk|\le d} b_\mk h_\mk(y_i-\wh \mu_j) - (z_i + \eta_i)}^2,
\]
and the empirical risk minimizer (ERM) for the $j$th Voronoi cell as
\begin{align}
    (\wh b^{(j)}_{\mk})_{|\mk|\le d} \defeq \amin_{(b_{\mk})_{|\mk|\le d}\in B_\Rn}
\wh L^{(j)}((b_{\mk})_{|\mk|\le d})\quad.\label{eq:ermdef}
\end{align}
We then combine these ERM solutions for all Voronoi cells and define the following piecewise polynomial function to approximate the score:
\begin{align}
    \wh g(y) =\sumo j{k'} \one_{V_j}(y) \sum_{|\mk|\le d} \wh b_\mk^{(j)} h_\mk(y-\wh \mu_j),\label{eq:piecewisedef}
\end{align}
where $k'$ is the total number of Voronoi cells. To show that $\wh g$ is a decent approximation for the score function $f_{\sigma^2}$, we need to bound the generalization error of these ERM solutions~\eqref{eq:ermdef}. To do this, we derive moment bounds for the features and the noise. 
Let 
\[
L^{(j)}((b_{\mk})_{|\mk|\le d})= 
\E 
\ab{\sum_{|\mk|\le d} \one_{V_j}(Y)\pa{b_\mk h_\mk(Y-\wh \mu_j) - (Z + \eta)}}^2
\]
denote the population risk in $V_j$. We further define the shifted error of the empirical loss with respect to the population loss as
\begin{align*}
\err^{(j)}((b_{\mk})_{|\mk|\le d}) &\defeq
    (\wh L^{(j)} - L^{(j)}) ((b_{\mk})_{|\mk|\le d})
    + \rc {\samp}
    \sum_{i: y_i \in V_j}\an{\sum_{|\mk|\le d} z_i, \eta_i}
    - \rc {\samp} \sum_{i: y_i \in V_j} \pa{|\eta_i|^2 - \E |\eta_i|^2}\\
    & = 
    \rc {\samp}
    \sum_{i: y_i \in V_j}\an{\sum_{|\mk|\le d} b_\mk h_\mk(y_i-\wh \mu_j), \eta_i}\\
    &\quad + \rc {\samp}
    \sum_{i: y_i \in V_j}\pa{\ab{\sum_{|\mk|\le d} b_\mk h_\mk(y_i-\wh \mu_j) - z_i}^2 - \E \ab{\sum_{|\mk|\le d} b_\mk h_\mk(y_i-\wh \mu_j) - z_i}^2}.\numberthis\label{eq:ggap}
\end{align*}
The following lemma is standard and relates the generalization gap of ERM to $\err^{(j)}((b_{\mk})_{|\mk|\le d})$. As we see shortly, the terms in the definition~\eqref{eq:ggap} that do not depend on $(b_{\mk})_{|\mk|\le d}$ cancel out in the calculation of the generalization gap. 

\begin{lemma}[Uniform convergence $\Rightarrow$ Generalization gap]\label{lem:classicunif}
    The generalization gap for the ERM solution can be bounded as
    \begin{align*}
        L^{(j)}\pa{(\wh b^{(j)}_{\mk})_{|\mk|\le d}} - \min_{(b_{\mk})_{|\mk|\le d} \in B_\Rn} L^{(j)}\pa{(b_{\mk})_{|\mk|\le d}} \leq 2\max_{(b_{\mk})_{|\mk|\le d}\in B_\Rn}\ab{\err((b_{\mk})_{|\mk|\le d})}.
    \end{align*}
\end{lemma}
\begin{proof}
    To bound the generalization gap of the ERM solution $(\wh b_\mk^{(j)})_{|\mk|\le d_\ell, 1 \leq j \leq n}$ output by Algorithm~\ref{alg:mainalgo}, we compare its loss with the loss of an arbitrary point $(\wt b_\mk)_{|\mk|\le d}$ in $B_L(0)$: %plus the generalization gap:
    \begin{align*}
    L^{(j)}\pa{ (\wh b^{(j)}_{\mk})_{|\mk|\le d}} - 
    L^{(j)}\pa{ (\wt b_{\mk})_{|\mk|\le d}}
&=
L^{(j)}\pa{ (\wh b^{(j)}_{\mk})_{|\mk|\le d}} - 
    \wh L^{(j)}\pa{ (\wh b^{(j)}_{\mk})_{|\mk|\le d}}
    \\
    &\quad + 
    \wh L^{(j)}\pa{ (\wh b^{(j)}_{\mk})_{|\mk|\le d}} - 
    \wh L^{(j)}\pa{ (\wt b_{\mk})_{|\mk|\le d}}
    \\
    &\quad +
    \wh L^{(j)}\pa{ (\wt b_{\mk})_{|\mk|\le d}} - 
     L^{(j)}\pa{ (\wt b_{\mk})_{|\mk|\le d}}.\numberthis\label{eq:expansionzero}
    \end{align*}
    Note that $\wh L^{(j)}\pa{ (\wh b^{(j)}_{\mk})_{|\mk|\le d}} - 
    \wh L^{(j)}\pa{ (\wt b_{\mk})_{|\mk|\le d}} \leq 0$ by the definition of $(\wh b^{(j)}_{\mk})_{|\mk|\le d}$ as a minimizer of $\wh L^{(j)}$.
    Furthermore, by~\eqref{eq:ggap},
    \begin{align*}
    &L^{(j)}\pa{ (\wh b^{(j)}_{\mk})_{|\mk|\le d}} - 
    \wh L^{(j)}\pa{ (\wh b^{(j)}_{\mk})_{|\mk|\le d}}
    +
    \wh L^{(j)}\pa{ (\wt b_{\mk})_{|\mk|\le d}} - 
     L\pa{ (\wt b_{\mk})_{|\mk|\le d}}\\
     &= 
     \err^{(j)}((\wt b_{\mk})_{|\mk|\le d}) - \err^{(j)}((\wh b^{(j)}_{\mk})_{|\mk|\le d})
     \leq 2\max_{(b_{\mk})_{|\mk|\le d} \in B_\Rn(0)}\ab{\err^{(j)}((b_{\mk})_{|\mk|\le d})}.\numberthis\label{eq:expansion}
    \end{align*}
\end{proof}

Based on Lemma~\ref{lem:classicunif} and Equation~\eqref{eq:ggap}, to bound the generalization gap it suffices to bound
\begin{multline}
    \max_{(b_{\mk})_{|\mk|\le d}\in B_\Rn}\ab{\err((b_{\mk})_{|\mk|\le d})}
    \leq 
    \max_{(b_{\mk})_{|\mk|\le d}\in B_\Rn} \rc {\samp}
    \sum_{i: y_i \in V_j}\an{\sum_{|\mk|\le d} b_\mk h_\mk(y_i-\wh \mu_j), \eta_i}\\
    + \max_{(b_{\mk})_{|\mk|\le d}\in B_\Rn} \rc {\samp}
    \sum_{i: y_i \in V_j}\pa{\ab{\sum_{|\mk|\le d} b_\mk h_\mk(y_i-\wh \mu_j) - z_i}^2 - \E \ab{\sum_{|\mk|\le d} b_\mk h_\mk(y_i-\wh \mu_j) - z_i}^2}.\label{eq:errexpansion}
\end{multline}

Let $h(y)\in \R^{\binom{[n]}{\le d}}$ be given by $h(y)_{\mk } = h_\mk (y)$. Let $B \in \R^{n\times \binom{[n]}{\le d}}$ with columns $b_\mk$. Then $\ab{(b_\mk)_{|\mk|\le d}}=\ve{B}_F$ so 
\begin{align}
\nonumber
\max_{(b_\mk)_{|\mk|\le d}\in B_\Rn} 
\sum_{i: y_i \in V_j} \sum_{|\mk|\le d} \an{b_\mk h_\mk(y_i-\wh \mu_j), \eta_i}
&= \max_{\ve{B}_F\le \Rn}
\Tr\pa{B\sum_{i: y_i \in V_j} h(y_i-\wh \mu_j)\eta_i^\top}\\
&= \Rn \ve{\sum_{i: y_i \in V_j} h(y_i-\wh \mu_j)\eta_i^\top}_F.\label{eq:firsttermbound}
\end{align}

We carefully bound the moments of each of the terms in Equation~\eqref{eq:errexpansion} in Lemmas~\ref{lem:momentboundone} and~\ref{lem:secondmomentbound}, 
leading to high-probability bounds given in Lemmas~\ref{lem:highprobfirstterm} and~\ref{lem:highprobsecondterm}, respectively. From this, we obtain the following high probability bound on the generalization gap. 
\begin{lem}[Generalization gap]\label{lem:gengap}
Suppose $Q_0$ is supported on $B_\Di(0)$ and $\cal C$ is a complete set of $R$-warm starts for $Q_0$ with $|\cal C'|=k'$ and $d\geq (R/\sigma)^2$. 
    Given $\samp$ pairs of samples $\pa{\mu_i, y_i}$ distributed as $y_i = \mu_i + \xi_i$, for $\mu_i\sim Q_0$ and $\xi_i \sim \cal N(0,\si^2I_n)$ with 
    \begin{align*}
        %m = \Omega\pa{\frac{\pa{k'n\pa{\Rn+1}\pa{\Rn+\Di}}^2 \pa{\sqrt{\ln(k'/\delta)}\pa{n/d + 1}}^{4(d+2)}}{\ep^4}},
        \samp = \Omega\pa{\frac{n\pa{k'\pa{1+\Rn+\Di}^2}^2 \pa{\ln(k/\delta)\sqrt{n/d + 1}}^{4(d+1)}}{\ep^4}},
    \end{align*}
    we have with probability at least $1-\delta$ that 
    \begin{align*}
        \max_{(b_\mk)_{|\mk|\le d}\in B_\Rn(0)}  \ab{\err^{(j)}((b_{\mk})_{|\mk|\le d})} \leq \frac{\ep^2}{k'}.
    \end{align*}
    Furthermore, the piecewise polynomial approximation of the score defined in Equation~\eqref{eq:piecewisedef} satisfies
    \begin{align*}
        \ve{\wh g - f}_{L^2(P)}^2 - 
    \ve{\wt g - f}_{L^2(P)}^2 \leq \ep^2.
    \end{align*}
    for any other piece-wise polynomial $\wt g$ on the Voronoi cells whose coefficients  satisfy $(\wt b_{\mk}^{(j)})_{|\mk|\le d}\in B_\Rn$ for each cell $j$.
    %where $P = Q_0\convolve \gamma_{0,\sigma_0^2}$.
\end{lem}
We defer the proof to \Cref{s:gengap}.

\subsubsection{Moment and high-probability bounds}

For the following lemmas, we assume the following: $Q_0$ is supported on $B_\Di(0)$, $\cal C$ is a complete set of $R$-warm starts for $Q_0$, and $V_1,\ldots, V_{k'}$ is the corresponding Voronoi partition. We are given iid samples $(\mu_i, y_i), i=1,\dots,\samp$, where $y_i = \mu_i + \xi_i$, $\mu_i\sim Q_0$, and $\xi_i \sim \cal N(0,\si^2I_n)$. Recall that $h_\mk$ are the $\si$-rescaled Hermite polynomials and $h(y) = (h_\mk(y))_{|\mk|\le d}$. We moreover fix the Voronoi cell $V_j$ corresponding to $\wh \mu_j$, and (to simplify notation) suppose that $\wh \mu_j=0$.

\begin{lem}[Moment bound for first term]\label{lem:momentboundone}
%     In the same setting as Lemma~\ref{lem:highprobsecondterm}, for arbitrary Voronoi cell $V_j$ 
We have the following uniform convergence bound for $p$ even: %\khasha{changed}
    \[
\ve{\ve{\sumo i{\samp} \one_{V_j}(y_i) h(y_i) \eta_i^\top}_F}_{2p} = 
%O\pa{\Di n^{1/4}\binom{n+d}{d}^{3/4} e^{\sqrt{d/p} (R/\sigma)} (p-1)^{d}N^{1/2} 3^{2d/p}}
%O\pa{\Di \pa{3e^2(p-1)\pa{\frac{n}{d} + 1}}^{d} \sqrt N}.
O\pa{\Di n^{1/2}\pa{\frac{n}{d} + 1}^{d/2} \sqrt p\pa{e(p-1)}^d\sqrt {\samp}}.
    \]
\end{lem}
We will use properties of sub-exponential random variables; see \Cref{sec:subexponential} for background.
\begin{proof}
    In the following, the expectations without index refer to the mixture measure. 
    Let $\mu_i^{(2)}$ be independent and identically distributed to $\mu_i | Y= y_i$. Without loss of generality suppose the samples in which $y_i$ fall into $V_j$ are exactly $y_1,\dots,y_{\samp_j}$. Using the symmetrization technique with Jensen's inequality:
    \begin{align}
\E \ve{\sumo i{\samp_j} h(y_i) \eta_i^\top}_F^{2p}
&= \E \ve{\sumo i{\samp_j} h(y_i) (\mu_i - \E[\mu_i| Y=y_i])^\top}_F^{2p}\\
&= \E \ve{\sumo i{\samp_j} h(y_i) (\mu^{(2)}_i - \E[\mu_i| Y=y_i])^\top}_F^{2p}\\
&= \E \ve{\sumo i{\samp_j} h(y_i) (\E_{\mu_i}[(\mu^{(2)}_i - \mu_i)| Y=y_i])^\top}_F^{2p}\\
&\leq \E \ve{\sumo i{\samp_j} h(y_i) (\mu^{(2)}_i - \mu_i)^\top}_F^{2p}.
    \end{align}
Note that conditioned on any fixed value of $y_i$, the variable  $\mu^{(2)}_i - \mu_i$ is the difference between two iid random variables, each of whose norm is bounded by $\Di$. Therefore, given a fixed $\{y_i\}_{i=1}^{\samp_j}$, the variable $\ve{\sumo i{\samp_j} h(y_i) (\mu^{(2)}_i - \mu_i)^\top}_F^2$ is a quadratic form of a sub-Gaussian vector. From the properties of sub-Gaussian variables, for every $1\leq k\leq n$ and $|\mk| \leq d$,
\begin{align*}
   \pa{\sumo i{\samp_j} h_{\mk}(y_i) (\mu^{(2)}_i - \mu_i)_k}^2
\end{align*}
is sub-exponential with parameter $\pa{O\pa{\pa{\sumo i{\samp_j} h_\mk(y_i)^2 \Di^2}^2}, O\pa{\sumo i{\samp_j} h_\mk(y_i)^2 \Di^2}}$. Therefore, noting we do not have independence when we enumerate the variables over $\mk$, by \Cref{p:sum-subexp}, $\sum_{\mk}(\sumo i{\samp_j} h_{\mk}(y_i) (\mu^{(2)}_i - \mu_i)_k)^2$ is sub-exponential with parameters %\holden{What does this mean?}
%\holden{Should also have factor on second term}
\begin{align*}
    \pa{\binom{n+d}{d}O\pa{\sum_{|\mk| \leq d}\pa{\sumo i{\samp_j} h_\mk(y_i)^2 \Di^2}^2}, \binom{n+d}{d}O\pa{\max_{|\mk| \leq d}\sumo i{\samp_j} h_\mk(y_i)^2 \Di^2}}.
\end{align*}
%The latter is because we do not have independence when we enumerate the variables on $\mk$. %To see the properties of sub-exponential random variables that we use, see Appendix~\ref{sec:subexponential}. 
Finally summing over $1\leq k\leq n$, from the independence of the coordinates of $(\mu^{(2)} - \mu)$, by \Cref{p:sum-subexp-ind}, the variable
\begin{align*}
    X \defeq \ve{\sumo i{\samp_j} h(y_i) (\mu^{(2)}_i - \mu_i)^\top}_F^2 
    %= \sumo i{\samp_j} \ve{h(y_i)}^2 
    = \sumo kn \sum_{|\mk| \leq d} \pa{\sumo i{N_j} h_{\mk}(y_i) ((\mu^{(2)}_i - \mu_i)^\top)_k}^2
\end{align*}
is sub-exponential with parameters $$(v^2,\alpha) \defeq \pa{n\binom{n+d}{d}O\pa{\sum_{|\mk| \leq d}\pa{\sumo i{\samp_j} h_\mk(y_i)^2 \Di^2}^2}, n\binom{n+d}{d}O\pa{\max_{|\mk| \leq d}\sumo i{\samp_j} h_\mk(y_i)^2 \Di^2}}$$ and
 has expectation
\begin{align*}
    \E \ve{\sumo i{\samp_j} h(y_i) (\mu^{(2)}_i - \mu_i)^\top}_F^2 = \sumo i{N_j} \vn{h(y_i)}^2 \E\vn{\mu_i - \mu^{(2)}_i}^2
    = %2n\Di^2
    4\Di^2 \sumo i{\samp_j}\vn{h(y_i)}^2.
\end{align*}
But this implies 
\begin{align*}
    \mathbb P\Big(X - \E[X]\geq t\Big) \leq e^{-t/\kappa}, 
\end{align*}
for parameter
\begin{align*}
    \kappa = O\pa{\max\pa{v,\alpha}} = O\pa{\Di^2 n\binom{n+d}{d}\sumo i{\samp_j} \vn{h(y_i)}^2}.
\end{align*}
Thus, using the second property for sub-exponential variables in~\Cref{fact:subexp}, for some constant $c_1$ we have
\begin{align*}
    \E \ba{\ve{\sumo i{\samp_j} h(y_i) (\mu^{(2)}_i - \mu_i)^\top}_F^{2p} 
    \ \Big| \{y_i\}_{i=1}^{\samp_j}}
    \leq  O(p\kappa + \E X)^p = \pa{c_1 p\Di^2 n\binom{n+d}{d}\sumo i{\samp_j} \vn{h(y_i)}^2}^p.\numberthis\label{eq:combined}
\end{align*}
Then by Lemma~\ref{l:com-poly},
\begin{align*}
    \E \ve{\sumo i{\samp_j} h(y_i) (\mu^{(2)}_i - \mu_i)^\top}_F^{2p}
    &= \E\ba{\E\ba{{\ve{\sumo i{\samp_j} h(y_i) (\mu^{(2)}_i - \mu_i)^\top}_F^{2p}} \Big| \{y_i\}_{i=1}^{\samp_j}}}\\
    &\leq \E \pa{ c_1 p\Di^2 n\binom{n+d}{d}\sumo i{\samp_j} \vn{h(y_i)}^2}^p
    %&\leq  \E_{\{y_i\}_{i=1}^N \sim {\gamma}^{\otimes n}} O\pa{e^{2\sqrt{pd} (R/\sigma)} \Di^2 \sqrt{n\binom{n+d}{d}}\sumo i{N_j} \ve{h(y_i)}^2}^p.
\end{align*}
Therefore, given $\gamma = \gamma_{\wh \mu_j, \sigma^2}$ is the Gaussian around the warm start point $\wh \mu_j$, and \Cref{l:com-poly} applied to $|h|^p$ (which is a polynomial as $p$ is even) and \Cref{l:hc-poly},
\begin{align*}
    \ve{\ve{\sumo i{\samp} \one_{V_j}(y_i) h(y_i) \eta_i^\top}_F}_{2p}
    &\leq \sqrt c_1 \sqrt p\Di \pa{n\binom{n+d}{d}}^{1/2}\sqrt{\ve{\sumo i{\samp_j} \vn{h(y_i)}^2}_p}\\
    &\leq \sqrt c_1 \sqrt p\Di \pa{n\binom{n+d}{d}}^{1/2}\sqrt{\sumo i{\samp_j} \ve{\vn{h(y_i)}^2}_p}\\
    &\leq \sqrt c_1 \sqrt p\Di \pa{n\binom{n+d}{d}}^{1/2}\sqrt{\sumo i{\samp_j} e^{2\sqrt{d/p} (R/\sigma)}\ve{\vn{h(y_i)}^2}_{L^p(\gamma)}}\\
    &\leq \sqrt c_1 \sqrt p\Di \pa{n\binom{n+d}{d}}^{1/2}\sqrt{\sumo i{\samp_j} e^{2\sqrt{d/p} (R/\sigma)}\pa{p-1}^{2d}\ve{\vn{h(y_i)}^2}_{L^2(\gamma)}}\\
    &\leq  c_2\sqrt p\Di n^{1/2}\pa{\frac{n}{d} + 1}^{d/2} \pa{e(p-1)}^d\sqrt {\samp},
\end{align*}
for some universal constant $c_2$, %where $\ve{\cdot }_{p,\gamma}$ refers to the $p$th moment under measure $\gamma$, 
and we used the upper bound $\binom{n}k \le \rc e\pf{en}{k}^k$. %\frac{1}{\sqrt{2\pi n}} \pa{\frac{n}{k} + 1}^k$.

\end{proof} 

\begin{lem}[High probability bound for first term]\label{lem:highprobfirstterm}
%    In the same setting as Lemma~\ref{lem:highprobsecondterm}, restricting $y_i$'s to a single Voronoi cell, 
We have
    \begin{align*}
        \mathbb P\pa{\ve{\sumo i\samp \one_{V_j}(y_i) h(y_i) \eta_i^\top}_F \geq t} \leq 
        \exp\pa{-\Om\pa{\pa{\frac{n}{d} + 1}^{-1/2}\pa{\frac{t}{\Di \sqrt{\samp}n^{1/2} }}^{1/(d+1)}}}.
        %\exp\pa{-\Om\pa{\pa{\frac{n}{d} + 1}^{-1}\pa{\frac{t}{\Di \sqrt \samp}}^{1/d}}}.
    \end{align*}
\end{lem}
\begin{proof}
    By Lemma~\ref{lem:momentboundone}
    and Markov's inequality,
    \begin{align*}
        \mathbb P\pa{\ve{\sumo i\samp \one_{V_j}(y_i)h(y_i) \eta_i^\top}_F \geq t}
        &= \mathbb P\pa{\ve{\sumo i\samp \one_{V_j}(y_i)h(y_i) \eta_i^\top}_F^{2p} \geq t^{2p}}\\
        &\leq \frac{\E\ve{\sumo i\samp \one_{V_j}(y_i) h(y_i) \eta_i^\top}_F^{2p}}{t^{2p}}\\
        &\leq \pa{\frac{\Di n^{1/2}\pa{\frac{n}{d} + 1}^{d/2} \pa{e(p-1)}^{d+1}\sqrt \samp}{t}}^{2p}.
    \end{align*}
    Picking 
    \begin{align*}
        p-1 = \Te\pa{\frac{\pa{t/(\Di\sqrt \samp n^{1/2})}^{1/(d+1)}}{\pa{\frac{n}{d} + 1}^{1/2}}}
    \end{align*}
    with appropriate constant,
    we have
    \begin{align*}
        \mathbb P\pa{\ve{\sumo i\samp \one_{V_j}(y_i)h(y_i) \eta_i^\top}_F \geq t} \leq \exp\pa{-\Om\pa{\pa{\frac{n}{d} + 1}^{-1/2}\pa{\frac{t}{\Di \sqrt{\samp}n^{1/2} }}^{1/(d+1)}}}.
    \end{align*}
\end{proof}

\begin{lem}[Moment bound for second term]\label{lem:secondmomentbound}
    %In the same setting as Lemma~\ref{lem:highprobsecondterm}, 
    Assume $d\ge (R/\si)^2$. 
    %For each Voronoi cell $V_j$ 
    We have the following uniform convergence bound:
    \begin{align*}
    &\ve{\max_{(b_\mk)_{|\mk|\le d}\in B_{\Rn}}\sumo i\samp \one_{V_j}(y_i)\pa{\ab{\sum_{|\mk|\le d} b_\mk h_\mk(y_i) - z_i}^2 - \E \ab{\sum_{|\mk|\le d} b_\mk h_\mk(y_i) - z_i}^2}}_{2p} \\
    &= 
    O\pa{\Rn(\Rn+\Di)n^{1/2}e^d\pa{\frac{n}{d} + 1}^{d} p \pa{4(4p-1)}^{2d} \sqrt \samp + \sqrt{p\samp}\Di^2
    }.
    \end{align*}
\end{lem}
\begin{proof}
    Without loss of generality, suppose the first $\samp_j$ samples $y_1,\dots, y_{\samp_j}$ are in $V_j$. Again using symmetrization for the $y_i$'s,
    \begin{align*}
        &\ve{\max_{(b_\mk)_{|\mk|\le d}\in B_{\Rn}}\sumo i{\samp_j}\pa{\ab{\sum_{|\mk|\le d} b_\mk h_\mk(y_i) - z_i}^2 - \E \ab{\sum_{|\mk|\le d} b_\mk h_\mk(y_i) - z_i}^2}}_{2p}.
        \\
        &\leq 
        \ve{\max_{(b_\mk)_{|\mk|\le d}\in B_{\Rn}}\sumo i{\samp_j}\pa{\ab{\sum_{|\mk|\le d} b_\mk h_\mk(y^{(2)}_i) - z^{(2)}_i}^2 - \E \ab{\sum_{|\mk|\le d} b_\mk h_\mk(y_i) - z_i}^2}}_{2p}
        \\
        &\leq 
        \ve{\max_{\ve{B}_F \leq \Rn}\tr\pa{B \pa{\sumo i{\samp_j} h(y_i) h(y_i)^\top - \sumo i{\samp_j} h(y^{(2)}_i) h(y^{(2)}_i)^\top} B^\top} }_{2p}\\
        &\quad+  \ve{\max_{\ve{B}_F \leq \Rn}\Big\{\sumo i{\samp_j} z_i^\top B h(y_i) - \sumo i{\samp_j} {z_i^{(2)}}^\top B h(y^{(2)}_i)\Big\} }_{2p}+\ve{\sumo i{\samp_j} \pa{\ab{z^{(2)}_i}^2 - \vn{z_i}^2}}_{2p}\\
        &\leq 
         \ve{\Rn^2 \ve{\sumo i{\samp_j} h(y_i) h(y_i)^\top - \sumo i{\samp_j} h(y^{(2)}_i) h(y^{(2)}_i)^\top}_{\textup{op}} }_{2p}\\
        &\quad+  \ve{\max_{\ve{B}_F \leq \Rn} \an{B,\sumo i{\samp_j} z_i h(y_i)^\top - \sumo i{\samp_j} z^{(2)}_i h(y^{(2)}_i)^\top}}_{2p}+ \ve{\sumo i{\samp_j} \pa{\ab{z^{(2)}_i}^2 - \vn{z_i}^2}}_{2p}\\
        &\leq 
        \Rn^2 \ve{ \ve{\sumo i{\samp_j} h(y_i) h(y_i)^\top - \sumo i{\samp_j} h(y^{(2)}_i) h(y^{(2)}_i)^\top}_{F} }_{2p}\\
        &\quad+ \Rn \ve{\ve{\sumo i{\samp_j} z_i h(y_i)^\top - \sumo i{\samp_j} z^{(2)}_i h(y^{(2)}_i)^\top}_F }_{2p} + \ve{\sumo i{\samp_j} \pa{\ab{z^{(2)}_i}^2 - \vn{z_i}^2}}_{2p}.
    \end{align*}

For the third term, note that for each $1 \leq i \leq \samp$:
\begin{align*}
    \Big|\ab{z^{(2)}_i}^2 - \vn{z_i}^2\Big| \leq 2\Di^2,
\end{align*} 
and this has mean zero. Hence the third term is sub-Gaussian with parameter $O(\Di^2\sqrt \samp)$. Therefore
\begin{align}
    \ve{\sumo i{\samp_j} \pa{\ab{z^{(2)}_i}^2 - \vn{z_i}^2} }_{2p} \le O\pa{\sqrt{p\samp} \Di^2}.
    \label{eq:firstsubterm}
\end{align}
For the second term, note that for all $1\leq i\leq {\samp_j}$, $1\leq k\leq n$, and $|\mk| \leq \binom{n}{d}$, given $\gamma = \gamma_{\wh \mu_j, \sigma^2}$ is the Gaussian around the warm start point $\wh \mu_j$: by Cauchy-Schwarz, \Cref{l:com-poly} applied to $h_\mk^p$, \Cref{l:hc-poly}, and the assumption that $d\ge (R/\si)^2$, 
\begin{align*}
    \ve{{z_i}_k h_{\mk}(y_i) - {z_i}_k^{(2)} h_{\mk}(y^{(2)}_i)}_p
    &\leq 2\ve{{z_i}_k h_{\mk}(y_i)}_p\\
    &\leq 2\ve{{z_i}_k}_{2p}\ve{h_{\mk}(y_i)}_{2p}\\
    &\leq 2\Di \cdot 2^{\sqrt{d/(2p)} (R/\sigma)}\ve{h_{\mk}(y_i)}_{L^{2p}(\gamma)}\\
    &\leq 2\Di \cdot 2^{\sqrt{d/(2p)} (R/\sigma)}(2p-1)^d \ve{h_{\mk}(y_i)}_{L^2(\gamma)} \\
    &\le 2\Di \cdot \pa{2(2p-1)}^d.
\end{align*}
%where $\ve{.}_{p,\gamma}$ refers to the $p$th moment under measure $\gamma$.
Therefore, from the Rosenthal inequality~\cite{ibragimov2001best} with constant given by Pinelis~\cite{nagaev1978some}, there is a universal constant $c$ such that
\begin{align*}
    &\E \ab{\sumo i{\samp_j} {z_i}_k h_{\mk}(y_i) - {z_i}_k^{(2)} h_{\mk}(y^{(2)}_i)}^p\\ 
    &\leq (cp)^p \max \bc{\sumo i{\samp_j} \E \ab{{z_i}_k h_{\mk}(y_i) - {z_i}_k^{(2)} h_{\mk}(y^{(2)}_i)}^p, \pa{\sumo i{\samp_j} \E \ab{{z_i}_k h_{\mk}(y_i) - {z_i}_k^{(2)} h_{\mk}(y^{(2)}_i)}^2}^{p/2}}\\
    &\leq (cp)^p \max \bc{{\samp_j} 2^p\Di^p(2(2p-1))^{pd}, \pa{4{\samp_j} \Di^2 6^{2d}}^{p/2}}\\
    &\leq\pa{2cp\pa{6(2p-1)}^d \sqrt \samp\Di}^p,
\end{align*}
which according to the norm property of $p$th norm $\ve{\cdot}_p$ implies
\begin{align*}
    \ve{\sum_{\ab{\mk} \leq d}\sum_{k=1}^n\pa{\sumo i{\samp_j} {z_i}_k h_{\mk}(y_i) - {z_i}_k^{(2)} h_{\mk}^{(2)}}^{2}}_p
    &\leq \sum_{\ab{\mk} \leq d}\sum_{k=1}^n
    \ve{\pa{\sumo i{\samp_j} {z_i}_k h_{\mk}(y_i) - {z_i}_k^{(2)} h_{\mk}^{(2)}}^2}_p
    \\
    &\leq n\binom{n + d}{d}\pa{4cp\pa{6(4p-1)}^d \sqrt \samp\Di}^{2}.
\end{align*}
Therefore, using $\binom{n}k \le \rc e\pf{en}{k}^k$,
%$\binom{n+k}{k} \leq \frac{1}{\sqrt{2\pi n}} \pa{\frac{n}{k} + 1}^k$,
\begin{align}
    \ve{\ve{\sumo i{\samp_j} z_i h(y_i)^\top - \sumo i{\samp_j} z^{(2)}_i h(y^{(2)}_i)^\top}_F }_{2p}\leq 
     n^{1/2} \pa{e\pa{\frac{n}{d} + 1}}^{d/2} 4cp \pa{6(4p-1)}^{d} \sqrt \samp \Di.\label{eq:secondsubterm}
\end{align}

\iffalse
Restricting to the $j$th voronoi region $V_j$, we use Lemma~\ref{} to change the measure of $y_i, y^{(2)}_i$'s to $\gamma = \gamma_{\wh \mu_{j}, 1}$ and further apply Gaussian hypercontractivity
\begin{align*}
    &\E {\ab{\sumo i\samp h(y_i) - \sumo i\samp  h(y^{(2)}_i)}}^{2p}\\ 
    &\leq 
     e^{2\sqrt{dp} (R/\sigma)} \E_\gamma\ab{\sumo i\samp h(y_i) - \sumo i\samp h(y^{(2)}_i)}^{2p}\\ 
    &\leq 
    e^{2\sqrt{dp} (R/\sigma)} \pa{\frac{p-1}{3}}^{2dp}  \pa{\E_{\gamma} \ab{\sumo i\samp (h(y_i) - h(y^{(2)}_i))}^4}^{p/2}\\
    &= 
    e^{2\sqrt{dp} (R/\sigma)} (\frac{p-1}{3})^{2dp}  \pa{\binom{n+d}{d}\E_{\gamma} \sum_{|\mk| \leq d}\pa{\sumo i\samp h_\mk(y_i) - h_\mk(y^{(2)}_i)}^4}^{p/2}\\
    &\leq 
    e^{2\sqrt{dp} (R/\sigma)} (p-1)^{2dp} \binom{n+d}{d}^{p/2} \pa{ \sum_{|\mk| \leq d}\pa{\E_{\gamma}\pa{\sumo i\samp h_\mk(y_i) - h_\mk(y^{(2)}_i)}^2}^2}^{p/2}\\
    &= 
    e^{2\sqrt{dp} (R/\sigma)} (p-1)^{2dp} \binom{n+d}{d}^{p/2} \pa{ \sum_{|\mk| \leq d}\pa{\sumo i\samp \E_{\gamma}\pa{ h_\mk(y_i) - h_\mk(y^{(2)}_i)}^2}^2}^{p/2}\\
    &= 
    e^{2\sqrt{dp} (R/\sigma)} (p-1)^{2dp} \binom{n+d}{d}^{p/2} \pa{ \binom{n+d}{d}\pa{2\samp}^2}^{p/2}\\
    &= 
    e^{2\sqrt{dp} (R/\sigma)} (p-1)^{2dp} \binom{n+d}{d}^{p} (2\samp)^p\\
    &\leq 
    \pa{e(p-1)}^{2dp} \binom{n+d}{d}^{p} (2\samp)^p\\
    &\leq \pa{e(p-1)}^{2dp} \pa{e\pa{\frac{n}{d}+1}}^{2pd} (2\samp)^p,\numberthis\label{eq:similardev}
\end{align*}
\fi

For the first term, for all $|\mk_1|, |\mk_2| \leq d$, again by Cauchy-Schwarz, \Cref{l:com-poly} applied to $h_\mk^p$, \Cref{l:hc-poly}, and the assumption that $d\ge (R/\si)^2$, 
\begin{align*}
    \ve{h_{\mk_1}(y_i) h_{\mk_2}(y_i) - h_{\mk_1}(y^{(2)}_i) h_{\mk_2}(y^{(2)}_i)}_p
    &\leq 2\ve{h_{\mk_1}(y_i) h_{\mk_2}(y_i)}_p\\
    &\leq 2\ve{h_{\mk_1}(y_i)}_{2p}\ve{h_{\mk_2}(y_i)}_{2p}\\
    &\leq  2^{2\sqrt{d/(2p)} (R/\sigma) + 1}\ve{h_{\mk_1}(y_i)}_{L^{2p}(\gamma)}\ve{h_{\mk_2}(y_i)}_{L^{2p}(\gamma)}\\
    &\leq 2^{2\sqrt{d/(2p)} (R/\sigma) + 1} (2p-1)^{2d} \ve{h_{\mk_1}(y_i)}_{L^2(\gamma)}\ve{h_{\mk_2}(y_i)}_{L^2(\gamma)} \\
    &\leq \pa{4(2p-1)}^{2d}.
\end{align*}
Therefore, again using Rosenthal's inequality,%\khasha{TODO: fix this}
\begin{align*}
    &\E \ab{\sumo i{\samp_j} \pa{h_{\mk_1}(y_i) h_{\mk_2}(y_i) - h_{\mk_1}(y^{(2)}_i) h_{\mk_2}(y^{(2)}_i)}}^p\\ 
    &\leq (cp)^p \max \bc{\sumo i{\samp_j} \E \ab{h_{\mk_1}(y_i) h_{\mk_2}(y_i) - h_{\mk_1}(y^{(2)}_i) h_{\mk_2}(y^{(2)}_i)}^p, \pa{\sumo i{\samp_j} \E \ab{h_{\mk_1}(y_i) h_{\mk_2}(y_i) - h_{\mk_1}(y^{(2)}_i) h_{\mk_2}(y^{(2)}_i)}^2}^{p/2}}\\
    &\leq \pa{cp\pa{4(2p-1)}^{2d} \sqrt \samp}^p,
\end{align*}
which implies
\begin{align*}
    &\ve{\sum_{\ab{\mk_1},\ab{\mk_2} \leq d}\pa{\sumo i{\samp_j} \pa{h_{\mk_1}(y_i) h_{\mk_2}(y_i) - h_{\mk_1}(y^{(2)}_i) h_{\mk_2}(y^{(2)}_i)}}^{2}}_p\\
    &\leq \sum_{\ab{\mk_1},\ab{\mk_2} \leq d}
    \ve{\pa{\sumo i{\samp_j} \pa{h_{\mk_1}(y_i) h_{\mk_2}(y_i) - h_{\mk_1}(y^{(2)}_i) h_{\mk_2}(y^{(2)}_i)}}^2}_p
    \\
    &\leq \binom{n+d}{d}^2\pa{2cp\pa{4(4p-1)}^{2d} \sqrt \samp}^2.
\end{align*}
Therefore, using $\binom{n}k \le \rc e\pf{en}{k}^k$,
%$\binom{n+k}{k} \leq \frac{1}{\sqrt{2\pi n}} \pa{\frac{n}{k} + 1}^k$,
\begin{align}
    \ve{\ve{\sumo i{\samp_j} \pa{h(y_i) h(y_i)^\top - h(y^{(2)}_i) h(y^{(2)}_i)^\top}}_F }_{2p}\leq 
     %n^{1/2}
     \pa{e\pa{\frac{n}{d} + 1}}^{d} 2cp \pa{4(4p-1)}^{2d} \sqrt \samp.\label{eq:thirdsubterm}
\end{align}
%using a similar derivation as in Equation~\eqref{eq:similardev} except for dimension $\binom{n+d}{d}^2$ and degree $2d$ (i.e. viewing matrices as a vector), 

Combining Equations~\eqref{eq:thirdsubterm},~\eqref{eq:secondsubterm}, and~\eqref{eq:firstsubterm} completes the proof.
\end{proof}

\begin{lem}[High-probability bound for second term]\label{lem:highprobsecondterm}
    %Given that $Q_0$ is supported on $B_\Di(0)$, consider the iid samples $(\mu_i, y_i), i=1,\dots,\samp$, each restricted to a Voronoi region $V_j$ for a set of $R$-warm starts. Then, 
    %The $\sigma$-rescaled Hermite polynomials $h_\mk$ with degree at most $|\mk| \leq d$ such that $d \geq (R/\sigma)^2$ satisfy the following high probability bound: 
    Assume $d \geq (R/\sigma)^2$. We have
    \begin{align*}
        &\mathbb P\pa{\max_{(b_\mk)_{|\mk|\le d}\in B_{\Rn}(0)}\sumo i\samp\pa{\ab{\sum_{|\mk|\le d} b_\mk h_\mk(y_i) - z_i}^2 - \E \ab{\sum_{|\mk|\le d} b_\mk h_\mk(y_i) - z_i}^2} \geq t} \\
        &\leq \exp\pa{-\Om\pa{\pa{\frac{n}{d} + 1}^{-1/2}\pa{\frac{t}{n^{1/2}\Rn(\Rn+\Di) \sqrt \samp}}^{1/(2(d+1))}
        \wedge \pa{\frac{t}{\Di^2\sqrt \samp}}^2
        }}.
    \end{align*}
\end{lem}
\begin{proof}
We have
    \begin{align*}
        &\mathbb P\pa{\max_{(b_\mk)_{|\mk|\le d}\in B_{\Rn}(0)}\sumo i\samp\pa{\ab{\sum_{|\mk|\le d} b_\mk h_\mk(y_i) - z_i}^2 - \E \ab{\sum_{|\mk|\le d} b_\mk h_\mk(y_i) - z_i}^2} \geq t}\\
        &= \frac{\E \ab{\max_{(b_\mk)_{|\mk|\le d}\in B_{\Rn}(0)}\sumo i\samp\pa{\ab{\sum_{|\mk|\le d} b_\mk h_\mk(y_i) - z_i}^2 - \E \ab{\sum_{|\mk|\le d} b_\mk h_\mk(y_i) - z_i}^2}}^{2p}}{t^{2p}}\\
        &\leq \pa{\frac{n^{1/2}\Rn(\Rn+\Di)e^d\pa{\frac{n}{d} + 1}^{d} p \pa{4(4p-1)}^{2d} \sqrt \samp + %\sqrt{p\samp}\Di
        O(\sqrt{p\samp}\Di^2)}{t}}^{2p}.
    \end{align*}
    Picking 
    \begin{align*}
        p = O\pa{\frac{\pa{c't/(n^{1/2}\Rn(\Rn+\Di)\sqrt \samp)}^{1/(2(d+1))}}{\pa{\frac{n}{d} + 1}^{1/2}} \wedge \pa{\frac{c't}{\Di^2\sqrt \samp}}^2
        },
    \end{align*}
    for universal constant $c'$ small enough, the proof is complete.
\end{proof}

\subsubsection{Generalization gap}
\label{s:gengap}

\begin{proof}[Proof of \Cref{lem:gengap}]
    Note that from Equation~\eqref{eq:ggap}, for every fixed Voronoi cell $V_j$:
    \begin{multline}
        \max_{(b_\mk)_{|\mk|\le d}\in B_\Rn} \ab{\err^{(j)}((b_{\mk})_{|\mk|\le d})}  
        \leq \frac{1}\samp\ve{\sum_{i: y_i \in V_j} h(y_i-\wh \mu_j)\eta_i^\top}_F \\
        + 
        \max_{(b_{\mk})_{|\mk|\le d} \in B_\Rn} \rc \samp
    \sum_{i: y_i \in V_j}\pa{\ab{\sum_{|\mk|\le d} b_\mk h_\mk(y_i-\wh \mu_j) - z_i}^2 - \E \ab{\sum_{|\mk|\le d} b_\mk h_\mk(y_i-\wh \mu_j) - z_i}^2}.\label{eq:divisionn}
    \end{multline}
    But from Lemma~\ref{lem:highprobfirstterm}, given 
    %\begin{align*}
    %    m = \Omega\pa{\frac{k'^2 \ln^{2(d+2)}(1/\delta) \Di^2 \pa{n/d + 1}^{4(d+2)}}{\ep^4}},
    %\end{align*}
    \begin{align*}
        \samp = \Omega\pa{\frac{k'^2 \ln^{2(d+1)}(k'/\delta) \Di^2 n\pa{n/d + 1}^{(d+1)}}{\ep^4}},
        %2(d+1) -> (d+1)/2
    \end{align*}
    samples, with probability at least $1-\delta/(2k')$ the absolute value of the first term in Equation~\eqref{eq:divisionn} is at most $\ep^2/(2k')$. 
    %From Lemma~\ref{lem:highprobfirstterm},
    %\begin{align*}
    %    N = \Omega\pa{\frac{M^2\ln^{1/2}(1/\delta)}{\ep^4}}.
    %\end{align*}
    Furthermore, based on Lemma~\ref{lem:highprobsecondterm}, given
    \begin{align*}
        \samp = \Omega\pa{\frac{k'^2\ln^{4(d+1)}(k'/\delta) n(\Rn+\Di)^4 (n/d + 1)^{2(d+1)}}{\ep^4}}
        %4(d+1) -> 2(d+1)
        %(M+1)^2(M+D)^2 -> (M+D)^4
    \end{align*}
    with probability at least $1-\delta/(2k')$ the second term is bounded by $\ep^2/(2k')$. Applying a union bound, the sum of the first and second terms is bounded by $\ep^2$ with probability at least $1-\delta/k'$. This shows the first claim. To show the second claim, note that
    \begin{align*}
        L^{(j)}((\wh b_{\mk})_{|\mk|\le d})&= 
    \E_{(\mu, Y)} 
\ab{\one_{V_j}(Y)\pa{\wh g(Y) - (Z + \eta)}}^2\\
&= \E_{(\mu, Y)} 
\ab{ \one_{V_j}(Y)\pa{\wh g(Y) - Z}}^2 + \E[\one_{V_j}(Y) \ab{\eta}^2],
    \end{align*}
where $\wh g$ is defined as in \eqref{eq:piecewisedef}
and similarly
\begin{align*}
        L^{(j)}((\wt b_{\mk})_{|\mk|\le d})= \E_{(\mu, Y)} 
\ab{ \one_{V_j}(Y)\pa{\wt g(Y) - Z}}^2 + \E[\one_{V_j}(Y) \ab{\eta}^2],
    \end{align*}
    where $\wt g$ is defined analogously.
    Therefore
    \begin{align}
        L^{(j)}((\wh b_{\mk})_{|\mk|\le d}) - L^{(j)}((\wt b_{\mk})_{|\mk|\le d}) = 
        \E_{(\mu, Y)} 
\ab{ \one_{V_j}(Y)\pa{\wh g(Y) - Z}}^2 - \E_{(\mu, Y)} 
\ab{ \one_{V_j}(Y)\pa{\wt g(Y) - Z}}^2.\label{eq:middlepoint}
    \end{align}
    Summing Equation~\eqref{eq:middlepoint} for $1\leq j\leq k'$ implies
    \begin{align*}
    \sumo j{k'} L^{(j)}((\wh b_{\mk})_{|\mk|\le d}) - L^{(j)}((\wt b_{\mk})_{|\mk|\le d}) = 
        \ve{\wh g - f}_{L^2(P)}^2 - 
    \ve{\wt g - f}_{L^2(P)}^2.
    \end{align*}
    But from the previous claim and union bound, we know with probability at least $1-\delta$ each $L^{(j)}((\wh b_{\mk})_{|\mk|\le d}) - L^{(j)}((\wt b_{\mk})_{|\mk|\le d})$ is bounded by $\ep^2/k'$. This completes the proof.
\end{proof}

\subsection{Maintaining warm starts}
\label{s:warm}

We would like to maintain warm starts to the centers of all Gaussians (of non-negligible mass) as we decrease the noise level. By choosing the highest noise level large enough, $\mathbf 0$ will be a warm start. 
The key observation is that with high probability, the score function points in a direction close to a mean; this remains true with the estimated score when the error is small.

\begin{algorithm}[h!]
\caption{Picking set of warm starts}
\begin{algorithmic}[1]
    \INPUT Score estimate $s$, noise level $\si$, data points $y_1,\ldots, y_\samp\sim Q_{\si^2}$ ($\samp= \Om\pf{\ln (1/\de) k}{\al_{\min}}$), min weight $\al_{\min}$, %max radius of Gaussian means $\SR$, 
    failure probability $\de$.
    \State For each $i$, let $\wh \mu_i = y_i + \si^2 s(y_i)$.
    \State Let $U = \{\wh \mu_1,\ldots, \wh \mu_\samp\}$. \algorithmiccomment{Uncovered means}
    %\State Let $V=\phi$, $\wh{\cal C}=\phi$. (Covered means)
    \State Let $\wt R = C\pa{R_0+\si\sqrt{\ln \pf{1}{\al_{\min}}}}$ for an appropriate constant $C$.
    \For{$t=1$ to $C'k\ln \prc{\al_{\min}}$}{}\algorithmiccomment{Greedy set cover}
        \State Let $\wh \mu = \amax_{\mu \in U} 
        |B_\mu(\wt R) \cap U|$.
        %\{\wh \mu_1,\ldots, \wh \mu_m\}
        %\ab{\set{i}{|\wh\mu_i - \mu|\le C\sqrt{\ln (k/\ep)}}} $.
        \State Let $D =
        B_{\wh \mu}(\wt R) \cap U
        %\set{\wh\mu_i\in U}{|\wh\mu_i - \wh \mu|\le C\sqrt{\ln (k/\ep)}}
        $.
        \State $\wh{\cal C}\leftarrow \wh{\cal C}\cup \{\wh \mu\}$, %$V \leftarrow V\cup D$, 
        $U\leftarrow U\bs D$. 
    \EndFor{}
    \OUTPUT{Set of warm starts $\wh{\cal C}$}
\end{algorithmic}
\label{a:warm-starts}
\end{algorithm}

We first show that we do not lose too much in the score estimate if we only consider the means that are close to the mean that a data point came from.

\begin{lem}[Good score estimation $\implies$ Warm starts]\label{lem:scorewarmstart}
    Consider a mixture of Gaussians satisfying Assumption~\ref{a:mog} with $\si_0^2=1$. There is a universal constant such that the following holds.
    Let $f_{\si^2}(y) = y + \si^2 \nb \ln q_{\si^2}(y)$.  
    Suppose we are given a function $g$ satisfying
    \begin{align}
        \ve{f-g}_{L^2(Q_{\si^2})}^2 
\le (R_0 + \si)^2 \al_{\min}.
\label{e:good-score-est}
    \end{align}
    Suppose we have $\samp=\Om\pf{\ln (1/\de) k}{\al_{\min}}$ samples sampled from $q_{\si^2}$. Let $\cal C$ be the output of~\Cref{a:warm-starts}. Then with probability $\ge 1-\de$, for radius $\wt R= C\pa{R_0 + \si \sqrt{\ln \prc{\al_{\min}}}}$ with universal constant $C$, the support of $Q_0$ is contained in $\bigcup_{\wh\mu\in \cal C} B_{\wt R}(\wh \mu)$.
\end{lem}
\begin{proof}
    By~\eqref{e:good-score-est} and Chebyshev's inequality,
    \[
Q_{\si^2} (
\vn{f-g} \ge 4(R_0+\si) 
) = 
Q_{\si^2} (
\vn{f-g}^2 \ge 16(R_0+\si)^2 
) \le \fc{\ve{f-g}_{L^2(q_\si^2)}^2}{16(R_0+\si)^2} \le \fc{\al_{\min}}{16}.
    \]
%\fixme{Assume $\si_0^2=1$.}
    Consider drawing $\mu \sim Q_0$, $\xi\sim \cal N(0,I_n)$, and $Y=\mu+\si\xi$. 
    Let $\cal C =\{\ol \mu_1,\ldots, \ol \mu_k\}$ be as in \Cref{a:mog}. 
    Let $f_{\textup{loc}, \si^2} = f_{\textup{loc}, \si^2}^{\cal C, R'}$ be as in \Cref{l:nearby-score}, where $R'=3R_0 + 2\sqrt 2 \si \sqrt{\ln \pf{k}{\ep'}}$. Then 
\[
\E
\vn{f_{\si^2}(y) - f_{\textup{loc}, \si^2}(y)}^2 \lesssim 
\pa{R_0^2 + \si^2 \ln \pf k{\ep'}}
\ep' .
\]
By choosing $\ep' = \fc{c\al_{\min}}{\ln\prc{\al_{\min}}}$ for a small enough constant $c$ and noting $\al_{\min}\le \rc k$, we obtain that 
$\E
\vn{f_{\si^2}(y) - f_{\textup{loc}, \si^2}(y)}^2\le (R_0+\si)^2 \al_{\min}$.
Again by Chebyshev's inequality,
\[
\Pj \pa{\vn{f(y) - f_{\textup{loc}, \si^2}(y)} 
\ge 4(R_0+\si)
}\le\fc{\al_{\min}}{16}.
\]
Hence %with probability $\ge 1-\fc{\al_{\min}}2$, 
\begin{align}\label{e:g-floc}
\Pj\pa{\vn{g(y) - f_{\textup{loc}, \si^2}(y)}\ge 8(R_0+\si)} \le \fc{\al_{\min}}8.
\end{align}
Let $V_1,\ldots, V_k$ be the Voronoi partition corresponding to $\cal C$.
Letting $i$ be such that $\mu\in V_i$, we have that $\vn{f_{\textup{loc}, \si^2}(y) - \cmean_i} \le R'$.
Hence, under the event in~\eqref{e:g-floc}, by the triangle inequality,
% \[
% \ve{\E_{\trow{\mu'\sim Q_0|_{S(\mu)}}{Y'=\mu'+\xi}}[\mu' |Y'=y] - \mu} \le 
% \]
\[
\vn{g(y) - \ol \mu_i} \le R'+ 8(R_0+\si). 
\]
By assumption, $Q_0(B_{R_0%\si_0
}(\cmean_i))\ge \al_{\min}$, so letting $R''=R' + 8(R_0+\si) $, 
%R_0\si + C_2 r + 4\si \le C_3r
\begin{align*}
Q_{\si^2} (g(y) \in B_{R''}(\cmean_i))&\ge \fc{7\al_{\min}}8\\
Q_{\si^2} \pa{g(y) \in \bigcup_{i=1}^k B_{R''}(\cmean_i))}&\ge 1-\fc{\al_{\min}}8.
\end{align*}

By the Chernoff bounds, for independent $Z_1,\ldots, Z_\samp\sim \mathsf{Bernoulli}(p)$, for $c\ge 0$, 
\begin{align*}
\Pj \pa{\fc{Z_1+\cdots +Z_\samp}{\samp}\le (1-c)p} 
%\le e^{-D((1+c)p\|p)} 
&\le e^{-\fc{c^2p\samp}{2}}\\
\Pj \pa{\fc{Z_1+\cdots +Z_{\samp}}\samp\ge (1+c)p} 
&\le e^{-\fc{c^2p\samp}{2+c}}.
\end{align*}
%For $p=\fc{\al_{\min}}2$, $\ep = \fc{\al_{\min}}4$, it suffices to have $m\ge \fc{16\ln(k/\de)}{\al_{\min}}$ to make this $\le \de$. 
For fixed $c$, it suffices to have $\samp=\Om\pf{\ln (1/\de)}{p}$ to make this $\le \de$. 
Applying this to $p=\fc{7\al_{\min}}8$, $(1-c)p=\fc{3\al_{\min}}4$
and $p=\fc{\al_{\min}}8$, $(1+c)p=\fc{\al_{\min}}4$ respectively, by a union bound, 
given $\samp=\Om\pf{\ln(k/\de)}{\al_{\min}} $ iid draws $\mu_1,\ldots, \mu_m \sim Q_{\si^2}$, we have
\[
\Pj\coltwo{\forall 1\le i\le k, \, 
\ab{\set{j}{g(y_j) \in B_{R''}(\cmean_i)}}\ge \fc{3\al_{\min}}4 \samp}{ \text{ and }
\ab{\set{j}{g(y_j) \in \bigcup_{i=1}^k B_{R''}(\cmean_i)}}\ge \pa{1-\fc{\al_{\min}}{4}}\samp
}\ge 1-\de.
\]
Suppose this event holds. 
Consider the sets $B_{R''+R_0}(g(y_j))$. Choose $C$ such that \\
$\wt R = C\pa{R_0+\si\sqrt{\ln \pf{1}{\al_{\min}}}}\ge R''+R_0$.
For each $i$, choose $j(i)$ such that $g(y_{j(i)})\in B_{R''}(\cmean_i)$. Then $B_{R_0%\si_0
}(\cmean_i)\subeq B_{R''+R_0%\si
}(y_{j(i)})$ so by the second event, these $k$ sets cover $1-\fc{\al_{\min}}4$ proportion of the $g(y_j)$, $1\le j\le \samp$. To finish, apply \Cref{l:set-cover} to obtain that the output of the algorithm covers $1-\fc{\al_{\min}}2$ proportion of the $g(y_j)$. In light of the first event, for each $1\le i\le k$, it must contain some $g(y_j)\in B_{R''}(\cmean_i)$. This finishes the proof.
\end{proof}

\begin{lem}\label{l:set-cover}
    Let $\cal S$ be a set of subsets of $X$. % such that $\bigcup_{S\in \cal S} S =X$. 
    Let $k$ be the minimum number of sets in $\cal S$ required to cover $(1-\ep)|X|$ elements of $X$. 
    Consider the greedy algorithm where at each step, we take the set containing the most uncovered elements, as in \Cref{a:warm-starts}.
    Then the greedy algorithm finds $O\pa{k\ln\prc{\ep}}$ sets covering $(1-2\ep)|X|$ elements of $X$.
\end{lem}
The proof is based on the classic proof of the approximation ratio for set cover~\cite{johnson1973approximation}.
\begin{proof}
    Let $S_1,S_2,\ldots$ be the sets chosen by the greedy algorithm, let $U_i= \bigcup_{i'=1}^{i} S_{i'}$, 
    and let $x_1,x_2,\ldots$ be the covered elements in order. 
    Define the cost of an element $x_j$ as follows: Let $S_i$ be the first set where $x_j$ appears, and set
    \[
c_j = \rc{\ab{S_i\bs U_{i-1}}}.
    \]
    Now suppose that $U_{i-1} = \{x_1,\ldots, x_{j-1}\}$, and consider the cost of $x_j$. We claim that 
    \[
    c_j\le \fc{k}{(1-\ep)n - (j-1)}. 
    \]
    To see this, let $A_1,\ldots, A_k$ be an optimal cover of $(1-\ep)|X|$ elements of $X$. Then they must cover $(1-\ep)|X| - (j-1)$ elements of $X\bs U_{i-1}$, so 
    \[
    \sum_{\ell =1}^k |A_\ell \cap (X\bs U_{i-1})| \ge (1-\ep)n - (j-1)
    \]
    By optimality of $S_i$, $|S_i\cap (X\bs U_{i-1})|\ge \rc k ((1-\ep)n - (j-1))$, which shows the claim. Then the number of sets required to cover the first $(1-2\ep)|X|$ elements is given by the sum of costs of those elements. It is the ceiling of
    \[
\sumo j{\ce{(1-2\ep)n}} c_j 
\le \sumo j{\ce{(1-2\ep)n}} \fc{k}{(1-\ep)n - j+1}
\le k \sumz j{(1-2\ep)n} \rc{\ep n+j} = O\pa{k\ln \prc\ep}. %\qedhere
    \]
\end{proof}

\subsection{Proof of~\Cref{t:main} \&~\Cref{c:manifold}}
Combining all the pieces, we prove~\Cref{t:main} and~\Cref{c:manifold}.

\begin{proof}[Proof of \Cref{t:main}]
    Here we show that Algorithm~\ref{alg:mainalgo} successfully samples from $P_0 = Q_0 \convolve \cal N(0,\sigma_0^2)$ with the claimed time and sample complexity.
    Let $M_2 = \E_{x\sim P_0} \vn{x}^2$.
    Let $\wh P_{t_1}$ be the output of Algorithm~\ref{alg:mainalgo}. 
    By the triangle and Pinsker's inequality,
     \[\TV(P_0, \wh P_{t_1}) \leq \TV(P_0, P_{t_1}) + \TV(P_{t_1}, \wh P_{t_1}) = \TV(P_0, P_{t_1}) + \sqrt{\rc 2\KL(P_{t_1} \| \wh P_{t_1})}.\] 
    Choose the starting time $t_1 = \frac{\ep^2 \sigma_0^2}{2\sqrt n}$, so that by \Cref{lem:tvguarantee} we have $\TV(P_0, P_{t_1}) \leq \fc{\ep^2}{2}$.
    Hence it is sufficient to prove $\KL(P_{t_1}\|\wh P_{t_1}) \leq \fc{\ep^2}2$; actually it suffices to find parameters to make $\KL(P_{t_1}\| \wh P_{t_1}) = O(\ep^2)$ as the exact constant can be adjusted by rescaling $\ep$ by a constant. 
    We implement Algorithm~\ref{alg:mainalgo} with the step size schedule obtained recursively by the equality version of condition 3 in Theorem~\ref{t:dm-kl}, i.e. $t_k + 1 = (t_{k+1} + 1)\max\{e^{-2\kappa}, (t_{k+1} + 1)^{-\kappa}\}$ for $\kappa = \frac{\ep^2}{M_2 + n\ln(T+1)}$, with ending time $t_{\Nstep} = T = \frac{M_2 + d}{\ep^2}$, and number of iterations $\Nstep = O\pa{\frac{1}{\kappa}\ln\pa{\frac{T+1}{t_1}}}$. 
Then, given that we pick $ \ep_\ell^2= \fc{\ep^2}{\ln (T+1)}$, 
Theorem~\ref{t:dm-kl} tells us that  $\KL(P_{t_1}\| \wh P_{t_1})= O(\ep^2)$, as needed.

Hence the problem reduces to giving sufficiently accurate estimates of the score function, to guarantee an accuracy of $\ep_\ell^2= \fc{\ep^2(\si_\ell^2+1)}{\ln (T+1)}$ for all time steps $t_1,\dots, t_{\Nstep}$.  %total variation distance accuracy
 Obtaining such accurate scores is our main contribution and the proof consists of showing by backwards induction on $\ell=\Nstep,\dots,1$ that the following hold:
 \begin{enumerate}
     \item The set $\cal C_\ell$ is a complete set of $R_\ell$-warm starts for $Q_0$ (see \Cref{d:ws}), where 
 $$R_\ell \defeq C\pa{R_0 + 2\sigma_{\ell} \sqrt{\ln\pa{\pa{\frac{R_0}{\sigma_\ell} + 1}\frac{ k}{\alpha_{\min}}}}}$$ for a large enough universal constant $C$. Moreover, $|\cal C_\ell|\le k' = O\pa{k\ln \prc{\al_{\min}}}$. 
 \label{i:ws}
    \item The score estimate $s_{t_\ell}$ is $\ep_\ell$-accurate in $L^2(P_{t_\ell})$, for $\ep_\ell^2 = \frac{\ep^2(\si_\ell^2+1)}{\ln(T+1)}$.
    \label{i:s}
 \end{enumerate}
 More precisely, we will show that \ref{i:ws} holds for $\ell=\Nstep, \ldots, \Nstep-\ell'$ and \ref{i:s} holds for $\ell=\Nstep,\ldots, \Nstep-\ell'+1$
 with probability at least $1-\fc{\ell' \de}{\Nstep}$.
 
    The base case of the induction ($\ell = \Nstep, \,\ell'=0$) is to show that $\cal C_{\Nstep} = \{0\}$ is a complete set of $R_{\Nstep}$-warm-starts. But note that the assumption $\ep^2 \leq \frac{M_2 + n}{\Di^2}$ implies the variance at step $\Nstep$ satisfies $\sigma_{\Nstep}^2 = T + \sigma_0^2 \geq T = \frac{M_2 + n}{\ep^2} \geq \Di^2$, which given that we pick $C \geq 1$ means $R_{\Nstep} \geq \sigma_{\Nstep} \geq \Di$. Therefore, from the definition of $\Di$, we have that $Q_0$ is supported on $B_{\Di} \subseteq B_{R_{\Nstep}}$.
    
    Next, we show the induction step; here, the hypothesis of induction for step $\ell$ is that the set $\cal C_\ell$ is a complete set of $R_\ell$-warm starts for $Q_0$. Then, we show that \ref{i:s} for $\ell$ and \ref{i:ws} for $\ell-1$ are satisfied, after excluding an event of probability at most $\fc{\de}{\Nstep}$. 
    First we handle \ref{i:s}.
Using~\Cref{lem:lowdegree} with \[\wt \ep^2 \defeq \frac{\ep^2}{\pa{\pf{R_\ell}{\si_\ell}^2 + \ln \pf{k'R_\ell(M_2+n)}{\si_\ell\ep}} \ln (T+1)},\]
there exists a piece-wise 
polynomial $\wt g_\ell$ on the Voronoi partition of $\cal C_\ell$
that approximates the score function with the desired accuracy,
\begin{align}\label{eq:avalin}
    \ve{\wt g_\ell - f_{\sigma_\ell^2}}^2_{P_{t_\ell}} \lesssim 
    \wt \ep^2\pa{R_\ell^2 + \si_\ell^2 \ln \pf{k'}{\wt \ep}}
    \lesssim \fc{\ep^2(\si_\ell^2+1)}{\ln(T+1)}.
\end{align}
We note 
\begin{align*}
    \ln \prc{\wt \ep} = 
    O\pa{\ln \fc{\pf{R_{\ell}}{\si_{\ell}}(M_2+n)}{\ep}} = O\pa{\ln \prc{\ep}}
\end{align*}
assuming that $\ep\le \min\bc{ \rc 2, \fc{\si_0}{R_0}, \rc{D} , \rc n , \al_{\min}}$ and noting $\si_0\le \si_\ell$, $\al_{\min}\le \rc k$. 
The degree of $\wt g_\ell$ is at most 
    \begin{align}
    \nonumber
    d_{\ell} &= 
    O\pa{\pa{\fc{R_\ell}{\si_\ell} + \sqrt{\ln \pf{k'}{\wt \ep}}}^6\ln \prc{\wt \ep}^4} \\
    \nonumber
    &= 
    O\pa{\pa{\frac{R_0}{\sigma_\ell} + \sqrt{\ln\pa{\pa{1 + \frac{R_0}{\sigma_\ell}}\frac{k}{ \alpha_{\min}}}} + \sqrt{\ln \pf{1}{\ep}}}^6 \ln \prc{\ep}^4} \\
    &= 
    O\pa{\pa{\ln \pf{1}{\ep}^3 + \pa{\frac{R_0}{\sigma_0}}^6}\ln\pa{\frac{1}{\ep}}^4}.\label{eq:dbound}
    %&= O\pa{\pa{\ln \pf{k'\Di}{\ep} + }}.\numberthis
    % \\
    % &= O\pa{\pa{\ln \pf{k'\Di}{\ep} + \frac{R_0}{\sigma_\ell}}^5}.
    \end{align}
    Now for an arbitrary warm start point $\wh \mu^{(\ell)}_i\in \cal C_\ell$,
    by \Cref{lem:lowdegree}, there exists a polynomial $\wt g_\ell = \sum_{|\mk|\le d_\ell} \wt b_\mk^{(i)} h_\mk (y-\wh \mu_i^{(\ell)})$ where $h_\mk = h_{\mk,\si_\ell^2}$ are the Hermite polynomials with variance $\si^2$, and  
    the coefficients satisfy
    \begin{align*}
        \vn{(\wt b_\mk^{(i)})_{|\mk|\le d_\ell}}
         = \ve{\wt g_\ell}_{\gamma_{\wh \mu_i,\sigma_\ell^2}} \leq \Di.
    \end{align*}

    But letting $d_\ell$ be the RHS of~\eqref{eq:dbound} with appropriate constants, the condition $d_\ell \geq (R_\ell/\sigma_\ell)^2$ is satisfied, and hence we can use~\Cref{lem:gengap} with parameters $\Di, \Rn$ both set to $\Di$; the implication is that given 
    \begin{align*}
        \samp = %\pa{
        {\ep'}^{-4}n\pa{k (\Di + 1)^2\ln \prc{\al_{\min}}}^2 
        \pa{n \ln \pf{k'\Nstep}{\de}}^{cd_\ell}
        % \pa{\frac{n \ln\pf{k}{\alpha_{\min}\delta}}{c'\pa{\ln \pf{k\Di}{\ep}^3 + \pa{\frac{R_0}{\sigma_0}}^6}\ln\pa{\frac{1}{\ep}}^2}}^{c' \pa{\ln \pf{k\Di}{\ep}^3 + \pa{\frac{R_0}{\sigma_0}}^6}\ln\pa{\frac{1}{\ep}}^2}
        %}
    \end{align*}
    samples for ${\ep'}^2 = \fc{\ep^2(\si_\ell+1)}{\ln (T+1)}$ and some universal constant $c$, we have the guarantee 
    \begin{align*}
    \ve{\wh g_\ell - f}_{L^2(P_{t_\ell})}^2 - 
    \ve{\wt g_\ell - f}_{L^2(P_{t_\ell})}^2 \leq  {\ep'}^2
    %\frac{\wt \ep^2}{2}.
\end{align*}
after excluding an event of probability at most $\fc{\de}{2\Nstep}$. 
Combining this with Equation~\eqref{eq:avalin}, we obtain
\begin{align*}
    \ve{\wh g_\ell - f}_{L^2(P_{t_\ell})}^2 \leq O({\ep'}^2) 
\end{align*}
as desired. 

Next, we handle \ref{i:ws}, i.e. show that $\cal C_{\ell-1}$ is a complete set of $R_{\ell-1}$-warm starts for $Q_0$. Note that when $\si_\ell$ changes by a constant factor, \ref{i:ws} is still satisfied with a modified constant, so we only have to update $\cal C_\ell$ each time $t_\ell+1$ halves (as done in \Cref{alg:mainalgo}). 
According to Lemma~\ref{lem:scorewarmstart}, it is sufficient to find an estimate $g$ for the score which satisfies 
$\ve{f-g}_{L^2(P_{t_\ell})}^2 
\le (R_0+\si_\ell)^2 \al_{\min}.$ 
But from our assumption we have 
${\ep'}^2 = \fc{\ep^2(\si_\ell+1)}{\ln (T+1)} \le (R_0+\si_\ell)^2\al_{\min}$,
%= \frac{\ep^2}{\Di\ln(k')} \leq \frac{\ep^2}{\Di} \leq \sigma_0^2 \alpha_{\min}$
 which means the polynomial $\wt g_\ell$ that we obtained for proving part \ref{i:s} in the induction already satisfies the required accuracy for updating the warm starts as well; hence we can use the same degree of polynomial for our regression here. Again we exclude an event of probability at most $\fc{\de}{2\Nstep}$. Note that the sample size required by \Cref{a:warm-starts}  is $O\pf{\ln (\Nstep/\de) k}{\al_{\min}}$, which is negligible compared to the number of samples needed for regression. This completes the induction step.

 Multiplying $N$ by the number of steps $\Nstep$ and dropping lower-order terms (recalling the assumption on $\ep$), we get that the total sample complexity is 
 \[
\samp^{\textup{total}} = 
\pa{n \ln \prc{\de}}^{O(d_\ell)}
= \pa{n \ln \prc{\de}}^{O\pa{\pa{\ln \pf{1}{\ep}^3 + \pa{\frac{R_0}{\sigma_0}}^6}\ln\pa{\frac{1}{\ep}}^4}}.
 \]  

Note that the regression problems that we solve to obtain score estimates in Algorithm~\ref{alg:mainalgo} can still use the same batch of samples, %as the DDPM framework that we use in 
 because we apply a union bound. 
 %Theorem~\ref{t:dm-kl} does not require them to be independent. 
 However, we do need to use fresh samples each time for updating the warm starts, separate from the ones that we use for score estimates. This is because our uniform convergence bounds in~\Cref{lem:gengap} require the condition that the samples are independent from the randomness used in the construction of the Voronoi partition. 
\end{proof}

 \begin{proof}[Proof of \Cref{c:manifold}]
 This follows from \Cref{t:main} with $\al_{\min} = \fc{\ep}{C^l}$ and substituting\\
 $\min\bc{\fc{\ep}{C^l}, \rc2, \rc{R_0}, \rc n, \rc D}$ for $\ep$ in \Cref{t:main}, noting $\ln \pf{C^l}{\ep} = O\pa{l + \ln \prc{\ep}}$.
 \end{proof}

\section*{Acknowledgements}

The authors would like to thank Sitan Chen and Allen Liu for helpful discussions.
Jonathan Kelner's work on this project was partially supported by NSF Medium CCF-1955217 and NSF TRIPODS 1740751.

\bibliographystyle{apalike}
\bibliography{main.bib}

\begin{thebibliography}{}

\bibitem[Acharya et~al., 2017]{acharya2017sample}
Acharya, J., Diakonikolas, I., Li, J., and Schmidt, L. (2017).
\newblock Sample-optimal density estimation in nearly-linear time.
\newblock In {\em Proceedings of the Twenty-Eighth Annual ACM-SIAM Symposium on Discrete Algorithms}, pages 1278--1289. SIAM.

\bibitem[Anderson, 1982]{anderson1982reverse}
Anderson, B.~D. (1982).
\newblock Reverse-time diffusion equation models.
\newblock {\em Stochastic Processes and their Applications}, 12(3):313--326.

\bibitem[Arora and Kannan, 2005]{arora2005learning}
Arora, S. and Kannan, R. (2005).
\newblock Learning mixtures of separated nonspherical gaussians.
\newblock {\em Annals of Applied Probability}, pages 69--92.

\bibitem[Ashtiani et~al., 2018]{ashtiani2018nearly}
Ashtiani, H., Ben-David, S., Harvey, N., Liaw, C., Mehrabian, A., and Plan, Y. (2018).
\newblock Nearly tight sample complexity bounds for learning mixtures of gaussians via sample compression schemes.
\newblock {\em Advances in Neural Information Processing Systems}, 31.

\bibitem[Bakshi et~al., 2022]{bakshi2022robustly}
Bakshi, A., Diakonikolas, I., Jia, H., Kane, D.~M., Kothari, P.~K., and Vempala, S.~S. (2022).
\newblock Robustly learning mixtures of k arbitrary gaussians.
\newblock In {\em Proceedings of the 54th Annual ACM SIGACT Symposium on Theory of Computing}, pages 1234--1247.

\bibitem[Benjamini et~al., 1999]{benjamini1999noise}
Benjamini, I., Kalai, G., and Schramm, O. (1999).
\newblock Noise sensitivity of boolean functions and applications to percolation.
\newblock {\em Publications Math{\'e}matiques de l'Institut des Hautes {\'E}tudes Scientifiques}, 90:5--43.

\bibitem[Benton et~al., 2023]{benton2023linear}
Benton, J., De~Bortoli, V., Doucet, A., and Deligiannidis, G. (2023).
\newblock Linear convergence bounds for diffusion models via stochastic localization.
\newblock {\em arXiv preprint arXiv:2308.03686}.

\bibitem[Bietti et~al., 2023]{bietti2023learning}
Bietti, A., Bruna, J., and Pillaud-Vivien, L. (2023).
\newblock On learning gaussian multi-index models with gradient flow.
\newblock {\em arXiv preprint arXiv:2310.19793}.

\bibitem[Biroli and M{\'e}zard, 2023]{biroli2023generative}
Biroli, G. and M{\'e}zard, M. (2023).
\newblock Generative diffusion in very large dimensions.
\newblock {\em Journal of Statistical Mechanics: Theory and Experiment}, 2023(9):093402.

\bibitem[Chen et~al., 2023a]{chen2023improved}
Chen, H., Lee, H., and Lu, J. (2023a).
\newblock Improved analysis of score-based generative modeling: User-friendly bounds under minimal smoothness assumptions.
\newblock In {\em International Conference on Machine Learning}, pages 4735--4763. PMLR.

\bibitem[Chen et~al., 2023b]{chen2023score}
Chen, M., Huang, K., Zhao, T., and Wang, M. (2023b).
\newblock Score approximation, estimation and distribution recovery of diffusion models on low-dimensional data.
\newblock {\em arXiv preprint arXiv:2302.07194}.

\bibitem[Chen et~al., 2024]{chen2024learning}
Chen, S., Kontonis, V., and Shah, K. (2024).
\newblock Learning general gaussian mixtures with efficient score matching.

\bibitem[Chen and Eldan, 2022]{chen2022localization}
Chen, Y. and Eldan, R. (2022).
\newblock Localization schemes: A framework for proving mixing bounds for markov chains.
\newblock In {\em 2022 IEEE 63rd Annual Symposium on Foundations of Computer Science (FOCS)}, pages 110--122. IEEE.

\bibitem[Cole and Lu, 2024]{cole2024score}
Cole, F. and Lu, Y. (2024).
\newblock Score-based generative models break the curse of dimensionality in learning a family of sub-gaussian probability distributions.
\newblock {\em arXiv preprint arXiv:2402.08082}.

\bibitem[Cui et~al., 2023]{cui2023analysis}
Cui, H., Krzakala, F., Vanden-Eijnden, E., and Zdeborov{\'a}, L. (2023).
\newblock Analysis of learning a flow-based generative model from limited sample complexity.
\newblock {\em arXiv preprint arXiv:2310.03575}.

\bibitem[Dasgupta and Schulman, 2000]{dasgupta2000two}
Dasgupta, S. and Schulman, L.~J. (2000).
\newblock A two-round variant of em for gaussian mixtures.
\newblock In {\em Proceedings of the Sixteenth conference on Uncertainty in artificial intelligence}, pages 152--159.

\bibitem[De~Bortoli, 2022]{de2022convergence}
De~Bortoli, V. (2022).
\newblock Convergence of denoising diffusion models under the manifold hypothesis.
\newblock {\em arXiv preprint arXiv:2208.05314}.

\bibitem[Diakonikolas and Kane, 2020]{diakonikolas2020small}
Diakonikolas, I. and Kane, D.~M. (2020).
\newblock Small covers for near-zero sets of polynomials and learning latent variable models.
\newblock In {\em 2020 IEEE 61st Annual Symposium on Foundations of Computer Science (FOCS)}, pages 184--195. IEEE.

\bibitem[Diakonikolas et~al., 2017]{diakonikolas2017statistical}
Diakonikolas, I., Kane, D.~M., and Stewart, A. (2017).
\newblock Statistical query lower bounds for robust estimation of high-dimensional gaussians and gaussian mixtures.
\newblock In {\em 2017 IEEE 58th Annual Symposium on Foundations of Computer Science (FOCS)}, pages 73--84. IEEE.

\bibitem[Doss et~al., 2020]{doss2020optimal}
Doss, N., Wu, Y., Yang, P., and Zhou, H.~H. (2020).
\newblock Optimal estimation of high-dimensional location gaussian mixtures.
\newblock {\em arXiv preprint arXiv:2002.05818}.

\bibitem[El~Alaoui et~al., 2022]{el2022sampling}
El~Alaoui, A., Montanari, A., and Sellke, M. (2022).
\newblock Sampling from the sherrington-kirkpatrick gibbs measure via algorithmic stochastic localization.
\newblock In {\em 2022 IEEE 63rd Annual Symposium on Foundations of Computer Science (FOCS)}, pages 323--334. IEEE.

\bibitem[Eldan, 2013]{eldan2013thin}
Eldan, R. (2013).
\newblock Thin shell implies spectral gap up to polylog via a stochastic localization scheme.
\newblock {\em Geometric and Functional Analysis}, 23(2):532--569.

\bibitem[Gupte et~al., 2022]{gupte2022continuous}
Gupte, A., Vafa, N., and Vaikuntanathan, V. (2022).
\newblock Continuous lwe is as hard as lwe \& applications to learning gaussian mixtures.
\newblock In {\em 2022 IEEE 63rd Annual Symposium on Foundations of Computer Science (FOCS)}, pages 1162--1173. IEEE.

\bibitem[Hopkins and Li, 2018]{hopkins2018mixture}
Hopkins, S.~B. and Li, J. (2018).
\newblock Mixture models, robustness, and sum of squares proofs.
\newblock In {\em Proceedings of the 50th Annual ACM SIGACT Symposium on Theory of Computing}, pages 1021--1034.

\bibitem[Ibragimov and Sharakhmetov, 2001]{ibragimov2001best}
Ibragimov, R. and Sharakhmetov, S. (2001).
\newblock The best constant in the {R}osenthal inequality for nonnegative random variables.
\newblock {\em Statistics \& probability letters}, 55(4):367--376.

\bibitem[Johnson, 1973]{johnson1973approximation}
Johnson, D.~S. (1973).
\newblock Approximation algorithms for combinatorial problems.
\newblock In {\em Proceedings of the fifth annual ACM symposium on Theory of computing}, pages 38--49.

\bibitem[Kalai et~al., 2008]{kalai2008agnostically}
Kalai, A.~T., Klivans, A.~R., Mansour, Y., and Servedio, R.~A. (2008).
\newblock Agnostically learning halfspaces.
\newblock {\em SIAM Journal on Computing}, 37(6):1777--1805.

\bibitem[Kim and Guntuboyina, 2022]{kim2022minimax}
Kim, A.~K. and Guntuboyina, A. (2022).
\newblock Minimax bounds for estimating multivariate gaussian location mixtures.
\newblock {\em Electronic Journal of Statistics}, 16(1):1461--1484.

\bibitem[Klivans et~al., 2004]{klivans2004learning}
Klivans, A.~R., O'Donnell, R., and Servedio, R.~A. (2004).
\newblock Learning intersections and thresholds of halfspaces.
\newblock {\em Journal of Computer and System Sciences}, 68(4):808--840.

\bibitem[Klivans et~al., 2008]{klivans2008learning}
Klivans, A.~R., O'Donnell, R., and Servedio, R.~A. (2008).
\newblock Learning geometric concepts via gaussian surface area.
\newblock In {\em 2008 49th Annual IEEE Symposium on Foundations of Computer Science}, pages 541--550. IEEE.

\bibitem[Koehler et~al., 2022]{koehler2022statistical}
Koehler, F., Heckett, A., and Risteski, A. (2022).
\newblock Statistical efficiency of score matching: The view from isoperimetry.
\newblock {\em arXiv preprint arXiv:2210.00726}.

\bibitem[Kothari et~al., 2018]{kothari2018robust}
Kothari, P.~K., Steinhardt, J., and Steurer, D. (2018).
\newblock Robust moment estimation and improved clustering via sum of squares.
\newblock In {\em Proceedings of the 50th Annual ACM SIGACT Symposium on Theory of Computing}, pages 1035--1046.

\bibitem[Lee et~al., 2023]{lee2023convergence}
Lee, H., Lu, J., and Tan, Y. (2023).
\newblock Convergence of score-based generative modeling for general data distributions.
\newblock In {\em International Conference on Algorithmic Learning Theory}, pages 946--985. PMLR.

\bibitem[Li and Chen, 2024]{li2024critical}
Li, M. and Chen, S. (2024).
\newblock Critical windows: non-asymptotic theory for feature emergence in diffusion models.
\newblock {\em arXiv preprint arXiv:2403.01633}.

\bibitem[Liu and Li, 2022]{liu2022clustering}
Liu, A. and Li, J. (2022).
\newblock Clustering mixtures with almost optimal separation in polynomial time.
\newblock In {\em Proceedings of the 54th Annual ACM SIGACT Symposium on Theory of Computing}, pages 1248--1261.

\bibitem[Liu and Moitra, 2021]{liu2021settling}
Liu, A. and Moitra, A. (2021).
\newblock Settling the robust learnability of mixtures of gaussians.
\newblock In {\em Proceedings of the 53rd Annual ACM SIGACT Symposium on Theory of Computing}, pages 518--531.

\bibitem[Mei and Wu, 2023]{mei2023deep}
Mei, S. and Wu, Y. (2023).
\newblock Deep networks as denoising algorithms: Sample-efficient learning of diffusion models in high-dimensional graphical models.
\newblock {\em arXiv preprint arXiv:2309.11420}.

\bibitem[Moitra and Valiant, 2010]{moitra2010settling}
Moitra, A. and Valiant, G. (2010).
\newblock Settling the polynomial learnability of mixtures of gaussians.
\newblock In {\em 2010 IEEE 51st Annual Symposium on Foundations of Computer Science}, pages 93--102. IEEE.

\bibitem[Montanari, 2023]{montanari2023sampling}
Montanari, A. (2023).
\newblock Sampling, diffusions, and stochastic localization.
\newblock {\em arXiv preprint arXiv:2305.10690}.

\bibitem[Nagaev and Pinelis, 1978]{nagaev1978some}
Nagaev, S. and Pinelis, I. (1978).
\newblock Some inequalities for the distribution of sums of independent random variables.
\newblock {\em Theory of Probability \& Its Applications}, 22(2):248--256.

\bibitem[Nelson, 1967]{nelson1967dynamical}
Nelson, E. (1967).
\newblock {\em Dynamical theories of Brownian motion}, volume 101.
\newblock Princeton university press.

\bibitem[Oko et~al., 2023]{oko2023diffusion}
Oko, K., Akiyama, S., and Suzuki, T. (2023).
\newblock Diffusion models are minimax optimal distribution estimators.
\newblock {\em arXiv preprint arXiv:2303.01861}.

\bibitem[Pabbaraju et~al., 2024]{pabbaraju2024provable}
Pabbaraju, C., Rohatgi, D., Sevekari, A.~P., Lee, H., Moitra, A., and Risteski, A. (2024).
\newblock Provable benefits of score matching.
\newblock {\em Advances in Neural Information Processing Systems}, 36.

\bibitem[Qin and Risteski, 2023]{qin2023fit}
Qin, Y. and Risteski, A. (2023).
\newblock Fit like you sample: Sample-efficient generalized score matching from fast mixing markov chains.
\newblock {\em arXiv preprint arXiv:2306.09332}.

\bibitem[Regev and Vijayaraghavan, 2017]{regev2017learning}
Regev, O. and Vijayaraghavan, A. (2017).
\newblock On learning mixtures of well-separated gaussians.
\newblock In {\em 2017 IEEE 58th Annual Symposium on Foundations of Computer Science (FOCS)}, pages 85--96. IEEE.

\bibitem[Robbins, 1992]{robbins1992empirical}
Robbins, H.~E. (1992).
\newblock An empirical bayes approach to statistics.
\newblock In {\em Breakthroughs in Statistics: Foundations and basic theory}, pages 388--394. Springer.

\bibitem[Saha and Guntuboyina, 2020]{saha2020nonparametric}
Saha, S. and Guntuboyina, A. (2020).
\newblock On the nonparametric maximum likelihood estimator for gaussian location mixture densities with application to gaussian denoising.
\newblock {\em The Annals of Statistics}, 48(2):738--762.

\bibitem[Schwab and Zech, 2023]{schwab2023deep}
Schwab, C. and Zech, J. (2023).
\newblock Deep learning in high dimension: Neural network expression rates for analytic functions in ${L}^2(\mathbb{R}^d, \gamma_d)$.
\newblock {\em SIAM/ASA Journal on Uncertainty Quantification}, 11(1):199--234.

\bibitem[Shah et~al., 2023]{shah2023learning}
Shah, K., Chen, S., and Klivans, A. (2023).
\newblock Learning mixtures of gaussians using the ddpm objective.
\newblock {\em arXiv preprint arXiv:2307.01178}.

\bibitem[Sohl-Dickstein et~al., 2015]{sohl2015deep}
Sohl-Dickstein, J., Weiss, E., Maheswaranathan, N., and Ganguli, S. (2015).
\newblock Deep unsupervised learning using nonequilibrium thermodynamics.
\newblock In {\em International Conference on Machine Learning}, pages 2256--2265. PMLR.

\bibitem[Song and Ermon, 2019]{song2019generative}
Song, Y. and Ermon, S. (2019).
\newblock Generative modeling by estimating gradients of the data distribution.
\newblock In {\em Proceedings of the 33rd Annual Conference on Neural Information Processing Systems}.

\bibitem[Song et~al., 2020]{song2020score}
Song, Y., Sohl-Dickstein, J., Kingma, D.~P., Kumar, A., Ermon, S., and Poole, B. (2020).
\newblock Score-based generative modeling through stochastic differential equations.
\newblock In {\em International Conference on Learning Representations}.

\bibitem[Talagrand, 2010]{talagrand2010mean}
Talagrand, M. (2010).
\newblock {\em Mean field models for spin glasses: Volume I: Basic examples}, volume~54.
\newblock Springer Science \& Business Media.

\bibitem[Tang and Zhao, 2024]{tang2024score}
Tang, W. and Zhao, H. (2024).
\newblock Score-based diffusion models via stochastic differential equations--a technical tutorial.
\newblock {\em arXiv preprint arXiv:2402.07487}.

\bibitem[Titterington et~al., 1985]{titterington1985statistical}
Titterington, D.~M., Smith, A.~F., and Makov, U.~E. (1985).
\newblock Statistical analysis of finite mixture distributions.
\newblock {\em Chichester-New York: J. Willey \& Sons}, 646.

\bibitem[Vempala and Wang, 2004]{vempala2004spectral}
Vempala, S. and Wang, G. (2004).
\newblock A spectral algorithm for learning mixture models.
\newblock {\em Journal of Computer and System Sciences}, 68(4):841--860.

\bibitem[Vershynin, 2020]{vershynin2020high}
Vershynin, R. (2020).
\newblock High-dimensional probability.
\newblock {\em University of California, Irvine}, 10:11.

\bibitem[Vincent, 2011]{vincent2011connection}
Vincent, P. (2011).
\newblock A connection between score matching and denoising autoencoders.
\newblock {\em Neural computation}, 23(7):1661--1674.

\bibitem[Wibisono et~al., 2024]{wibisono2024optimal}
Wibisono, A., Wu, Y., and Yang, K.~Y. (2024).
\newblock Optimal score estimation via empirical bayes smoothing.
\newblock {\em arXiv preprint arXiv:2402.07747}.

\bibitem[Wu et~al., 2024]{wu2024theoretical}
Wu, Y., Chen, M., Li, Z., Wang, M., and Wei, Y. (2024).
\newblock Theoretical insights for diffusion guidance: A case study for gaussian mixture models.
\newblock {\em arXiv preprint arXiv:2403.01639}.

\end{thebibliography}

%\newpage

\appendix
\section{Proof of~\Cref{t:dm-kl}}
\label{s:dm-appendix}
We conclude the proof of~\Cref{t:dm-kl} from the analogous guarantee for the VP process.
\begin{thm}[Reverse KL guarantee for variance-preserving diffusion models {\cite[Theorem 2]{benton2023linear}}]\label{t:dm-kl-orig}
Let $0<t_1<t_2<\ldots < t_{\Nstep}=T$ and $\zeta_k = t_{k+1}-t_k$. Suppose $T\ge 1$.  
Let $P_t^{\textup{OU}}$ denote the distribution of the Ornstein-Uhlenbeck process run for time $t$ initialized from $P_0$, that is, the distribution of $e^{-t}X+ \sqrt{1-e^{-2t}}\cdot \xi$ where $X\sim P_0$ and $\xi\sim \cal N(0,I_d)$. 
%Consider the following assumptions.
Assume the following.
\begin{enumerate}
    \item 
    We have a score function estimate $s_t(x)$ for each $t=t_k$ such that 
    \[
    \E_{P_{t_k}^{\textup{OU}}}\vn{\nb \ln  p_{t_k}^{\textup{OU}}(x) - s_{t_k}(x)}^2 \le \ep_k^2.
    \]
    \item The data distribution has bounded second moment $M_2 = \E_{P_0}\vn{x}^2$.
    %\item For all $t\in [t_k,t_{k+1}]$, $\nb \ln p_t(x)$ is $L_k$-Lipschitz.
    \item The 
    step size schedule satisfies
    $
\zeta_k \le \ka \min\{1,t_k\}
    $ (i.e. $\zeta_k\le \min\bc{\ka, \fc{\ka}{\ka+1} t_{k+1}}$)
    %step sizes satisfy $h_{k-1}\le \fc{\ka}2 (T-t_k)$ where $\ka\le 1$.
\end{enumerate}
Let $\wh p_{t}$ denote the distribution of the exponential integrator discretization of the reverse process when initialized at $\wh P_T = \cal N(0,I_n)$, given by
\[
        \wh x_{t_k} = e^{\zeta_k} \wh x_{t_{k+1}} + 2(e^{\zeta_k}-1) s_{t_{k+1}}^x(\wh x_{t_{k+1}}) + \sqrt{e^{2\zeta_{k}}-1} \cdot \xi_{t_k}, \quad \xi_{t_k}\sim \cal N(0,I_n).  
\]
Then
\begin{align}
\label{e:dm-kl-orig}
\KL(p_{t_1}\|\wh p_{t_1}) 
\lesssim (n+M_2)e^{-2T} + \sumo k{\Nstep} \zeta_k \ep_k^2 + 
\ka n T + \ka^2 n\Nstep + \ka M_2.
\end{align}
%Moreover, we can choose a schedule of length $O\pa{\rc{\ka} \ln\pf{T+1}{t_1}}$ to make assumption 3 hold.
\end{thm}
Note that as opposed to \cite{benton2023linear}, we index time in the forward direction.

\begin{proof}[Proof of Theorem~\ref{t:dm-kl}]
    Consider the following two sets of continuous processes,
    \begin{align*}
        dx_t &= -x_t \,dt + \sqrt 2 \,dW_t &
        dy_t &= dW_t
        % dx_t^{\leftarrow} &= x_t^{\leftarrow} + \nb \ln p_{T-t}(x_t) + \sqrt 2\cdot  d\wt W_t&
        % dy_t^{\leftarrow} &= \nb \ln q_{T-t}(x_t) + d\wt W_t\\
        % \wh x_{t+h} &= e^h \wh x_t + h s_{T-t}^x(\wh x_t) + \sqrt{2h} \cdot \xi_t &
        % \wh y_{t+h} &= \wh y_t + h s_{T-t}^y(\wh y_t) + \sqrt{2h}\cdot  \xi_t.
    \end{align*}
    with same initial distribution $x_0,y_0\sim P$. Let $P_t^x$ and $P_t^y$ be the distribution of the two processes at time $t$. 
    We claim that 
    \begin{align*}
        x_t &= e^{-t} y_{e^{2t}-1}.
        % \prc{e^t}^d\cdot p_t\pf{x}{e^t}  &= q_{e^{2t-1}}(y)\\
    \end{align*}
    To see this, we check that $e^{-t} y_{e^{2t}-1}$ satisfies the equation for $x_t$:
    \begin{align*}
        d(e^{-t}y_{e^{2t}-1})
        &= -e^{-t} y_{e^{2t}-1} \,dt 
        + e^{-t} \,dy_{e^{2t}-1} \\
        &= -(e^{-t}y_{e^{2t}-1})\,dt
        + e^{-t} \sqrt{2e^{2t}} \,dW_t\\
        &= -(e^{-t}y_{e^{2t}-1})\,dt
        + \sqrt 2\,dW_t.
    \end{align*}
    By the change-of-variables formula,
    \begin{align*}
        p_t^x(x) &= e^{tn} p_{e^{2t}-1}(e^t x)\\
        \implies\nb \ln p_t^x(x) &= e^t \nb \ln p_{e^{2t}-1}(e^tx).
    \end{align*}
    Write $s_t^y=s_t$ for clarity. Defining $s_t^x$ to satisfy
    \begin{align*}
        s_t^x(x) &= e^t s_{e^{2t}-1}^y(e^tx),
    \end{align*}
    we have
    \begin{align*}
        \E_{P_t^x}\vn{\nb \ln p_t^x(x) - s_t^x(x)}^2 &= e^{2t} \E_{P_t^y} \vn{\nb \ln p_{e^{2t}}^y (y) - s_{e^{2t}-1}^y(y)}^2.
    \end{align*}
    Now define
    \begin{align*}
    t_k^x &= \fc{\ln (t_k+1)}2, &
    \zeta_k^x &= t_{k+1}^x - t_k^x,
        % t_k^y &= e^{2t_k}-1, &
        % h_k^y &= t_{k+1}^y - t_k^y.
    \end{align*}
    and consider the discrete processes defined backwards in time, %, where the step size $h$ chosen here depends on the current time,
    \begin{align*}
        \wh x_{t} &= e^\zeta \wh x_{t+h} + 2(e^\zeta-1) s_{t+\zeta}^x(\wh x_{t+\zeta}) + \sqrt{e^{2\zeta}-1} \cdot \xi_t^x  & (t,\zeta) &= (t_k^x,h_k^x)\\
        \wh y_{t} &= \wh y_{t+\zeta} + %h
        2\ba{(t+\zeta-1) - \sqrt{(t-1)(t+\zeta-1)}}
        s_{t+\zeta}^y(\wh y_{t+\zeta}) + \sqrt{\zeta}\cdot  \xi_t& (t,\zeta) &= (t_k,\zeta_k),
    \end{align*}
    where $x_{t_{\Nstep}^x}\sim P_{t_{\Nstep}^x}^x$ and $y_{t_{\Nstep}} = e^{t_{\Nstep}^x} x_{t_{\Nstep}^x}$ (so that $y_{t_{\Nstep}}\sim P_{t_{\Nstep}}^y$), and we couple $\xi_{t_k^x}^x = \xi_{t_k}$.
    %\holden{Note this is a nonstandard discretization... odd that the factor is not $h$.}
    We can inductively check that 
    \begin{align*}
    \wh x_{t_k} &= e^{-t_k^x} \wh y_{t_k}.
        %\wh y_{t_k^y} &= e^{t_k} \wh x_{t_k}.
    \end{align*}
    Indeed, if this holds for $k+1$, then
\begin{align*}
    \wh x_{t_k^x} &= 
    e^{\zeta_k^x} \wh x_{t_{k+1}^x} + 2(e^{\zeta_k^x}-1) s_{t_{k+1}^x} (\wh x_{t_{k+1}^x}) + \sqrt{e^{2\zeta_k^x}-1} \cdot \xi_t^x\\
    &= e^{\zeta_k^x} (e^{-t_{k+1}^x} \wh y_{t_{k+1}})
    + 2(e^{\zeta_k^x}-1) e^{t_{k+1}^x} s_{t_{k+1}}^y (\wh y_{t_{k+1}}) + \sqrt{e^{2\zeta_k^x}-1} \cdot \xi_t^x\\
    &= e^{-t_k^x}\pa{
        \wh y_{t_{k+1}} + 2(e^{\zeta_k^x}-1) e^{t_{k+1}^x+t_k^x} s_{t_{k+1}}^y (\wh y_{t_{k+1}}) + \sqrt{t_{k+1} - t_k} \cdot \xi_t
    }\\
    &= e^{-t_k^x}\pa{
        \wh y_{t_{k+1}} + 2
        \ba{(t_{k+1}-1) - \sqrt{(t_{k}-1)(t_{k+1}-1)}}
        s_{t_{k+1}}^y (\wh y_{t_{k+1}}) + \sqrt{t_{k+1} - t_k} \cdot \xi_t
    }\\
    &= e^{-t_k^x} \wh y_{t_{k}}.
\end{align*}
Hence, by Theorem~\ref{t:dm-kl-orig},
\begin{align*}
    \KL(p_{t_1}^y\|\wh p_{t_1}^y)
    = 
    \KL(p_{t_1}^x\|\wh p_{t_1}^x)
    &\lesssim
    (n+M_2)e^{-2t_{\Nstep}^x} + \sumo k{\Nstep} \zeta_k^x \cdot %\rc{t_k^y+1}
    e^{-2t_k^x}
    \cdot \ep_k^2 + 
\ka n t_{\Nstep}^x + \ka^2 n\Nstep + \ka M_2\\
&\lesssim
    \fc{n+M_2}{T+1} + \sumo k{\Nstep} \ln \pf{t_{k+1}^y+1}{t_k^y+1} \cdot \rc{t_k^y+1}
    \cdot \ep_k^2 + 
\ka n \ln (T+1) + \ka^2 n\Nstep + \ka M_2.
\end{align*}
Because $y_T = \sqrt{T+1} \cdot x_{\fc{\ln (T+1)}2}$, 
the initialization $x_{t_N^x}\sim \cal N(0, I)$ corresponds to $y_T\sim \cal N(0,(T+1)\cdot I)$.
The requirement on step sizes is
\begin{align*}
    \fc{\ln(t_{k+1}+1) - \ln (t_k+1)}{2} &\le \min\bc{\ka, \fc{\ka}{\ka+1}\cdot \fc{\ln (t_{k+1}+1)}{2}}\\
    \iff
    t_{k+1}+1 &\ge (t_{k+1}+1) \max\{e^{-2\ka}, (t_{k+1}+1)^{-\fc{\ka}{\ka+1}}\}.
\end{align*}
Note that 
\[\max\{e^{-2\ka}, (t+1)^{-\fc{\ka}{\ka+1}}\} = \begin{cases}
    e^{-2\ka}, &t\ge e^{2(\ka+1)}-1\\
    (t+1)^{-\ka}, & t\le e^{2(\ka+1)}-1.
\end{cases}\]
Hence for $\ka<1$, the number of steps required is $O\pf{\ln (T+1)}{\ka}$ to get to a constant and $\rc{\ka}\ln \prc{t_1}$ to get down to $t_1$ (as we need $s$ steps where $\pa{1-\fc{\ka}{\ka+1}}^s \lesssim \ln (1+t_1)\sim t_1$), for a total of $O\pa{\rc \ka \ln \pf{T+1}{t_1}}$. 
The last part follows from bounding every term and noting that $\sum_{k=1}^{\Nstep-1}\ln \pf{t_{k+1}+1}{t_k+1} \rc{t_k+1} \ep_{k+1}^2 \le \sum_{k=1}^{\Nstep-1} \ln \pf{t_{k+1}+1}{t_k+1} \max_{1\le k\le \Nstep-1} \rc{t_k+1} \ep_{k+1}^2\lesssim \ln \pf{T+1}{t_1+1} \fc{\ep^2}{\ln (T+1)}\le \ep^2$.
\end{proof}
\section{Inequalities}

\begin{thm}[Gaussian hypercontractivity, \cite{nelson1967dynamical}]\label{t:hypc}
    Let $q>p>1$ and $t$ be such that $e^{-t} = \fc{p-1}{q-1}$. Let $g:\R^n\to \R$ be a function such that $\ve{g}_{L^p(\ga)}<\iy$. 
    Then
\[
\ve{\sP_t g}_{L^q(\ga)}\le \ve{g}_{L^p(\ga)}.
\]
%\holden{does this work for vector-valued $g$ if we use $L^2$ vector norm?}
\end{thm}

\begin{lem}\label{l:hc-poly}
Let $q>2$ and let $f$ be a polynomial of degree at most $d$. Then 
\[
\ve{f}_{L^q(\ga)}
\le (q-1)^d \ve{f}_{L^2(\ga)}.
\]
\end{lem}
\begin{proof}
    Write $f=\sum_{|\mk|\le d} a_{\mk} h_{\mk}$ in the Hermite basis. Then $f=\sP_tg$ where $g=\sum_{|\mk|\le d} e^{|\mk|t} a_{\mk} h_{\mk}$. By Gaussian hypercontractivity (\Cref{t:hypc}), choosing $t$ such that $e^{-t} = \fc{1}{q-1}$,
    \begin{align} %\label{e:hc-deg-d}
    \ve{f}_{L^{q}(\ga)}
    = 
    \ve{\sP_t g}_{L^{q}(\ga)}
    \le \ve{g}_{L^2(\ga)} 
    \le  e^{td} \ve{f}_{L^2(\ga)}
     = (q-1)^d \ve{f}_{L^2(\ga)}.
    \end{align}
\end{proof}

The following tells us how the $L^2$ norm changes when we change from one Gaussian to another, for a bounded degree polynomial.
\begin{lem}\label{l:com-poly}
Let $f$ be a polynomial of degree at most $d$ and let $\nu$ be a measure such that for all $a\ge 0$, $\ve{\dd{\nu}{\ga}}_{L^{1+a}(\ga)}\le e^{\fc{aR^2}2}$ (e.g. from Lemma~\ref{l:chi-a}, $Q\convolve\cal N(0,I_n)$ where $Q$ is supported on $B_R(0)$). Then 
\[
\ve{f}_{\nu}^2 \le 
\ve{f}_\ga^2 e^{2\sqrt d R}.
\]
\end{lem}
Note this works for $\R^n$-valued polynomials as well, since in this case $\ve{f}_\nu^2 = \sumo in \ve{f_i}_\nu^2$.

\begin{proof}
%Suppose $\nu$ is a mixture of Gaussians at distances $\le R$. 
By H\"older's inequality and the given assumption,
    \begin{align*}
        \int_{\R^n} |f|^2 \,d\nu 
        &= \int_{\R^n} |f|^2 \dd{\nu}{\ga} \,d\ga \le \ve{f}_{L^{2\pa{1+\rc p}}(\ga)}^2 \ve{\dd{\nu}{\ga}}_{L^{p+1}(\ga)}
        \le \ve{f}_{L^{2\pa{1+\rc p}}(\ga)}^2 e^{\fc{pR^2}2}.
    \end{align*} 
    %~\Cref{l:chi-a}.
    By \Cref{l:hc-poly} with $q=2\pa{1+\rc p}$, 
\begin{align}\label{e:hc-deg-d}
    \ve{f}_{L^{2\pa{1+\rc p}}(\ga)}
    \le\pa{1+\fc 2p}^d \ve{f}_{L^2(\ga)}
    \le e^{\fc{2d}p} \ve{f}_{L^2(\ga)}. 
\end{align}

    Now take $p=\fc{2\sqrt{d}}R$ to get 
    \begin{align*}
\ve{f}_\nu^2 \le e^{\fc{2d}{p} + \fc{pR^2}2} \ve{f}_{\ga}^2 = e^{2\sqrt d R}\ve{f}_{\ga}^2.         
    \end{align*}
\end{proof}

\begin{lem}\label{l:chi-a}
Suppose $P=Q\convolve\cal N(0,I_n)$, where $Q$ is supported on $B_R(0)$. Then for all $a\ge 0$,
\[
\ve{\dd{P}{\ga}}_{L^{1+a}(\ga)} = 
\pa{\int \pa{\dd P{\ga}}^{1+a}d\ga}^{\rc{1+a}}\le \exp\pf{aR^2}2
\quad \text{and}
\quad 
\ve{\dd{\ga}{P}}_{L^{1+a}(P)}\le \exp\pf{aR^2}2
\]
\end{lem}
\begin{proof}
    By convexity, it suffices to consider when $P=\cal N(\mu, I_n)$ with $\ve{\mu}\le R$. Let $Z=(2\pi)^{n/2}$ be the normalizing constant of the standard Gaussian. Then 
\begin{align*}
    \pa{\int \pa{\dd P{\ga}}^{1+a}d\ga}^{\rc{1+a}}
    &= \rc Z \int_{\R^n} 
    \pf{e^{-\fc{\ve{x}^2}2}}{e^{-\fc{\ve{x-\mu}^2}2}}^a e^{-\fc{\ve{x-\mu}^2}2} \dx\\
    &= \rc Z \int_{\R^n} e^{-\fc{\ve{x}^2}2 - (a+1)\an{x,\mu} + (a+1)\fc{\ve{\mu}^2}2}\\
    &= \rc Z\int_{\R^n} e^{-\fc{\ve{x+(a+1)\mu}^2}2 + \fc{a(a+1)\ve{\mu}^2}2}\dx
    = e^{\fc{a(a+1)\ve{\mu}^2}2}
    \le e^{\fc{a(a+1)R^2}2}.
\end{align*}
Raising to the power $\rc{a+1}$ then gives the first result in this case. For $P=\cal N(\mu, I_n)$ the same inequality follows from symmetry.
\end{proof}

\section{Bounding $TV(P_0, P_{\sigma^2})$}
Note that Algorithm~\ref{alg:mainalgo} ultimately samples from $P_{t_1}$ for some positive time $t_1$ close to zero. Hence, to prove a TV guarantee on the distribution of the sample with respect to the target mixture $P_0$, we need to bound  $TV\pa{P_{t_1}, P_{0}}$. Fortunately,
this task is convenient using the Gaussian-convolution structure hidden in $P_0$ and $P_{t_1}$, by upper bounding their KL divergence. %we note that the TV distance between $p_0$ and $p_t$ for small $t$ can be bounded. This is true for any smooth $p_0$ \cite{lee2023convergence}, but in our case admits a direct proof.
\begin{lem}\label{lem:tvguarantee}
    We have
    \[
\TV(P_0, P_{\si^2}) \le \rc{\sqrt 2} \fc{\si^2\sqrt n}{\si_0^2}.
    \]
\end{lem}
\begin{proof}
    Note
    \begin{align*}
        \KL(\cal N(0,\si_0^2I_n) \| \cal N(0,(\si_0^2+\si^2)I_n)
        &\le \rc 2 
        \ba{-n \ln \fc{\si_0^2+\si^2}{\si_0^2} - n + n \fc{\si_0^2+\si^2}{\si_0^2}}\\
        &\le \fc n2
        \ba{- \fc{\si^2}{\si_0^2} + \rc 2\fc{\si^4}{\si_0^4} + \fc{\si^2}{\si_0^2}}\le \fc n4\fc{\si^4}{\si_0^4}.
    \end{align*}
    By Pinsker's inequality, 
    \[\TV(\cal N(0,\si_0^2I_n) , \cal N(0,(\si_0^2+\si^2)I_n) )\le \sqrt{2\KL(\cal N(0,\si_0^2I_n) \| \cal N(0,(\si_0^2+\si^2)I_n)}\le \rc{\sqrt 2} \fc{\si^2\sqrt n}{\si_0^2}.\] The result follows from joint convexity of TV distance.
\end{proof}

%\fixme{TV guarantee to $p_0$ under smoothness.}

\section{Sub-exponential random variables}\label{sec:subexponential}
\begin{definition}\label{def:subexp}
    We say centered random variable $X$ is $(v^2, \alpha)$ sub-exponential for $v,\alpha > 0$ if
    \begin{align*}
        \E e^{\lambda X} \leq e^{\frac{\lambda^2 v^2}{2}}, \forall \lambda: \quad \ab{\lambda} \leq \alpha^{-1}.
    \end{align*}
    We call any random variable sub-exponential if its centered version is sub-exponential.
\end{definition}

Sub-exponential random variables obey a desirable concentration of measure due to their fast decaying tails. This further implies their moments have a slow growth rate, a property that is of interest.

\begin{pr}[Sub-exponential properties~\cite{vershynin2020high}]\label{fact:subexp}
    The following properties are equivalent:
    \begin{itemize}
        \item $X$ is $((c_0v)^2,c_0v)$-sub-exponential for some universal constant $c_0$.
        \item $X$ satisfies the tail bound $\mathbb P(X \geq t) \leq 2\exp\pa{-c_1t/v}$ for universal constant $c_1$. 
        \item For all $p \geq 1$, the $p$th moment of $X$ is bounded as $\ve{X}_p \leq c_2 v p$ for universal constant $c_2$.
    \end{itemize}
\end{pr}
By equivalent, we mean that given one of the properties, each of the other properties holds for some constant depending on the original constant. In particular, the second property implies that the square of a $v$-sub-Gaussian variable is $(O(v^2),O(v))$-sub-exponential. Another important property of sub-exponential variables is that they behaves nicely under summation.

\begin{pr}[Adding sub-exponential variables]
\label{p:sum-subexp-ind}
    Given independent $(v_i^2,\al_i)$-sub-exponential random variables $X_i$ for $1\le i\le n$, 
    $\sumo in X_i$ is $\pa{\sumo in v_i^2 , \max_{1\le i\le n}\al_i}$-sub-exponential.
\end{pr}
%\jon{I couldn't find a good reference other than lecture notes, so I just wrote the one line proof.}
\begin{proof}
By induction, it suffices to prove this for the case $n=2$. 
    Our assumptions imply that $\E e^{\lambda X_1} \leq e^{\frac{\lambda^2 v_1^2}{2}}$ and $\E e^{\lambda X_2} \leq e^{\frac{\lambda^2 v_2^2}{2}}$
    for all $\lambda$ such that $\ab{\lambda} \leq \min\{\alpha_1^{-1},\alpha_2^{-1}\}=\max\{\alpha_1, \alpha_2\}^{-1}$.
    By the independence of $X_1$ and $X_2$, we therefore have
 \[
 \E e^{\lambda (X_1+X_2)}
 =
 \E e^{\lambda X_1}\E e^{\lambda X_2}
  \leq e^{\frac{\lambda^2 v_1^2}{2}}  e^{\frac{\lambda^2 v_2^2}{2}}
	= e^{\frac{\lambda^2 \left(v_1^2 + v_2^2\right)}{2}}    
    \]
    for all such $\lambda$, so $X_1+X_2$ is $(v_1^2 + v_2^2, \max\{\alpha_1, \alpha_2\})$, as claimed.
\end{proof}
When the random variables are not independent, we accrue an extra factor of the number of terms.
\begin{pr}\label{p:sum-subexp}
    Given $(v_i^2,\al_i)$-sub-exponential random variables $X_i$ for $1\le i\le n$ (not necessarily independent), 
    $\sumo in X_i$ is $\pa{n\cdot \sumo in v_i^2 , n \cdot \max_{1\le i\le n}\al_i}$-sub-exponential.
\end{pr}
\begin{proof}
    For $|\la|\le \rc n \min_{1\le i\le n} \al_i^{-1} = (n\cdot \max_{1\le i\le n}\al_i)^{-1}$, the defining inequality for the sub-exponential random variables are satisfied for $\la n$. By H\"older's inequality,
    \[
\E e^{\la \sumo in X_i} 
\le \prodo in \pa{\E e^{\la nX_i}}^{1/n}
\le \prodo in e^{\fc{n^2\la^2v_i^2}2\cdot \rc n}
= e^{\fc{\la^2 n\sumo in v_i^2}2}.
    \]
\end{proof}

\section{Derivation of Tweedie's formula}
\label{s:tweedie}
We derive \eqref{e:score-Q}. 
Consider $\mu\sim Q$, $\xi \sim \cal  N(0,\si^2 I_n)$ drawn independently, and let $Y=\mu+\xi$. 
Let $Q_{\si^2} = Q*\cal N(0,\si^2 I_n)$, and suppose it has density $q_{\si^2}$. %Let $V_t = \nb \ln q_t$. 
By Bayes's Rule, letting $q(y|\mu) = \ga_{\mu, \si^2}(y)$ denote the density of $Y$ given $\mu$, 
\begin{align*}
\nb %V_{\si^2}
\ln q_{\si^2} (y) &= \nb \log (q*\ga_{\si^2}(y)) = \nb_y \log \int_{\R^d} e^{-\fc{\ve{y-\mu}^2}{2\si^2}}\,dQ_0(\mu)\\
&= -\fc{\int_{\R^d} \fc{y-\mu}{\si^2}  e^{-\fc{\ve{y-\mu}^2}{2\si^2}}\,dQ_0(\mu)}{\int_{\R^d} e^{-\fc{\ve{y-\mu}^2}{2\si^2}}\,dQ_0(\mu)}
= \fc{\int_{\R^d} \fc{\mu-y}{\si^2} q(y|\mu)\,dQ_0(\mu)}{\int_{\R^d} q(y|\mu)\,dQ_0(\mu)}\\
&=\rc{\si^2} \E[\mu-y|Y=y].
%=\rc{\si^2} \E[\xi|Y=y].
\end{align*}

\end{document}